\documentclass[11pt]{article}
\usepackage{fullpage}

\usepackage[utf8]{inputenc} % allow utf-8 input
\usepackage[T1]{fontenc}    % use 8-bit T1 fonts
\usepackage{url}            % simple URL typesetting
\usepackage{booktabs}       % professional-quality tables
\usepackage{amsfonts}       % blackboard math symbols
\usepackage{nicefrac}       % compact symbols for 1/2, etc.
\usepackage{multirow}
\usepackage{makecell}
\usepackage{hhline}

\usepackage{tocloft}
\usepackage{braket}
\usepackage[a4paper,margin=1in]{geometry}
\usepackage[toc,page,header]{appendix}
\usepackage{minitoc}
\usepackage{booktabs}
\usepackage{multirow}

\usepackage[colorlinks,citecolor=blue,urlcolor=blue,linkcolor=blue,linktocpage=true,pagebackref=true]{hyperref}
\renewcommand*{\backref}[1]{}
\renewcommand*{\backrefalt}[4]{(\small%
    \ifcase #1 not cited%
          \or cited on page~#2%
          \else cited on pages #2%
    \fi%
    )}
\usepackage{wustyle}

\let\citet\cite
\let\citep\cite

\usepackage{enumitem}

\usepackage[most]{tcolorbox}
\tcbset{tightbox/.style={
  enhanced,
  colback=white,          % 背景色
  colframe=black!20,      % 边框颜色
  boxrule=0.1pt,          % 边框粗细
  left=1pt, right=1pt,    % 内边距左右
  top=1pt, bottom=1pt,    % 内边距上下
  boxsep=0pt,             % 去掉额外padding
  arc=1pt,                % 圆角半径
  before skip=1pt,        % 盒子前后的垂直间距
  after skip=1pt,
}}

\usepackage{wrapfig}

\usepackage{framed}
\colorlet{shadecolor}{orange!15}

\definecolor{c4}{RGB}{255,225,187}
\definecolor{c2}{RGB}{209, 233, 184}
\definecolor{c3}{RGB}{218,243,246}
\definecolor{c1}{RGB}{249, 229, 229}
\definecolor{c5}{RGB}{255, 128, 128}
\definecolor{c6}{RGB}{251, 132, 002}

\title{Unveiling the Role of Learning Rate Schedules \\ via Functional Scaling Laws}
\title{Predicting Loss Dynamics via Functional Scaling Laws}
\title{\bf Functional Scaling Laws in Kernel Regression: Loss Dynamics and Learning Rate Schedules}

\author{
    Binghui Li$^{1,}$\thanks{Equal contribution.}\quad
Fengling Chen$^{2,}$\footnotemark[1]\quad
Zixun Huang$^{2,}$\footnotemark[1] \quad
Lean Wang$^{3,}$\footnotemark[1] \quad
Lei Wu$^{1,2,4,}$\thanks{Corresponding author.}
    \\
    \\
  $^1$Center for Machine Learning Research, Peking University
    \\
  $^2$School of Mathematical Sciences, Peking University 
    \\
  $^3$State Key Laboratory of Multimedia Information Processing,\\
  School of Computer Science, Peking University
    \\
  $^4$AI for Science Institute, Beijing\\[.5em]
  \texttt{\{libinghui, lean\}@pku.edu.cn},\quad  \texttt{flchen\_lwycc@stu.pku.edu.cn}\\ 
  \texttt{alexpku@stu.pku.edu.cn}, \quad\texttt{leiwu@math.pku.edu.cn}
}
\date{(Accepted at NeurIPS 2025)}

\begin{document}

\maketitle

\begin{abstract}
Scaling laws have emerged as a unifying lens for understanding and guiding the training of large language models (LLMs). 
However, existing studies predominantly focus on the final-step loss, leaving open whether the entire \emph{loss dynamics} obey similar laws and, crucially, how the \emph{learning rate schedule} (LRS) shapes them.
We address these gaps in a controlled theoretical setting by analyzing stochastic gradient descent (SGD) on a power-law kernel regression model. The key insight is a novel {\bf intrinsic-time} viewpoint, which captures the training progress more faithfully than iteration count. We then establish a {\bf Functional Scaling Law (FSL)} that captures the full loss trajectory under arbitrary LRSs, with the schedule’s influence entering through a simple convolutional functional.  
We further instantiate the theory for three representative LRSs---constant, exponential decay, and warmup–stable–decay (WSD)---and derive explicit scaling relations in both data- and compute-limited regimes.
These comparisons explain key empirical phenomena:
(i) higher-capacity models are more data- and compute-efficient;
(ii) learning-rate decay improves training efficiency; and
(iii) WSD-type schedules outperform pure decay.
Finally, experiments on LLMs ranging from 0.1B to 1B parameters demonstrate the practical relevance of FSL as a surrogate model for fitting and predicting loss trajectories in large-scale pre-training.
\end{abstract}

\doparttoc 
\faketableofcontents 
\part{} 
% \tableofcontents

\vspace*{-2.2em}
\section{Introduction}
\label{sec:intro}
% \vspace*{-0.5em}

It is well established that the training of large-scale deep learning models mysteriously follows {\em scaling laws}, whereby model performance improves {\it predictably} with available resources such as compute or data~\citep{hestness2017deep}. 
In particular, the landmark study by Kaplan et al.~\citep{kaplan2020scaling} demonstrated that, in LLM pre-training, the  loss~$L$ decreases with model size~$M$ and dataset size~$D$ according to a remarkably simple power-law relation:
\begin{equation}\label{eqn: scaling-law}
    L(M,D) = L_0 + C_M M^{-\alpha_M} + C_D D^{-\alpha_D},
\end{equation}
where $\alpha_M$ and $\alpha_D$ are the scaling exponents, $L_0$ denotes the irreducible loss, and $C_M, C_D$ are some constants.  Such empirical relations have proven  robust  across model architectures,  scales, and training setups~\cite{hoffmann2022training,touvron2023llama,liu2024deepseek}, and have become   foundational principles for guiding  LLM development~\cite{henighan2020scaling, kadra2023power, aghajanyan2023scaling, bi2024deepseek, shuai2024scaling, kumar2024scaling}. In practice, they are now routinely used  to design optimal resource-allocation strategies~\citep{hoffmann2022training} and to tune key hyperparameters such as learning rates and batch sizes~\citep{liu2024deepseek,li2025predictablescalei}. 

% Beyond LLMs, the scaling-law paradigm has profoundly shaped the evolution of modern artificial intelligence, offering a unifying empirical principle that links model design, data scaling, and compute allocation within a single coherent framework.

Despite their empirical success, the theoretical understanding of scaling laws remains  limited. Recent studies have begun to illuminate the underlying mechanisms~\citep{sharma2020neural,hutter2021learning,maloney2022solvable,wei2022more,jain2024scaling,michaud2024quantization,nam2024exactly,atanasov2024scaling,dohmatob2024tale,bahri2024explaining,bordelon2024dynamical,lin2024scaling,paquette2024fourplus,bordelon2024feature,zhang2024does}, yet two important gaps persist:
\begin{itemize}
\item {\bf Determinants of scaling efficiency.} Existing studies lack a systematic characterization of how key factors---such as model capacity, task difficulty, and hyperparameter choices---govern scaling efficiency, as reflected by the exponents $\alpha_M$ and $\alpha_D$. In particular, {\bf learning rate schedules (LRSs)} are known to be critical in practice~\citep{mccandlish2018empirical,bergsma2025straight,hagele2024scaling}, but their precise role in shaping scaling efficiency remains unclear.  

\item {\bf Beyond the final-step loss.}  
The classical scaling law~\eqref{eqn: scaling-law} addresses only the end-of-training loss~\citep{kaplan2020scaling,hoffmann2022training}, leaving open whether the entire \emph{loss trajectory} exhibits analogous scaling behavior.  
While recent empirical evidence~\citep{tissue2024scaling,luo2024a} suggests such possibility, a systematic study and theoretical justification are still missing.

%both large-scale validation and theoretical grounding.
\end{itemize}

\begin{figure}[!t]
    \centering
    \hspace*{2.5em}\includegraphics[width=0.5\textwidth]{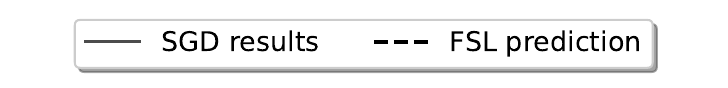} \\[-.5em]
    \subfloat[\small Loss Dynamics]{
    \includegraphics[width=0.44\textwidth]{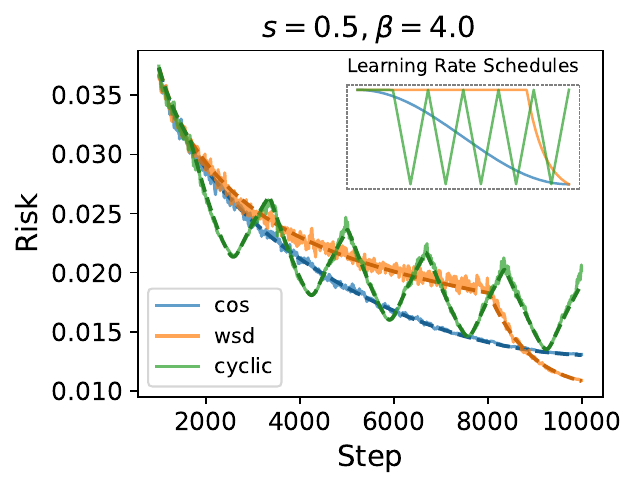}
      \vspace*{-.5em}
    }
    \subfloat[\small Scaling Behavior]{
      \includegraphics[width=0.33\textwidth]{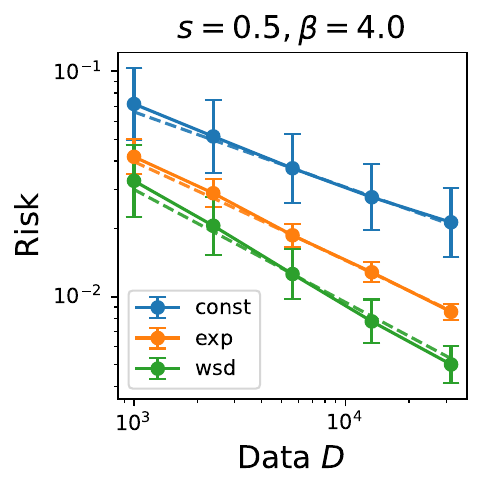}
      \vspace*{-.5em}
    }
    \vspace*{-.5em}
  \caption{\small 
\textbf{Functional Scaling Law (FSL) accurately captures the loss dynamics and scaling behavior of SGD in kernel regression.} In both subplots, \textbf{solid lines} denote the results of SGD, and \textbf{dashed lines} indicate the corresponding FSL predictions. 
\textbf{(a)}~Loss dynamics of SGD (averaged over 1000 runs) compared with FSL predictions under three learning-rate schedules: cosine, WSD-like, and a non-standard cyclic schedule.
\textbf{(b)}~Final-loss scaling predicted by FSL using the analytical formulas from Section~\ref{sec: effect-lrs}, compared with the mean of 200 independent SGD runs.
}
  \label{fig: sgd-fsl}
\vspace*{-1em}
\end{figure}

\subsection{Our Contribution}

In this paper, we take a step toward addressing these gaps in a controlled yet representative theoretical setting.  We study stochastic gradient descent (SGD) training of the {\bf power-law kernel (PLK)} regression---a  widely adopted  surrogate  for scaling-law analysis~\citep{bordelon2024dynamical,bahri2024explaining,paquette2024fourplus,lin2024scaling,bordelon2024feature}. 
The PLK regression is characterized by four parameters: the task difficulty~$s$, the capacity exponent~$\beta$, the model size~$M$, and the label-noise level $\sigma$.
To capture the influence of learning-rate schedules (LRSs), we model SGD via an {\bf intrinsic-time SDE}, where the concept of {\bf intrinsic time} emerges as a key quantity enabling a unified characterization of how different LRSs shape the loss dynamics and scaling behavior.  
Building on this formulation, we establish the {\bf Functional Scaling Law (FSL)}, which provides an accurate characterization of the entire {\it loss dynamics} beyond the traditional  final-loss prediction.

Concretely, for a general intrinsic-time LRS $\gamma:[0,\infty)\!\to\![0,\infty)$, and under some conditions,  
the dynamics of the expected loss $\EE[\cR(\bnu_t)]$ (where $t$ denotes the intrinsic time) satisfies:
\begin{tcolorbox}[boxrule=.7pt, boxsep=0pt, left=3pt, right=4pt, top=1pt, bottom=3pt]
\begin{equation}\label{eqn:fsl-intro}
    \EE[\cR(\bnu_t)] - \underbrace{\frac{\sigma^2}{2}}_{\text{irreducible error}}
    \;\eqsim\;
    \underbrace{\frac{1}{M^{s\beta}}}_{\text{approx. error}}
    + \underbrace{e(t)}_{\text{signal learning}}
    + \underbrace{\!\int_0^t\! \cK(t-z)\,[e(z)+\sigma^2]\,\gamma(z)\,\dd z}_{\text{noise accumulation}},
\end{equation}
\end{tcolorbox}
\noindent where $e(t)=(1+t)^{-s}$ and $\cK(t)=(1+t)^{-(2-1/\beta)}$.
 Each term in FSL admits clear interpretation: $\tfrac{\sigma^2}{2}$ denotes the \textbf{irreducible error} caused by label noise,  
$M^{-s\beta}$ represents the \textbf{approximation error},  
$e(t)$ characterizes the \textbf{signal-learning dynamics} under noiseless (full-batch) gradient descent,  
and the final term captures the injection and dissipation of gradient noise, with the LRS $\gamma$ entering through a tractable \textbf{convolutional functional}.  
The function $\cK$, referred to as {\bf forgetting kernel}, quantifies how fast the injected noise dissipates during training. Figure~\ref{fig: sgd-fsl}(left) shows that the FSL  accurately captures the loss dynamics of SGD under diverse LRSs.

Building on FSL, we  further derive explicit scaling relations for the final-step loss under three representative LRSs---constant, exponential decay~\cite{ge2019step}, and warmup–stable–decay (WSD)~\cite{zhai2022scaling,hu2024minicpm}---in both {\bf data-limited} and {\bf compute-limited regimes}. The results, summarized in Table~\ref{tab: optimal-scaling} and Figure~\ref{fig: sgd-fsl}(right), recover and extend prior analyses~\citep{bordelon2024dynamical,bordelon2024feature,paquette2024fourplus,lin2024scaling}, while revealing several unifying insights.
\begin{itemize}
    \item {\bf Scaling efficiency of different schedules.} WSD achieves the best scaling efficiency, followed by exponential decay and then constant schedules. This efficiency hierarchy provides  theoretical justification for the importance of learning-rate decay and explains empirical success of WSD~\citep{zhai2022scaling,hu2024minicpm,team2025kimi,liu2024deepseek}.
     
     \item {\bf Role of model capacity.} Higher-capacity models are consistently more efficient in both compute and data, highlighting the necessity of scaling model capacity~\cite{kaplan2020scaling}.

    \item {\bf Data–model trade-off.} Compute-optimal training requires scaling data more than model size, consistent with established heuristics in LLM pre-training~\cite{hoffmann2022training}.

    \item {\bf Scaling law for peak learning rate.} Optimal scaling requires the peak learning rate (LR) to  scale appropriately with the training budget (data or compute), revealing the importance of careful peak LR tuning~\cite{bjorckscaling,li2025predictablescalei}.
\end{itemize}

Beyond PLK regression, we also apply the FSL ansatz to fit and predict loss trajectories from LLM pre-training experiments with model sizes ranging from 0.1B to 1B parameters, covering both dense and MoE architectures.  
These results highlight the potential of FSL as a practical surrogate for understanding and guiding LLM pre-training.

\begin{table}[!t]
\caption{
\small
\textbf{Learning-rate schedule (LRS) strongly influences scaling efficiency in power-law kernel regression.}
Efficiency is determined by two key factors: relative task difficulty $s \in (0, \infty)$ and model capacity $\beta > 1$.
We distinguish between an \emph{easy-learning regime} ($s \ge 1 - 1/\beta$) and a \emph{hard-learning regime} ($s < 1 - 1/\beta$).
% Empirically, (1) the WSD schedule consistently outperforms exponential decay, which in turn surpasses the constant schedule; and (2) higher-capacity models (smaller $\beta$) are consistently more data- and compute-efficient.
}
\vspace*{.2em}
\renewcommand{\arraystretch}{1.2}
\resizebox{0.99\textwidth}{!}{
  \begin{tabular}{c|cc|cc}
    \Xhline{1.2pt} 
    \multirow{2}{*}{Learning Rate Schedule (LRS)} 
      & \multicolumn{2}{c|}{Data-Optimal Scaling Laws} 
      & \multicolumn{2}{c}{Compute-Optimal Scaling Laws} \\
    \hhline{|~----|} 
              & Easy  & Hard & Easy & Hard \\
    \Xhline{0.8pt} 
    Constant & \multicolumn{2}{c|}{$  D^{-\frac{s}{s+1}}$} 
             & \multicolumn{2}{c}{$  C^{-\frac{s\beta}{1+s\beta+{\beta}}}$} \\
    Exponential-decay 
             & $D^{-\frac{s\beta}{1+s\beta}} (\log D)^{\frac{s\beta}{1+s\beta}}$     
             & $D^{-s}(\log D)^{s}$    
             & $C^{-\frac{s\beta}{2+s\beta}}(\log C)^{\frac{s\beta}{2+s\beta}}$     
             & $C^{-\frac{s\beta}{1+{\beta}}}(\log C)^{\frac{s\beta}{1+{\beta}}}$ \\
    Warmup-stable-decay (WSD) 
             & $  D^{-\frac{s\beta}{1+s\beta}}(\log D)^{\frac{s\beta-s}{1+s\beta}}$     
             & $D^{-s}$     
             & $C^{-\frac{s\beta}{2+s\beta}}(\log C)^{\frac{s\beta-s}{2+s\beta}}$     
             & $C^{-\frac{s\beta}{1+\beta}}$ \\
    \Xhline{1.2pt} % 底部粗线
    \end{tabular}%
}
  \label{tab: optimal-scaling}%
  % \vspace*{-1em}
\end{table}

% \vspace*{-.5em}
\subsection{Related Work}

\paragraph*{Theoretical explanation of scaling laws.}
Among the growing body of work seeking to theoretically explain scaling laws~\citep{sharma2020neural,hutter2021learning,maloney2022solvable,wei2022more,jain2024scaling,michaud2024quantization,nam2024exactly,atanasov2024scaling,dohmatob2024tale,bahri2024explaining,bordelon2024dynamical,lin2024scaling,paquette2024fourplus,bordelon2024feature,pan2025understanding}, 
the most closely related are~\citep{bordelon2024dynamical,paquette2024fourplus,bordelon2024feature,lin2024scaling}, which also analyze PLK regression (often written in the equivalent linear-regression form).  
Specifically,~\citep{bordelon2024dynamical} studies gradient flow,~\citep{paquette2024fourplus,bordelon2024feature} analyze SGD with a constant LRS, and~\citep{lin2024scaling} considers  an exponential-decay LRS.  
In contrast, we establish a unified scaling law applicable to general LRSs, which not only recovers these prior results as special cases but also substantially extends them by capturing the  loss dynamics rather than only the final-step loss. This unification is enabled by introducing the key notion of {\it intrinsic time}, which more faithfully captures the effective training progress than the raw number of training steps.

\paragraph*{Predicting loss trajectories in LLM pre-training.}
Recent empirical studies~\citep{tissue2024scaling,luo2024a} have shown that the entire loss trajectory of LLM pre-training—not merely the final loss—can be accurately captured by suitable scaling relations.
A detailed description of the corresponding fitting procedures is provided in Appendix~\ref{sub:empirical_fitting_of_llm_pre_training_loss_trajectory}.
Our theory offers a theoretical explanation for these empirical findings.
Interestingly, the \emph{multi-power-law} (MPL) model proposed by~\cite{luo2024a} is closely connected to our FSL: through an integration-by-parts transformation, the FSL expression can be recast into a form that is nearly equivalent to the MPL formulation (see Appendix~\ref{appendix:llm_pre-training}).

\paragraph*{Warmup-Stable-Decay (WSD) LRS.}
A WSD schedule~\cite{zhai2022scaling,hu2024minicpm} maintains a  constant learning rate for a long stable phase, followed by a learning rate decay only near the end of training. Although unconventional, WSD has become popular in LLM pre-training~\cite{hu2024minicpm,hagele2024scaling} and is already deployed in training industry-scale LLMs such as \texttt{DeepSeek-V3}~\cite{liu2024deepseek} and \texttt{Kimi-K2}~\cite{team2025kimi}. Yet its mechanism remains poorly understood. While recent works~\cite{wen2024understanding,schaipp2025surprising} offer partial insights, we show---perhaps surprisingly---that even \emph{quadratic optimization}, corresponding to a kernel regression problem, already reproduces the essential advantage of WSD.  Furthermore, we quantify this advantage through explicit comparisons of scaling efficiency against constant and exponential-decay schedules.

\subsection{Notation}
For any $n \in \mathbb{N}$, let $[n] := \{1, 2, \dots, n\}$.  
For a positive semi-definite (PSD) matrix $\bS$, denote by $\mu_j(\bS)$ its $j$-th largest eigenvalue, and define the $\bS$-induced norm $\|\bu\|_{\bS} := \sqrt{\bu^\top \bS \bu}$ for any vector $\bu$.  
We write $\bA \preceq \bB$ (resp.~$\bA \succeq \bB$) if $\bB - \bA$ (resp.~$\bA - \bB$) is PSD.
Throughout the paper, we use $\eqsim$ to denote equivalence up to a  constant factor, and $\lesssim$ (resp.~$\gtrsim$) to denote an inequality up to a constant factor.  
For two nonnegative functions $f, g: \mathbb{R}_{\ge 0} \to \mathbb{R}_{\ge 0}$, we write $f(t) \eqsim g(t)$ if there exist constants $C_1, C_2 > 0$ (independent of $t$) such that
$
C_1 f(t) \le g(t) \le C_2 f(t), \; \forall\, t \ge 0.
$

% \section{Preliminaries}
% \label{sec:problem_setup}
% %\vspace*{-.3em}

% \vspace*{-.3em}
\section{Power-Law Kernel (PLK) Regression}
\label{sec:problem_setup}
% \vspace*{-.3em}

Let $\cX$ denote the input domain and $\cD$ the input distribution, and assume labels are generated by
\[
    y = \langle \bphi(\bx), \btheta^* \rangle + \epsilon,
\]
where $f^*(\bx):=\langle \bphi(\bx), \btheta^* \rangle$ is the target function, and the label noise $\epsilon \sim \cN(0,\sigma^2)$ is independent of $\bx$. 
Here $\bphi:\cX \to \RR^N$ with $N \in \NN_+ \cup \{\infty\}$ is a feature map, satisfying the following assumption:
\begin{assumption}[Hypercontractivity]\label{assumption: hypercontractivity}
Let $\bH := \EE_{\bx \sim \cD}[\bphi(\bx)\bphi(\bx)^\top]$ be the feature covariance.  
There exist $\bphi$-dependent constants $\rho_{-}, \rho_{+} > 0$ such that for any PSD matrix $\bA \in \RR^{N\times N}$,
\begin{equation*}
    \rho_{-}\,\tr(\bH\bA)\,\bA 
    \;\preceq\;
    \EE_{\bx \sim \cD}\!\Big[\left(\bphi(\bx)^\top \bA \bphi(\bx)\right)\, \bphi(\bx)\bphi(\bx)^\top \Big]
    - \bH \bA \bH
    \;\preceq\;
    \rho_{+}\,\tr(\bH\bA)\,\bA.
\end{equation*}
\end{assumption}

This condition ensures that the feature distribution is sufficiently regular—its fourth-order moments are controlled by the second-order ones~\cite{mei2022generalization}.
It holds, for example, for Gaussian features $\bphi(\bx)\sim \cN(0,\bH)$ with $\rho_{-}=1, \rho_{+}=2$ (see Lemma~\ref{lem:4th_moment_estim}).

\begin{assumption}\label{ass:diagonal}
$\bH=\operatorname{diag}(\lambda_1,\lambda_2,\dots,\lambda_N)$ with $\lambda_1 \ge \lambda_2 \ge \cdots \ge \lambda_N$.
\end{assumption}

This diagonalization assumption is made for simplicity.
Since SGD is invariant under orthogonal rotations of the feature space, the analysis under this assumption is equivalent to the general case up to rotation.

To learn $f^*$, we consider a {\bf model of width $M$}:
$
    f(\bx; \bv) =  \sum_{j=1}^M v_j \bw_j^\top \bphi(\bx)=:\langle \bv, \bW \bphi(\bx)\rangle,
$
where $\bv \in \RR^{M}$ denotes trainable weights and $\bW \in \RR^{M\times N}$ projects the $N$-dimensional features onto an $M$-dimensional subspace.  
We study two choices of projection $\bW$:  
\begin{itemize}
  \item \textbf{Top-$M$ features:} $\bw_j=\be_j$ for $j\in [M]$, i.e., selecting the top-$M$ features $\{\phi_j\}_{j=1}^M$;  
  \item \textbf{Random-$M$ features:} $\bw_j \sim \cN(0,I_N)$ independently for $j \in [M]$.  
\end{itemize}

The top-$M$ setting is  a particularly simple yet analytically representative case,  also adopted in prior scaling-law studies~\citep{nam2024exactly,ding2025scaling}.
For random features~\citep{bahri2024explaining,bordelon2024dynamical,paquette2024fourplus,lin2024scaling,bordelon2024feature}, we will show that, in certain regimes, their scaling behavior closely parallels that of the top-$M$ case. As clarified in Appendix~\ref{appendix_rkhs}, our setup is equivalent to learning with the kernel $K_\phi(\bx,\bx'):=\bphi(\bx)^\top \bphi(\bx')$.

\subsection{Model Capacity and Task Difficulty}
\label{subsec:source-capacity}
We now formalize the key notions of \emph{model capacity} and \emph{task difficulty}.  
Let $\widehat{\phi}_j := \phi_j / \lambda_j^{1/2}$ for $j \in [N]$, so that 
$\{\widehat{\phi}_j\}_{j=1}^N$ forms an orthonormal basis of $L^2(\cD)$.

\begin{assumption}[Model capacity]\label{ass:capacity} 
The spectrum of the feature map satisfies
$
\lambda_j \eqsim j^{-\beta}, \beta > 1.
$
\end{assumption}

The condition $\beta > 1$ ensures $\tr(\bH) = \sum_{j=1}^N \lambda_j \le C$ for some constant $C$ independent of $N$,  
making our analysis \emph{dimension-free} and applicable to the infinite-dimensional setting ($N = \infty$).  

Restricting attention to the top-$M$ features, any model of the form
\[
f(\cdot;\bv)
  = \sum_{j=1}^M v_j \phi_j
  = \sum_{j=1}^M v_j \lambda_j^{1/2} \, \widehat{\phi}_j
  \eqsim \sum_{j=1}^M v_j\, j^{-\beta/2} \widehat{\phi}_j
\]
reveals that higher-index features are increasingly down-weighted by the factor $j^{-\beta/2}$.  
As $\beta$ increases, the spectrum decays more rapidly, and the model \emph{effectively} relies on fewer features.  
Hence, the model's expressive power is governed by two complementary factors:
\begin{itemize}
\item the \textbf{model size} $M$, which controls how many features are retained, and  
\item the \textbf{capacity exponent} $\beta$, which controls how quickly these features decay in importance. 
\end{itemize}
%A smaller $\beta$ corresponds to a higher-capacity model.

\begin{tcolorbox}[boxrule=.7pt, boxsep=0pt, left=3pt, right=4pt, top=4pt, bottom=3pt]
\begin{remark}
Note that for a fixed target function $f^*$, one can adopt different (potentially) non-linear feature map $\bphi$ (and hence different $\beta$).  
The value of $\beta$ reflects the \emph{capacity} of the chosen features.  
For instance, consider $\bphi(\bx) = \nabla_{\theta}\cN(\bx;\theta)$, where $\cN(\cdot;\theta)$ denotes a neural network.  
Then $\bphi(\bx)$ corresponds to neural tangent features, and the associated kernel
$K_\phi(\bx,\bx') := \bphi(\bx)^\top \bphi(\bx')$ 
is known as the neural tangent kernel (NTK)~\citep{jacot2018neural}.  
In this case, the network depth and activation functions govern the spectral decay, and thus determine the effective exponent~$\beta$~\citep{bietti2020deep,wu2022spectral}.
\end{remark}
\end{tcolorbox}

\begin{assumption}[Task difficulty]\label{ass:source} 
Suppose $|\theta^*_j|^2 \eqsim j^{-1}\lambda_j^{\,s-1}$ for some $s>0$.
\end{assumption}
Under Assumptions~\ref{ass:capacity} and~\ref{ass:source}, the target function admits the expansion
\[
  f^*=\sum_{j=1}^N \theta_j^* \phi_j
  \eqsim \sum_{j=1}^N j^{-1/2}\lambda_j^{\,s/2}\;\widehat{\phi}_j
  \eqsim \sum_{j=1}^N j^{-(s\beta+1)/2}\;\widehat{\phi}_j.
\]
Since $\{\hat{\phi}_j\}$ are orthonormal, this assumption implies that  the spectral energy of $f^*$ decays as a power law.  
The exponent $\alpha:=s\beta$ therefore quantifies the task’s {\bf intrinsic difficulty}, which depends only on the target function itself and is independent of the model spectrum.  
In contrast, $s$ measures the {\bf relative difficulty} with respect to a model of capacity~$\beta$:  
for a fixed~$f^*$ (fixed~$\alpha$), adopting a higher-capacity model (smaller~$\beta$) increases~$s=\alpha/\beta$, making the task relatively easier. In other words, the same task appears easier to a higher-capacity model.

% This naturally raises the question:
% % \vspace*{-.4em}
% \begin{tcolorbox}[boxrule=.7pt, boxsep=0pt]
% \begin{center}
% \textit{For a fixed target function~$f^*$, how does increasing the model capacity (i.e., decreasing~$\beta$) affect the overall scaling behavior?}
% \end{center}
% % \end{tcolorbox}
% % \vspace*{-.4em}

We remark that similar assumptions---commonly referred to as {\em capacity} and {\em source conditions}—have been widely used in the analysis of kernel methods~\citep{caponnetto2005fast,caponnetto2007optimal,spigler2020asymptotic,bordelon2020spectrum,maloney2022solvable}.  
Our work builds upon and extends this line of research.

\section{One-Pass SGD and Intrinsic-Time SDE}
\label{sec:proof_sketch}

Given a data point $\bz=(\bx,y)\in\cX\times\RR$ and a model $f(\cdot;\bv)$, define the loss  
$
  \ell(\bz,\bv) = \tfrac{1}{2}\bigl(f(\bx;\bv)-y\bigr)^2.
$
Then, the population risk is
\begin{equation}\label{eqn: pop-risk}
  \cR(\bv)
  = \EE_{\bz}[\ell(\bz,\bv)]
  = \tfrac{1}{2}\|\bW^\top\bv-\btheta^*\|_{\bH}^2 + \tfrac{\sigma^2}{2}
  =: \cE(\bv) + \tfrac{\sigma^2}{2},
\end{equation}
where $\cE(\bv)$ denotes the excess risk.  
We minimize $\cR(\bv)$ via {\bf one-pass SGD}, given by 
\begin{equation}\label{eqn: sgd}
  \bv_{k+1}
  = \bv_{k}
  - \frac{\eta_k}{B_k}\sum_{\bz\in S_k}\nabla_{\bv}\ell(\bz,\bv_{k}),
\end{equation}
where $S_k:=\{(\bx_{k,j},y_{k,j})\}_{j=1}^{B_k}$ is a mini-batch of \iid samples, $\eta_k$ and $B_k$ are the learning rate and batch size, respectively. The initialization is set to $\bv_0=\boldsymbol{0}$.

Throughout, we refer to $\boldsymbol{\eta}:=(\eta_0,\eta_1,\ldots,\eta_{K-1})$ as the learning rate schedule (LRS).  
Common choices in practice include the cosine~\citep{loshchilov2016sgdr,touvron2023llama}, WSD~\citep{hu2024minicpm}, and multi-step~\citep{bi2024deepseek} schedules (see Appendix~\ref{sec:CLRS} for details).
To analyze the effect of LRS, we rewrite \eqref{eqn: sgd} as 
\begin{equation}\label{eqn: sgd2}
  \bv_{k+1}
  = \bv_{k}
  - \eta_k\bigl(\nabla\cR(\bv_{k}) + \bxi_k\bigr),
\end{equation}
where the gradient noise $\bxi_k = \tfrac{1}{B_k}\sum_{\bz\in S_k}\nabla\ell(\bz,\bv_{k})- \nabla\cR(\bv_{k})$ satisfies
$
  \EE[\bxi_k]=0, 
  \EE[\bxi_k\bxi_k^\top] = \frac{1}{B_k}\bSigma(\bv_{k}),
$
with $\bSigma(\cdot)$  denoting the noise covariance  for batch size $1$.

\begin{remark}
The hypercontractivity condition on the feature map $\bphi$ (Assumption~\ref{assumption: hypercontractivity}) ensures that
the noise covariance $\bSigma(\cdot)$ can be characterized up to multiplicative constants by an explicit analytical expression; see Lemma~\ref{lemma: noise-covariance-main}.
\end{remark}

\paragraph{Continuous-time limit.}  
Following prior work~\citep{li2017stochastic,li2019stochastic,li2020reconciling,li2021validity,pesme2021implicit,li2022what},
we analyze the continuous-time limit of SGD rather than the discrete update~\eqref{eqn: sgd} or~\eqref{eqn: sgd2}. This perspective makes the analysis more tractable and clarifies the emergence of scaling laws.
 Fix a discretization step size $h>0$  and let $\varphi_k:=\eta_k/h$ for  $k\in \NN$. Then,~\eqref{eqn: sgd2} becomes
$
    \bv_{k+1} = \bv_{k}- \varphi_{k} \nabla\cR(\bv_{k}) h - \varphi_k h \bxi_k.
$
For sufficiently small \(h\), this iteration is well approximated by the It\^{o}-type  SDE~\cite{li2019stochastic,orvieto2019continuous}:
\begin{align}\label{eqn: sde}
    \dd \bar{\bv}_\tau = - \varphi(\tau) \nabla \cR(\bar{\bv}_{\tau})\dd\tau + \varphi(\tau) \sqrt{\frac{h}{b(\tau)}\bSigma(\bar{\bv}_\tau)}\dd \bB_\tau,
\end{align}
where $\bB_{\tau}\in\RR^{M}$ is an $M$-dimensional Brownian motion, and 
\begin{itemize}
    \item $\varphi(\cdot)$ is the continuous-time LRS satisfying $\varphi(kh)=\eta_k/h$ for all $k\in \NN$;
    \item  $b(\cdot)$ is the continuous-time batch-size schedule  satisfying $b(kh) = B_k$ for all $k\in \NN$.
\end{itemize}
In~\eqref{eqn: sde}, the learning rate affects both the drift and diffusion terms, thereby coupling the deterministic and stochastic effects.

\paragraph{Intrinsic-time reparametrization.}  
In SDE~\eqref{eqn: sde}, the physical time $\tau$ serves as the continuous analogue of the discrete step index $k$.  
However, when the learning rate varies over time, the actual training progress is determined not by the number of updates $k$ but by the accumulated step size $\sum_{j=1}^k \eta_j$, which more faithfully reflects the total optimization effort.
Motivated by this observation, we introduce an \emph{intrinsic time} variable that {\it rescales} the physical time $\tau$ according to the LRS:
\begin{equation}
    t = T(\tau):=\int_0^\tau \varphi(r)\dd r,
\end{equation}
which measures the  LRS-adjusted training duration.  
Let $
    \bnu_t = \bar{\bv}_{T^{-1}(t)}.
$
Applying the \O{}ksendal's time change formula~\citep{oksendal2003stochastic} to the SDE~\eqref{eqn: sde} yields  the following \textbf{intrinsic-time SDE}:
\begin{tcolorbox}[boxrule=.7pt, boxsep=0pt, left=3pt, right=4pt, top=5pt, bottom=3pt]
  \begin{equation}\label{eqn: sde-instrinc}
  \dd \bnu_t = - \nabla \cR(\bnu_t)\dd t + \sqrt{\gamma(t)\,\bSigma(\bnu_t)}\dd \bB_t\,\, \text{ with }\,\, \gamma(t) = \frac{h \varphi(T^{-1}(t))}{b(T^{-1}(t))}.
  \end{equation}
\end{tcolorbox}
\noindent Here $\gamma(t)$ quantifies the joint effect of learning-rate and batch-size scheduling. Compared with~\eqref{eqn: sde}, the LRS dependence is absorbed from the drift and retained only in the diffusion term,  thereby {\bf decoupling the deterministic and stochastic effects}.  
This structural simplification greatly facilitates the subsequent scaling analysis.

For a clearer explanation of the connection between the discrete SGD~\eqref{eqn: sgd2} and the SDE formulations~\eqref{eqn: sde} and~\eqref{eqn: sde-instrinc}, we refer the reader to Appendix~\ref{sub:the_sde_modeling}.

\section{Intrinsic-Time Functional Scaling Laws}
\label{sec:main_results}

In this section, we present our main results on functional scaling laws.
A high-level outline of the core proof ideas is provided in Section~\ref{sec: proof_sketch}, while all technical details are deferred to Appendix~\ref{sec:proof_of_theorem_ref_thm_fsl}.

We begin with assumptions on the learning-rate schedule and model size. 
\begin{assumption}\label{ass:fsl}
Assume that Assumptions~\ref{assumption: hypercontractivity}, \ref{ass:capacity}, and~\ref{ass:source} hold, 
$M$ is sufficiently large with $N - M \gtrsim M$, 
and the LRS satisfies $\sup_{t \ge 0} \gamma(t) \le C$ for some sufficiently small constant $C > 0$.
\end{assumption}

% We shall use  $\bnu_t^{\topm}$ and $\bnu_t^{\rnd}$ denote the solution to the intrinsic-time SDE~\eqref{eqn: sde-instrinc} with top-$M$ features and random-$M$ features, respectively. 

\subsection{An Illustrative Setting and the Underlying Scaling Behavior}

\begin{tcolorbox}[boxrule=0.7pt, boxsep=0pt, left=3pt, right=3pt, top=5pt, bottom=3pt]
\begin{theorem}[Intrinsic-Time FSL, hard-regime]
\label{thm: fsl-hard-regime}
Under Assumption~\ref{ass:fsl}, let $\bnu_t$ denote the solution to the intrinsic-time SDE~\eqref{eqn: sde-instrinc} with top-$M$ features. 
Then, for any $f^*$ with difficulty $s\in (0,1-1/\beta]$ and any $\sigma \ge 0$, it holds for all $t\geq 0$  that
\begin{equation}\label{eqn: fsl0}
\EE[\cR(\bnu_t)] - \tfrac{1}{2}\sigma^2 \eqsim M^{-s\beta} + e(t) + \int_0^t \cK(t-z)[e(z)+\sigma^2]\gamma(z)\dd z,
\end{equation}
where $e(t):=(1+t)^{-s},\; \cK(t):=(1+t)^{-(2-1/\beta)}$. For the random-$M$ case, the same FSL holds with probability at least $1 - \exp\{-\Omega(M)\}$ over the randomness of $\bW$.
\end{theorem}
\end{tcolorbox}
This theorem establishes that, for hard tasks with $s \le 1 - 1/\beta$, the loss dynamics are fully characterized by the FSL~\eqref{eqn: fsl0}.
% A detailed proof is provided in Appendix~\ref{sec:proof_of_theorem_ref_thm_fsl}.
 Moreover, each term in the FSL~\eqref{eqn: fsl0} admits a clear interpretation:
   \begin{itemize}
    \item \textbf{Irreducible error: $\tfrac{1}{2}\sigma^2$.} This term is due to  label noise.
    \item \textbf{Approximation error: $M^{-s\beta}$.}  This term corresponds to the error due to finite model size,  with the scaling efficiency is determined by the task's intrinsic  difficulty $s\beta$.
   \item \textbf{Signal learning: $e(t)$.}  
     This term corresponds to learning under full-batch gradient descent,  
   capturing the  rate at which SGD extracts the signal $f^*$. Moreover, the rate depends on the task's relative  difficulty~$s$.  
   For a fixed target $f^*$ (fixed $\alpha=s\beta$), 
   increasing model capacity (smaller~$\beta$) accelerates its convergence
   since $s=\alpha/\beta$ becomes larger.

   \item \textbf{Noise accumulation: $\int_0^t \cK(t-z)[e(z)+\sigma^2]\gamma(z)\dd z$.} This term characterizes how the learning-rate and batch-size schedules 
   shape the accumulation and dissipation of stochastic noise.  
   The integrand $[e(z)+\sigma^2]\gamma(z)$ represents the instantaneous noise magnitude,
   where $e(z)$ captures mini-batch noise and $\sigma^2$ captures label noise.  
   The {\bf forgetting kernel} $\cK(\cdot)$ quantifies how noise injected at time~$z$
   still affects the loss at time~$t$.  Due to $\cK(t)\asymp t^{-(2-1/\beta)}$, 
   a higher-capacity model (smaller $\beta$) tends to forget noise more slowly. 
 \end{itemize}
Notably, the last two terms together constitute the optimization error and
two key factors govern the trade-off between the them:
\begin{itemize}[leftmargin=2em,itemsep=.1em]
\item {\bf Model capacity.} Increasing model capacity ($\beta\downarrow$) accelerates signal learning but simultaneously slows noise forgetting.
\item {\bf Learning-rate and batch-size schedules.} Smaller learning rates or larger batch sizes suppress noise injection but also shorten the intrinsic training time.  
However, sufficient intrinsic time is important: the signal-learning term requires it to effectively reduce the risk, while the noise-forgetting term relies on it to forget noise memorized in early training.  
Hence, effective schedules must balance these competing objectives—\emph{suppressing injected noise while maintaining enough intrinsic time for both learning and forgetting.}
\end{itemize}

\subsection{General Results: The Top-$M$ Feature Case}

\begin{wrapfigure}{r}{0.32\textwidth}
\vspace*{-1em}
\centering
\includegraphics[width=0.32\textwidth]{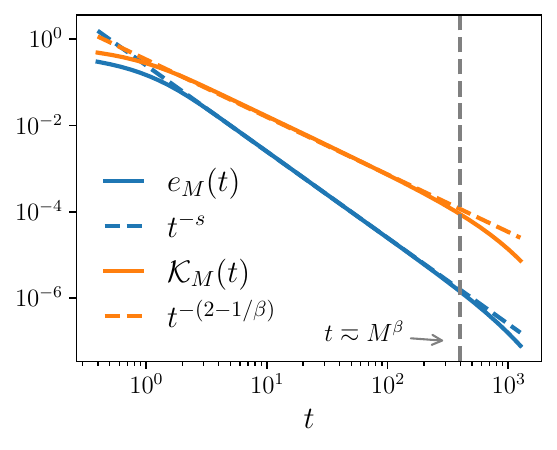}
\vspace*{-4em}
\end{wrapfigure}
The FSL~\eqref{eqn: fsl0} is established only for the hard-learning regime where $s \le 1 - 1/\beta$.  
We now show that analogous FSLs also hold in  general cases.  
To state the result, we define
\begin{equation*}
  e_M(t) = \sum_{j=1}^M \lambda_j |\theta_j^*|^2 e^{-2\lambda_j t},
  \qquad
  \cK_M(t) = \sum_{j=1}^M \lambda_j^2 e^{-2\lambda_j t}.
\end{equation*}
One can verify that both functions exhibit power-law decay for $1 \lesssim t \lesssim M^\beta$ (see the right figure and Lemma~\ref{lemma:ke-power}):
\begin{equation}
  e_M(t) \,\asymp\, t^{-s}, \qquad
  \cK_M(t) \,\asymp\, t^{-(2 - 1/\beta)}, \qquad 1 \lesssim t \lesssim M^\beta.
\end{equation}
Consequently, $e_\infty(t) \eqsim e(t)$ and $\cK_\infty(t) \eqsim \cK(t)$ for $t \geq 0$.

The following theorem provides a characterization of the loss dynamics for general case:
\begin{tcolorbox}[boxrule=0.7pt, boxsep=0pt, left=3pt, right=3pt, top=5pt, bottom=3pt]
\begin{theorem}[Intrinsic-Time FSL, top-$M$ features, general label noise]
\label{thm: fsl-general-noise}
Suppose Assumption~\ref{ass:fsl} holds. Let $\bnu_t$ denote the solution to the intrinsic-time SDE~\eqref{eqn: sde-instrinc} with the top-$M$ features. 
 Define 
$
\cF_M(t;\gamma)
  = e_M(t)
  + \int_0^t \cK_M(t-z)[e_M(z)+\sigma^2]\gamma(z)\dd z.
$
There exists a $c>0$ such that for $0\leq t \leq c M^\beta$,  it holds that 
  \begin{equation}\label{eqn: fsl}
\EE[\cR(\bnu_t)] - \tfrac{1}{2}\sigma^2
\;\eqsim\;
M^{-s\beta} + \cF_\infty(t;\gamma).
\end{equation}
For all $cM^\beta\leq t<\infty$, it holds that
  \begin{equation}\label{eqn: fsl-general}
  M^{-s\beta} + \cF_M(t;\gamma)
  \;\lesssim\;
  \EE[\cR(\bnu_t)] - \tfrac{1}{2}\sigma^2
  \;\lesssim\;
  M^{-s\beta} + \cF_\infty(t;\gamma).
  \end{equation}
Notably,  the constants implicit in~$\eqsim,\lesssim$ are independent of the noise level $\sigma$.
\end{theorem}
\end{tcolorbox}

The above characterization is \emph{uniform} with respect to the label-noise level~$\sigma$, and holds for all $s>0$ and $\beta>1$.  
It asserts that the exact FSL relation~\eqref{eqn: fsl} (i.e., the FSL~\eqref{eqn: fsl0}) remains valid up to the intrinsic time $t \le cM^{\beta} =: t_M$.  
For later times $t > t_M$, although the FSL may no longer hold exactly, the loss dynamics remain controlled from both sides as given in~\eqref{eqn: fsl-general}.  

At the critical  time $t_M$, we have $e_M(t_M) \!\asymp\! M^{-s\beta}$, indicating that signal learning has reached the approximation-error limit.  
Beyond this point, further training no longer improves the learned signal; instead, the dynamics become  dominated by noise effects.  
Depending on the interaction between the stochastic  noise and the decaying learning rate, additional training may either inject more noise or dissipate it.  
Thus, it is a priori unclear whether the total error will significantly increase or decrease after $t_M$; however, the upper bound in~\eqref{eqn: fsl-general} guarantees that the overall loss remains well-controlled, analogous to the  behavior  of the infinite-width limit ($M \to \infty$).

Nevertheless, an FSL  may still hold for all $t \ge 0$, under suitably stronger conditions.  
In Theorem~\ref{thm: fsl-hard-regime}, we considered the setting with tasks satisfying $s \le 1 - 1/\beta$.  
The following result shows that a similar characterization extends to general
task difficulty with constant label noise.

\begin{theorem}[Intrinsic-Time FSL, top-$M$ features, constant label noise]
\label{thm:fsl-const-noise}
Under Assumption~\ref{ass:fsl}, suppose $\sigma \gtrsim 1$.  
Let $\bnu_t$ denote the solution to the intrinsic-time SDE~\eqref{eqn: sde-instrinc} with the top-$M$ features.  
Then, for any $s > 0$ and all $t \ge 0$,
\[
  \EE[\cR(\bnu_t)] - \tfrac{1}{2}\sigma^2
\;\eqsim\;
  M^{-s\beta}
  + e_M(t)
  + \int_0^t \cK_M(t-z)\,[e_M(z)+\sigma^2]\,\gamma(z)\,\dd z.
\]
\end{theorem}

Theorem~\ref{thm: fsl-hard-regime} implies that the finite-$M$ functions $e_M$ and $\cK_M$ can be replaced by their infinite-width counterparts $e_\infty$ and $\cK_\infty$ in the hard-learning regime.  
The next result demonstrates that the same FSL characterization naturally extends to the noiseless case $\sigma = 0$.

\begin{theorem}[Intrinsic-Time FSL, top-$M$ features, zero label noise]
\label{thm:fsl-zero-noise}
Suppose Assumption~\ref{ass:fsl} holds and $\sigma = 0$. Let $\bnu_t$ denote the solution to the intrinsic-time SDE~\eqref{eqn: sde-instrinc} with the top-$M$ features.  
If $s \in [0, 2 - 1/\beta]$, then for all $t \ge 0$,
\[
  \EE[\cR(\bnu_t)]
  \;\eqsim\;
  M^{-s\beta}
  + e_M(t)
  + \int_0^t \cK_M(t-z)\,e_M(z)\,\gamma(z)\,\dd z.
\]
\end{theorem}

\subsection{General Results: The Random-$M$ Feature Case}
 For the random-features case, the modified feature covariance matrix is $\widehat{\bH} = \bW \bH \bW^\top$, whose eigenvalues we denote by $\widehat{\lambda}_1\ge \widehat{\lambda}_2\ge \cdots \ge \widehat{\lambda}_M$. We similarly define:
\begin{equation}
    \widehat{e}_M(t)=\sum_{j=1}^M \widehat{\lambda}_j |\theta_j^*|^2 e^{-2\widehat{\lambda}_j t}, \qquad \widehat{\cK}_M(t)=\sum_{j=1}^M \widehat{\lambda}_j^2 e^{-2\widehat{\lambda}_j t} . \nonumber
\end{equation}
The next theorem establishes that the same FSL characterization also holds when the top-$M$ features are replaced by randomly selected features.

\begin{theorem}[Intrinsic-Time FSL, random-$M$ features]
\label{thm: fsl-random}
Suppose Assumption~\ref{ass:fsl} holds and $s\in(0,1]$.  
Let $\bnu_t$ denote the solution to the intrinsic-time SDE~\eqref{eqn: sde-instrinc} with the random-$M$ features.  
Then, with probability at least $1 - \exp(-\Omega(M))$ over the randomness of the projection matrix~$\bW$,  
the results of Theorems~\ref{thm: fsl-general-noise}, \ref{thm:fsl-const-noise}, and~\ref{thm:fsl-zero-noise} continue to hold, after replacing $e_M(\cdot)$ and $\cK_M(\cdot)$ with their random-feature counterparts $\widehat{e}_M(\cdot)$ and $\widehat{\cK}_M(\cdot)$, respectively.
\end{theorem}

\begin{lemma}
\label{lem: random_eigen}
    With probability at least $1 - \exp(-\Omega(M))$ over the randomness of the projection matrix~$\bW$, it holds that $\widehat{\lambda}_j \eqsim \lambda_j \eqsim j^{-\beta}$ for any $j\in[M]$.
\end{lemma}

% The proof is deferred to Appendix~\ref{appendix:proof_random}.  
Theorem~\ref{thm: fsl-random} and Lemma~\ref{lem: random_eigen} together imply that when the task difficulty satisfies $s\le1$,  
training with random-$M$ features is similar to using the top-$M$ features, up to exponentially small probability.  
We emphasize, however, that for easier tasks with $s>1$,  
the behaviors of random and top feature may diverge and we leave this for future investigation.

\subsection{Key Proof Steps and Core Insights}
\label{sec: proof_sketch}

In this section, we outline the main ideas behind the proof of the FSL~\eqref{eqn: fsl}, highlighting the key techniques.  
Complete proofs of the above theorems are deferred to  Appendix~\ref{sec:proof_of_theorem_ref_thm_fsl}.

We will need the following characterization of the gradient noise structure.
\begin{lemma}[Noise structure]\label{lemma: noise-covariance-main}
For any $\bv \in \RR^M$, it holds that
\[
(2\rho_{-}\cE(\bv) + \sigma^2)\, \nabla^2 \cR(\bv)
\;\preceq\;
\bSigma(\bv)
\;\preceq\;
(2\rho_{+}\cE(\bv) + \sigma^2)\, \nabla^2 \cR(\bv),
\]
where $\nabla^2 \cR(\bv) = \bW\bH\bW^\top$,  
and the constants $\rho_{-}$ and $\rho_{+}$ are the same as in Assumption~\ref{assumption: hypercontractivity}.
\end{lemma}

The proof is provided in Appendix~\ref{sub:Analysis of SDE}.  
Let $\bxi(\bv)$ denote the gradient noise at $\bv$.  
Since $\cR(\bv) = \cE(\bv) + \tfrac{1}{2}\sigma^2$, it follows that for any direction $\bm{n}\in \SS^{M-1}$,
\[
\EE[|\bxi(\bv)^\top \bm{n}|^2]
    = \bm{n}^\top \bSigma(\bv)\bm{n}
    \;\eqsim\;
    \cR(\bv)\,\bm{n}^\top \nabla^2 \cR(\bv)\bm{n},
\]
where $\bm{n}^\top \nabla^2 \cR(\bv)\bm{n}$ represents the local curvature of the risk landscape along~$\bm{n}$.  
Hence, the noise energy in each direction is proportional to the product of the population risk and the curvature along that direction.  
This {\bf anisotropic structure} of the gradient noise—scaling with the risk and shaped by curvature—has also been reported in prior work~\citep{wu2022alignment,wang2023theoretical}.

For clarity, in this section, we  focus on the case of top-$M$ features, for which  the population risk takes the form
\begin{equation}\label{eqn: pop-risk-top-M}
    2\cE(\bv)=\sum_{j=1}^M \lambda_j (v_j-\theta^*_j)^2+\sum_{j=M+1}^N \lambda_j |\theta_j^*|^2.
\end{equation}
For the intrinsic-time SDE~\eqref{eqn: sde-instrinc},  each coordinate of $\bnu_t$ evolves as
\[
    \dd \nu_j(t) = - \lambda_j (\nu_j -\theta^*_j)\dd t + \sqrt{\gamma(t)} \sum_{k=1}^M\sqrt{\bSigma(\bnu_t)}_{jk} \dd B_t^{(k)},
\]
where $B_t^{(k)}$ is the $k$-th coordinate of the $M$-dimensional Browian motion $\bB_t$.

Observe that the summation of white noises is equivalent to a single stochastic term $\sqrt{q_j(t)}\dd B_j(t)$ with variance $q_j(t) := \sum_{k=1}^M |\sqrt{\bSigma(\bnu_t)}_{jk}|^2 = \|\sqrt{\bSigma(\bnu_t)} \be_j\|^2_2 = \be_j^\top \bSigma(\bnu_t) \be_j$, $\be_j$ is the $j$-th canonical basis vector for $j \in [M]$.
Therefore each coordinate of $\bnu_t$ satisfies
\[
    \dd \nu_j(t) = - \lambda_j (\nu_j -\theta^*_j)\dd t + \sqrt{\gamma(t)q_j(t)} \dd B_j(t),
\]
where $q_j(t) = \be_j^\top \bSigma(\bnu_t)\be_j$ is the variance of  gradient noise along~$\be_j$.
By applying It\^{o}'s formula to $(\nu_j -\theta^*_j)^2$ and noting $\bnu(0)=\bm{0}$, we obtain
\[
     \EE[(\nu_j(t)-\theta_j^*)^2] = (0-\theta_j^*)^2e^{-2\lambda_j t}
        +  \int_0^t  e^{-2\lambda_j(t-z)} \gamma(z) \EE[q_j(z)]\dd z.
\]
Let $\cE_t=\cE(\bnu_t)$ and plugging the above equation into \eqref{eqn: pop-risk-top-M} gives
\begin{align}\label{eqn: 123}
2\EE[\cE_t] &=    \sum_{j=1}^M \lambda_j |\theta_j^*|^2 e^{-2\lambda_j t} + 
    \sum_{j=1}^M \lambda_j\int_{0}^{t} e^{-2\lambda_j (t-z)}\gamma(z) \EE[q_j(z)] \dd z + \sum_{j=M+1}^{N}\lambda_j |\theta_{j}^{*}|^{2}.
\end{align}
By Lemma~\ref{lemma: noise-covariance-main} and noting $\nabla^2\cR(\bv)=\diag(\lambda_1,\dots,\lambda_M)$, we have 
\[
q_j(t)=\be_j^\top \bSigma(\bnu_t)\be_j \eqsim \lambda_j \cR(\bnu_t)  = \lambda_j(\cE(\bnu_t)+\sigma^2/2). 
\]
Let $\delta_M=\sum_{j=M+1}^{N}\lambda_j |\theta_{j}^{*}|^{2}$, $e_M(t)=\sum_{j=1}^M \lambda_j|\theta_j^*|^2e^{-2\lambda_j t}$, and $\cK_M(t) = \sum_{j=1}^M \lambda_j^2 e^{-2\lambda_j t}$. 
Plugging them back into~\eqref{eqn: 123} gives the following {\bf Volterra equation}:
\begin{tcolorbox}[boxrule=0.7pt, boxsep=0pt, left=3pt, right=3pt, top=5pt, bottom=3pt]
\begin{equation}\label{eqn: volterra-equation}
\EE[\cE_t] \eqsim \delta_M + e_M(t) + \int_{0}^{t} \cK_M(t-z)
    \gamma(z) (\EE[\cE_z]+\sigma^2) \dd z.
\end{equation}
\end{tcolorbox}

The above equation characterizes the expected loss dynamics of SGD under a general spectrum and has been derived in prior works such as~\citep{paquette2021sgd,paquette2021dynamics,lee2022trajectory,paquette2024homogenization,paquette2024fourplus}.  
Our key observation is that, under the power-law assumptions on $\{\theta_j^*\}_j$ and $\{\lambda_j\}_j$ (Assumptions~\ref{ass:capacity} and~\ref{ass:source}),  
the solution to~\eqref{eqn: volterra-equation} admits a \textit{sharp asymptotic characterization},  
providing explicit upper and lower bounds that precisely capture its scaling behavior.

Let $f(t):=\EE[\cE_t]$, $g(t):=\delta_M + e_M(t) + \sigma^2\int_{0}^{t} \cK_M(t-z)\gamma(z) \dd z$, and define the linear operator 
\[
 \cT f(t) = \int_{0}^{t}\cK_M(t-z)\gamma(z)f(z)\dd z.   
\]
Then, the Volterra equation~\eqref{eqn: volterra-equation} can be expressed in the compact form
$
    f = g + \cT f.
$
Formally, its solution can be expanded as an infinite series:
\begin{equation}
\label{ineq: inf_sum}
    f = (\cI - \cT)^{-1} g = g + \cT g + \cT^2 g + \cT^3 g + \cdots.
\end{equation}
The key observation is that, under Assumptions~\ref{ass:capacity} and~\ref{ass:source},  
{\bf the higher-order terms $\cT^k g$ for $k\!\ge\!2$ can be well controlled by the first-order term~$\cT g$.} This is due to the scale invariance of the forgetting kernel:
\begin{tcolorbox}[boxrule=0.7pt, boxsep=0pt, left=3pt, right=3pt, top=5pt, bottom=3pt]
\begin{lemma}[Scale invariance]
\label{lem: half-scale}
    It holds for any $t \lesssim M^{\beta}$ that  $\tfrac{\cK_M (t/2)}{\cK_M(t)} \simeq 1$.
\end{lemma}
\end{tcolorbox}
\begin{proof}
The result follows from the fact that $\cK_M$ exhibits a power-law decay for $t\!\lesssim\! M^\beta$ (see \eqref{equ: forgetting-kernel-power-law}).  
% Indeed,
% \[
% \cK_M(t)
%    = \sum_{j=1}^M \lambda_j^2 e^{-2\lambda_j t}
%    \;\eqsim\;
%    \int_{M^{-\beta}}^{1} z^{\,1-\frac{1}{\beta}} e^{-2 z t}\dd z
%    \;\eqsim\;
%    t^{-(2-1/\beta)}, 
%    \qquad 1 \lesssim t \lesssim M^{\beta}.
% \]
Consequently,
$
\cK_M(t/2)
   \eqsim
   (t/2)^{-(2-1/\beta)}
   \eqsim
   t^{-(2-1/\beta)}
   \eqsim
   \cK_M(t),
$
which establishes the claim.
\end{proof}
% By using Lemma~\ref{lem: half-scale}, we can derive the following corollary:
\begin{corollary}[Subconvolution property]
\label{cor: contraction}
    For any $t \lesssim M^{\beta}$, it holds $\cK_M*\cK_M(t)\lesssim\cK_M(t)$.
\end{corollary}
\begin{proof}
Noting  $\cK_M$ is non-increasing and integrable  
and applying Lemma~\ref{lem: half-scale}, we obtain
\begin{equation*}
\begin{aligned}
    (\cK_M * \cK_M)(t)
    &= \int_{0}^{t/2} \cK_M(t-z)\cK_M(z)\dd z
       + \int_{t/2}^{t} \cK_M(t-z)\cK_M(z)\dd z \\
    &\le
       2\,\cK_M(t/2)\!\int_{0}^{t/2}\!\cK_M(z)\dd z
       \;\lesssim\;
       \cK_M(t/2)
       \;\lesssim\;
       \cK_M(t),
\end{aligned}
\end{equation*}
which proves the claim.
\end{proof}

Let $\|\gamma\|_\infty=\sup_{t\geq 0}\gamma(t)$. By Corollary~\ref{cor: contraction}, it holds for any $t\lesssim M^\beta$ that
\begin{equation}
\begin{aligned}
    \cT ^2 g (t) 
    &\leq \|\gamma\|_\infty  \int_{0}^{t} \cK_M*\cK_M(t-z) \gamma(z) g(z) \dd z \\
    &\lesssim \|\gamma\|_\infty\int_{0}^{t} \cK_M(t-z) \gamma(z) g(z) \dd z = \|\gamma\|_\infty \cT g(t). \nonumber
\end{aligned}
\end{equation}
When $\|\gamma\|_\infty$ is sufficiently small, there exists a constant $0<c<1$ such that  $\cT^k g (t) \leq c^{k-1} \cT g(t)$ holds for any $0\leq t\lesssim M^\beta$. Hence,
\begin{equation}
    f(t)  \leq  g(t) + \cT g(t)  + \sum_{k=2}^\infty c^{k-1} \cT g(t) \lesssim g(t) + \cT g(t) + \tfrac{c}{1-c} \cT g(t). \nonumber
\end{equation}
Combining $f(t) \geq g(t) + \cT g(t)$ with the above upper bound, we conclude $f(t) \eqsim g(t) + \cT g(t)$, i.e., we have
\begin{equation}
    \EE[\cR(\bnu_t)] - \tfrac{1}{2}\sigma^2
\;\eqsim\;
  \delta_M
  + e_M(t)
  + \int_0^t \cK_M(t-z)\,[e_M(z)+\sigma^2]\,\gamma(z)\,\dd z.
\end{equation}
%, which completes the proof of FSL~\eqref{eqn: fsl}.

% \begin{lemma}[Uniform contraction]
% \label{lem: uniform-contraction}
%     There exists a constant contraction factor $0<c<1$ such that, for any $t\geq 0$, it holds $\cT ^2 g (t) \leq c \cT g(t)$.
% \end{lemma}

% By using Lemma~\ref{lem: uniform-contraction}, we have
% \begin{equation}
%     f  \approx g + \cT g + \cT^2 g + \cT^3 g + \cdots \leq g + \cT g + c \cT g + c^2 \cT g + \cdots \lesssim g + \cT g.
% \end{equation}

% \begin{lemma}[Half-scale comparability]
%     For any $t\geq 0$, it holds $\cK_M (t/2) \lesssim \cK_M(t)$.
% \end{lemma}

% \begin{equation}
%     \cK_M*\cK_M(t) = \int_{0}^{t} \cK_M(t-\tau)\cK_M(\tau) \dd \tau \leq \int_{0}^{\frac{t}{2}}\cK_M(\frac{t}{2})\cK_M(\tau) \dd \tau + \int_{\frac{t}{2}}^{t}\cK_M(\frac{t}{2})\cK_M(\tau) \dd \tau \lesssim \cK_M(t).
% \end{equation}

\paragraph{The emergence of power-law scaling.}
We now illustrate how power-law scaling arises in our setting from a multi-task learning viewpoint.
% For brevity,   consider the case of the top-$M$ features and an infinitesimal learning rate, where the SDE~\eqref{eqn: sde-instrinc} reduces to the gradient flow ODE:
% $
%     \dd \bnu_{t} = -\bW^\top\bW\bH \bnu_{t} \dd t . 
% $
% Noting that $\bW^\top\bW\bH = \operatorname{diag}\{\lambda_1,\lambda_2,\cdots,\lambda_{M},0,\cdots,0\}$ is diagonal, consequently the ODE is solvable and gives the following  expression of the excess risk:
% \begin{equation}
%     \cR(\bnu_t)-\tfrac{1}{2}\sigma^{2} \eqsim  \underbrace{\sum_{j=1}^{M}\lambda_j |\theta_{j}^{*}|^{2}e^{-2\lambda_j t} + \int_{0}^{t}}_{\text{learned sub-tasks}} +\underbrace{\sum_{j=M+1}^{N}\lambda_j |\theta_{j}^{*}|^{2}}_{\text{unlearned sub-tasks}}. \nonumber
% \end{equation} 
Intuitively,  we can view the learning of each eigenfunction as a sub-task. Due to the limited model size, student model can at most learn the top-$M$ eigenfunctions. Within this view, our FSL framework reveals three distinct manifestations of power-law behavior:
\begin{itemize}
    \item[(i)] \textbf{Approximation error.} Approximation error accounts for total risk of the $N-M$ unlearned sub-tasks, which follows a power-law decay due to 
    \begin{equation}
        \delta_M = \sum_{j=M+1}^{N}\lambda_j |\theta_{j}^{*}|^{2}\eqsim M^{-s\beta}, \text{ if } N - M \gtrsim M.
    \end{equation}
    \item[(ii)] \textbf{Signal learning.} For each sub-task, the sub-task risk converges exponentially w.r.t.~the intrinsic time $t$. However, owing to the power-law structure of $\{\lambda_j\}, \{\theta_{j}^{*}\}$, the total multi-task risk exhibits a power-law decay for sufficiently large $M$ due to 
    \begin{equation}
        e_M(t) = \sum_{j=1}^{M}\lambda_j |\theta_{j}^{*}|^{2}e^{-2\lambda_j t} \eqsim \int_{M^{-\beta}}^{1}u^{s-1}e^{-2 u t}\dd u \eqsim \frac{1}{t^{s}}, \text{ if }1 \lesssim t \lesssim M^{\beta}.
    \end{equation}
    \item[(iii)] \textbf{Noise forgetting.} Analogously, for the forgetting kernel, we also have
    \begin{equation}
    \label{equ: forgetting-kernel-power-law}
        \cK_M(t)
   = \sum_{j=1}^M \lambda_j^2 e^{-2\lambda_j t}
   \;\eqsim\;
   \int_{M^{-\beta}}^{1} z^{\,1-\frac{1}{\beta}} e^{-2 z t}\dd z
   \;\eqsim\;
   t^{-(2-1/\beta)}, 
   \text{ if }  1 \lesssim t \lesssim M^{\beta}.
    \end{equation}
\end{itemize}

\begin{tcolorbox}[boxrule=0.7pt, boxsep=0pt, left=3pt, right=3pt, top=5pt, bottom=3pt]
\begin{remark}[Task Accumulation Effect]
The above derivation reveals that, although a single task may not exhibit an exact power-law scaling with respect to intrinsic time, the \emph{collective accumulation of many tasks} gives rise to such scaling behavior. 
As the number of tasks increases, the power-law regime progressively extends, and in the idealized limit $M \to \infty$, it spans the entire training process.
In our setting, each sub-task naturally follows an exponential learning curve $\exp(-\lambda t)$, yet the underlying mechanism appears more universal. 
We anticipate that similar power-law behavior may arise even when individual task dynamics deviate from this form, which we leave for future investigation.
\end{remark}
\end{tcolorbox}

\section{Learning Rate Schedules Impact Scaling Efficiency}
\label{sec: effect-lrs}
Having established the general FSL, we now instantiate it under three representative LRSs---constant, exponential decay, and warmup–stable–decay (WSD)---to examine how schedule design influences scaling efficiency. All proofs can be found in Appendix~\ref{appendix:compute}. For clarity, we make:
\begin{assumption}\label{assump: task-difficulty}
Assume constant label noise $\sigma^2 \gtrsim 1$ and batch size $b(\tau)=B$ for all $\tau \ge 0$.
\end{assumption}

Under this assumption, given a physical-time LRS function $\varphi(\cdot)$,
Theorem~\ref{thm:fsl-const-noise} implies that the FSL for $t \gtrsim 1$ simplifies to
\begin{align}
\label{eqn: fsl-simplified}
\notag \EE[\cR(\bnu_t)] - \tfrac{1}{2}\sigma^2 
% &\eqsim M^{-s\beta}+ e(t) + \sigma^2\int_0^t \cK(t-r)\gamma(r)\dd r\\ 
% \notag &\eqsim M^{-s\beta} + t^{-s} + \int_0^t K(t-r)[e(r)+\sigma^2]\gamma(r)\dd r\\ 
 &\eqsim M^{-s\beta} + e_M(t) + \frac{\sigma^2}{B}\int_0^t \cK_M(t-r)\varphi(T^{-1}(r))\dd r.
\end{align}
% where the  $e(r)$ term in $N_t(\gamma)$  is absorbed by the full-batch GD term due to $\int_0^t \cK(t-r) e(r)\dd r\eqsim \max(t^{-s}, t^{-(2-1/\beta)})\eqsim t^{-s}$ 
Let $\cE_K = \EE[\cR(\bnu_{Kh})] - \tfrac{1}{2}\sigma^2$ denote the expected excess risk after $K$ training steps. For each LRS, we derive concrete scaling laws describing how $\cE_K$ scales with the model size $M$, the total step count $K$, as well as the LRS's hyperparameters.
We then reinterpret these results from a resource-allocation perspective by optimizing under two canonical constraints:
(i) the data-limited regime, where the total data size $D := BK$ is fixed; and
(ii) the compute-limited regime~\citep{hoffmann2022training}, where the total compute $C := MD$ is fixed.
For each regime, we further examine how the {\bf optimally tuned hyperparameters} (e.g., the peak learning rate) should scale with increasing available resources.

Finally, for clarity, we distinguish between two task regimes: an {\bf easy-learning regime}, where $s \ge 1 - 1/\beta$, and a {\bf hard-learning regime}, where $s < 1 - 1/\beta$.

\subsection{Constant LRS}

% For constant LRS, we have the following result, with the proof deferred  to Appendix \ref{sub:Proofs for The Constant LRS}.
\begin{theorem}[Scaling law for constant LRS]
\label{thm:chinchilla}
    Under Assumption~\ref{assump: task-difficulty}, for any $\eta K \gtrsim1$, we have
    \begin{equation}\label{eqn: scaling-law-const}
        \cE_K \eqsim M^{-s\beta} +(\eta K)^{-s} + \frac{\eta}{B}\sigma^{2}. 
    \end{equation}
\end{theorem}
% In this case, the noise term is simplified to $\frac{\eta}{B}\big(\sigma^{2}+(\eta K)^{-(2-\frac{1}{\beta})}\big)$. The fit-dependent component $(\eta K)^{-(2-\frac{1}{\beta})}$ is indeed dominated by the full-batch GD term $(\eta K)^{-s}$ when $s \le 2 - \frac{1}{\beta}$. In the subsequent statement, we shall omit this term for clarity.
Let $\gamma:=\eta/B$ be the {\it effective learning rate}.  Then, the scaling law can be rewritten as
\begin{equation*}
\cE_K \eqsim  M^{-s\beta} + (\gamma D)^{-s} + \gamma \sigma^{2}=:h(\gamma,M, D),
\end{equation*}
where the excess risk depends the learning rate via  $\gamma=\eta/B$. This suggests that we should scale the learning rate linearly with respect to  batch size (a.k.a. linear scaling rule)~\citep{krizhevsky2014one,goyal2017accurate,mccandlish2018empirical}.

{\bf Data-optimal scaling.} Clearly, this involves minimizing $h(\cdot)$ while keeping $D$ fixed. A straightforward calculation yields:
\begin{equation}\label{eqn: data-optim-scaling-const-lrs}
    \gamma_{\opt} \eqsim D^{-\frac{s}{s+1}},\quad M_{\opt}\gtrsim D^{\frac{1}{(1+s)\beta}},\qquad \cE_{\opt} \eqsim D^{-\frac{s}{s+1}}.
\end{equation}
Notably, both the best achievable excess risk  $\cE_{\opt}$ and optimal learning rate $\gamma_{\opt}$  depend exclusively on the task's relative difficulty  $s$. For a fixed target (fixed $\alpha$), a higher‑capacity model (smaller $\beta$) gives a larger $s=\alpha/\beta$ and is therefore more data‑efficient.

{\bf Compute-optimal scaling.}  This involves  minimizing $h(\cdot)$ while keeping $C:=DM$ fixed.
The solution is summarized as follows, with the derivation deferred to Appendix~\ref{sub:Proofs for The Constant LRS}:
\begin{align}
\gamma_{\opt} \eqsim C^{-\frac{s\beta}{1+(s+1)\beta}}, \,\, M_{\opt} \eqsim C^{\frac{1}{1+(s+1)\beta}},\,\,  D_{\opt} \eqsim C^{\frac{(s+1)\beta}{1+(s+1)\beta}},\qquad \cE_{\opt} \eqsim C^{-\frac{s\beta}{1+s\beta+{\beta}}}.
\end{align}
This shows that the performance of the compute-optimal model improves with the total compute budget $C$ in a power law. For a fixed task (\(\alpha=s\beta\) fixed), we have the following observations:

\begin{itemize}
  \item Increasing  model capacity ( $\beta\downarrow$)  enhances  compute efficiency---the extra \(\beta\) in the scaling exponent \(\frac{s\beta}{1 + s\beta + {\beta}}\) quantifies this gain.  This explains a well-known empirical observation in  LLM pre-training: Large models are more compute-efficient than small models~\cite{kaplan2020scaling,hoffmann2022training}. 
\item  The optimal learning rate $\gamma_{\mathrm{opt}}$ decreases as $C$ grows,  and  the compute-optimal allocation favors investing more in data than in model size---again consistent with current LLM pre-training practice~\citep{bi2024deepseek,shuai2024scaling,hoffmann2022training}.
\end{itemize}

Note that \cite{bordelon2024feature} also investigated compute-optimal scaling for constant LRS but assumed a fixed learning rate and no label noise.
In contrast, we consider a more realistic scenario where the learning rate is optimally tuned and the irreducible risk is present, leading to a compute-optimal scaling law that  matches empirical observations.

\subsection{Exponential Decay LRS}
  For a given number of training steps $K$~\citep{ge2019step,wu2022last}, an exponential decay (exp-decay) LRS is given by
\begin{equation*}
    \varphi(\tau) = a \exp(-\zeta \tau), \qquad \varphi(Kh) = b,
\end{equation*}
where $\zeta$ is chosen such that $\varphi(Kh) =b$. For brevity, we assume $h=1$.
Note that the two hyperparameters $a$ and $b$ specify the peak and final learning rates, respectively.

\begin{theorem}[Scaling law for exp-decay LRS]\label{thm:geometric}
Under Assumption~\ref{assump: task-difficulty}, we have
    \begin{equation}\label{eqn: exp-decay-scaling-law}
        \cE_K \eqsim  M^{-s\beta} + T^{-s} + \sigma^2 \left(\frac{b}{B}+(a-b)\frac{\min\{M, T^{1/\beta}\}}{B T}\right),
    \end{equation}
    where $T=(a-b)K/\log(a/b) \gtrsim 1$ is the total intrinsic training time.
\end{theorem}

% \paragraph*{The minimal learning rate.} Comparing the case of constant LRS, an interesting observation is that for the hard-learning scenario ($s\leq 1-1/\beta$), the optimal LR $\eta_\opt$ is constant and no need to scale with the compute budget.  This contrasts with the constant LRS, where the optimal rate decreases as the compute budget increases, consistent with empirical observations \citep{bi2024deepseek, shuai2024scaling}. One possible explanation is that, in practice, the LRS typically decays to a small but non-zero minimum value, often $\eta_{\operatorname{min}} = \eta_{\operatorname{max}}/10$ \citep{touvron2023llama,bi2024deepseek,hu2024minicpm}. In such cases, our result for constant LRS essentially provides a reasonable lower bound. This comparison suggests that the widely adopted decay-to-$1/10$ LRS may not be optimal.

Let $b = a/K$. Then the  intrinsic time becomes $T=a(K-1)/\log K$, whereas a constant LRS with step size $\eta=a$ yields $T=a K$. Thus, exp-decay LRS drives the learning rate down to as small as $a/K$, yet sacrifices only a logarithmic factor of intrinsic time compared to the constant schedule.  

{\bf Data-optimal scaling.} Let $\gamma=a/B$ be the  effective peak learning rate.
By minimizing the right hand side of \eqref{eqn: exp-decay-scaling-law} with respect to $a,b, K, B,M$ under the constraint $KB=D$ 
(see Appendix~\ref{sub:proof_for_the_exponential_decay_lrs}), We obtain $M_{\opt}=\infty$ and
\begin{itemize}
    \item If $s\geq 1-1/\beta$, then $\gamma_\opt \eqsim (D/\log D)^{-\frac{1+s\beta-\beta}{1+s\beta}}$ and $\cE_{\opt}\eqsim (D/\log D)^{-\frac{s\beta}{s\beta+1}}$.
    \item If $s<1-1/\beta$, then $\gamma_\opt\eqsim 1$ and $\cE_\opt \eqsim (D/\log D)^{-s}$.
\end{itemize}
% The hard-learning regime arises due to the  stability constraint: learning rates must remain bounded when scaling $D$ up. 
Compared with the constant LRS (see Eq.~\eqref{eqn: data-optim-scaling-const-lrs}), exp-decay LRS achieves a strictly faster decay of the excess risk, justifying  the importance of learning‐rate decay in stochastic optimization.

% According to classical results in kernel methods~\citep{spigler2020asymptotic,dieuleveut2015non,li2023asymptotic}, the minimax-optimal rate under Assumption~\eqref{assump: task-difficulty} is 
% $
% O(D^{-\frac{\alpha}{\alpha + 1}}).
% $
% In the easy-learning regime, SGD with exponentially decay LRS, upto logarithmic factors, matches this optimal rate. 
% In contrast, in the hard-learning case, SGD under the current model is suboptimal. However, for a fixed target (i.e., fixed \(\alpha\)), one can always increase model capacity (i.e., decrease \(\beta\)) such that \( \beta < \alpha + 1 \), placing the task into the easy-learning regime. This again underscores the cruciality of using high-capacity models.

{\bf Compute-optimal scaling.} A straightforward calculation (see Appendix~\ref{sub:proof_for_the_exponential_decay_lrs}) yields:
%\vspace*{-.5em}  
\begin{itemize}[leftmargin=2em,itemsep=.1em]
    \item 
    If $s\geq 1-1/\beta$, then
    $
        \gamma_{\opt} \eqsim (C/\log C)^{-\frac{1+s\beta-\beta}{2+s\beta}}, 
        M_{\opt} \eqsim (C/\log C)^{\frac{1}{2+s\beta}}, 
        D_{\opt} \eqsim C^{\frac{1+s\beta}{2+s\beta}}(\log C)^{\frac{1}{2+s\beta}},
    $
    and
    $
        \cE_{\opt}\eqsim (C/\log C)^{-\frac{s\beta}{2+s\beta}}. \nonumber
    $
    \item If $s< 1-1/\beta$, then
    $
        \gamma_{\opt} \eqsim 1,
        M_{\opt} \eqsim (C/\log C)^{\frac{1}{1+\beta}},
        D_{\opt} \eqsim C^{\frac{\beta}{1+\beta}}(\log C)^{\frac{1}{1+\beta}}, \nonumber
    $
    and
    $
        \cE_{\opt} \eqsim (C/\log C)^{-\frac{s\beta}{1+{\beta}}}. \nonumber
    $
\end{itemize}
%\vspace*{-.5em}  

In the easy-learning regime, the excess-risk rate is determined solely by the intrinsic difficulty $\alpha=s\beta$; hence, increasing model capacity alone does not lead to asymptotic gains. The compute-optimal allocation consistently favors data over model and moreover, the optimal compute split depends solely on the task's intrinsic difficulty, with ratio 
$
D_{\opt}/M_{\opt} \;\eqsim\; C^{\alpha/(2 + \alpha)}
$
\emph{decreasing} as the task becomes harder.  This implies that, for harder tasks, one should allocate more compute to increasing model size.
The optimal \(\gamma_{\opt}\) decreases with the compute budget \(C\), and for fixed \(\alpha\), higher-capacity models (\(\beta\downarrow\)) require smaller \(\gamma_{\opt}\).

In the hard-learning regime, data still dominates compute allocation, but now the optimal split depends only on model capacity, independent of the  task difficulty. Moreover, the optimal maximal learning rate remains constant ($\gamma_{\opt}\eqsim1$).
These results imply that a single, universal choice of compute split and learning rate  suffices to attain optimal scaling across all tasks satisfying
 $s < 1-1/\beta$, greatly simplifying hyperparameter tuning. Finally, in this regime,  higher‐capacity models (smaller $\beta$) become strictly more compute‐efficient, as evidenced by the excess-risk scaling exponent $-s\beta/(1+{\beta})$.

\subsection{WSD-like LRS}
We lastly turn to  consider a WSD-like LRS~\citep{zhai2022scaling,hu2024minicpm}, which comprises a $K_1$-step {\bf stable phase} followed by a $K_2$-step {\bf decay phase}, for a total  $K=K_1 + K_2$ steps, given by
\begin{equation}\label{eqn: wsd-1}
\varphi(\tau) = \begin{cases}
    a&,\text{ if }\tau\le K_1 h;\\
    a \exp(-\zeta(\tau-K_1 h))&,\text{ if } \tau > K_1 h.
\end{cases}
\end{equation}
where $\zeta$ is chosen such that $\varphi(Kh) =b$.  
For brevity, we assume $h=1$ and let $r =K_2/K$. This schedule  is thus characterized by three hyperparameters: the peak learning rate $a$, the final learning rate $b$, and the decay proportion $r$, which controls the duration of decay-phase. (The warmup phase is omitted, as it does not affect our analysis.)

\begin{theorem}[Scaling law for WSD-like LRS]
    \label{thm: wsd}
    Under Assumption~\ref{assump: task-difficulty}, we have for the LRS~\eqref{eqn: wsd-1}:
    \begin{equation}
        \cE_K \eqsim M^{-s\beta} + (T_1+T_2)^{-s} + \sigma^2 \left(\frac{b}{B}+(a-b)\frac{\min\{M, T_2^{1/\beta}\}}{B T_2}\right),
    \end{equation}
    where $T_1=a K_1 \gtrsim 1$ and $T_2 = (a-b)K_2/\log(a/b) \gtrsim 1$ denote the intrinsic training times of the stable and decay phases, respectively.
\end{theorem}
We  see that WSD-like LRS can leverage the initial stable phase to boost the intrinsic training time. For a decay proportion $r<1$, we have \( T=T_1+T_2\geq (1-r)K a \), which far exceeds the the intrinsic time $T\eqsim aK/\log K$ achieved by the pure exp-decay LRS. Consequently, WSD removes  logarithmic factors in the full-batch GD term, without altering the noise term's order as long as $r>0$. Building on this insights, we show that WSD can indeed improve the scaling efficiency, as detailed below.

{\bf Data-optimal scaling.}
    Assuming $b = a/K$, we have $M_{\opt}=\infty$ and
    \begin{itemize}
        \item If $s\geq 1-1/\beta$, then $\gamma_\opt \eqsim D^{-\frac{1+s\beta-\beta}{1+s\beta}}(\log D)^{\frac{\beta-1}{1+s\beta}}$,  $r_{\opt}\eqsim 1$ and $\cE_{\opt}\eqsim D^{-\frac{s\beta}{s\beta+1}}(\log D)^{\frac{s\beta-s}{1+s\beta}}$.
        \item If $s< 1-1/\beta$, then $\gamma_\opt\eqsim 1$,
			$r_\opt \gtrsim
			D^{\frac{s\beta+1-\beta}{\beta-1}}\log D$ and $\cE_\opt \eqsim D^{-s}$.
    \end{itemize}
Compared with the exp‐decay LRS, both regimes enjoy an additional logarithmic improvement in excess‐risk decay. In particular, for the hard‐learning regime, the logarithmic factor is completely eliminated. To achieve such improvement, the  {\bf decay-phase duration only needs to scale sublinearly with \(D\)}, as indicated by  $r_{\opt}\to 0$ as $D\to \infty$. This  aligns with the WSD practice in LLM pre-training, where the decay phase typically occupies only 10\%-20\% of the total training duration. Moreover, our theory suggests that for harder tasks, the decay fraction can be reduced even further to enhance compute efficiency.

\paragraph{Compute-optimal scaling.}
Analogous improvements hold in the compute-limited regime.
Assuming $b = a/K$ and imposing the compute constraint $MD=C$, the  compute-optimal satisfies:
%\vspace*{-.5em}
\begin{itemize}[leftmargin=2em,itemsep=.1em]
    \item 
    If $s\geq 1-1/\beta$, then
    $
        \gamma_{\opt} \eqsim (C/\log C)^{-\frac{1+s\beta-\beta}{2+s\beta}}, r_{\opt}\eqsim 1, 
        M_{\opt} \eqsim (C/\log C)^{\frac{1}{2+s\beta}}$, $
        D_{\opt} \eqsim C^{\frac{1+s\beta}{2+s\beta}}(\log C)^{\frac{1}{2+s\beta}},
    $
    and
    $
        \cE_{\opt}\eqsim C^{-\frac{s\beta}{2+s\beta}}(\log C)^{\frac{s\beta-s}{2+s\beta}}.
    $
    \item If $s< 1-1/\beta$, then
    $
        \gamma_{\opt} \eqsim 1, r_\opt \gtrsim
      D^{-\frac{\beta-1-s\beta}{\beta-1}}\log D,
        M_{\opt} \eqsim C^{\frac{1}{1+\beta}},
        D_{\opt} \eqsim C^{\frac{\beta}{1+\beta}}, \nonumber
    $
    and
    $
        \cE_{\opt} \eqsim C^{-\frac{s\beta}{1+\beta}} . \nonumber
    $
\end{itemize}

\section{Experiments}
\label{sec: experiments}

\subsection{Power-Law Kernel Regression}

While the FSL is derived in the continuous-time limit, we now verify that it also accurately captures the loss dynamics and scaling behavior of the discrete-time SGD~\eqref{eqn: sgd}.  
Specifically, we consider the PLK regression with difficulty $s=0.5$ and  capacity $\beta=4$, corresponding to a hard-learning regime and the results are shown in Figure~\ref{fig: sgd-fsl}.

\paragraph*{FSL accurately captures the loss dynamics of SGD.}
Figure~\ref{fig: sgd-fsl}(left) compares the loss dynamics of  SGD with the predictions of the FSL under three representative LRSs: cosine, WSD, and  an unconventional cyclical schedule~\cite{smith2017cyclical}.  
Across all cases, the FSL provides a remarkably accurate description of the SGD's loss evolution. Comparing the WSD and cosine schedules, we observe that the loss under WSD exhibits a slower decay during the stable phase but undergoes a much sharper  drop once the decay phase begins, ultimately yielding a lower final loss.
This seemingly counterintuitive two-phase dynamical behavior of WSD aligns well with empirical observations in practical LLM pre-training~\citep{hu2024minicpm,wen2024understanding,team2025kimi}. 

\paragraph*{FSL predicts the scaling behavior of SGD.}
Figure~\ref{fig: sgd-fsl}(right) further validates the scaling laws derived in Section~\ref{sec: effect-lrs} for the three canonical LRSs—constant, exponential, and WSD-like~\eqref{eqn: wsd-1}.  
The results show that the final-step loss of SGD closely follows the theoretical predictions of FSL.  
Among these schedules, WSD yields the best scaling performance, followed by exponential decay, while the constant schedule performs the worst.  

More experiment details and additional results experiments with varying $(s, \beta)$ and other LRSs are provided in Appendix~\ref{sub:the_kernel_teacher_student_setup}, and exhibit consistent behaviors.

\subsection{LLM Pre-training}
We now evaluate the practical utility of  FSL as a surrogate model for capturing the loss dynamics of LLM pre‑training.  Specifically, three popular LRSs: cosine, WSD, and the 8‑1‑1~\citep{bi2024deepseek} are considered; see Figure~\ref{fig: optimalr-moe}(left) for a visualization.  In the 8‑1‑1 LRS, the learning rate is reduced by a factor $\sqrt{10}$ at 80\% and 90\% of the total token budget, yielding a final value that is $0.1$ times the peak learning rate. For more experiment details, we refer to  Appendix~\ref{appendix:llm_pre-training}.

    \begin{figure}[t]
    \centering
    \subfloat[\small Fitting and prediction using FSL.]{
        \hspace*{-1em}
        \includegraphics[width=0.24\textwidth]{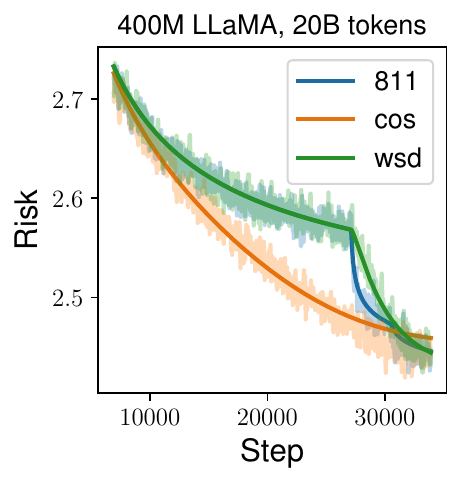}
        \hspace*{-.5em}
        \includegraphics[width=0.24\textwidth]{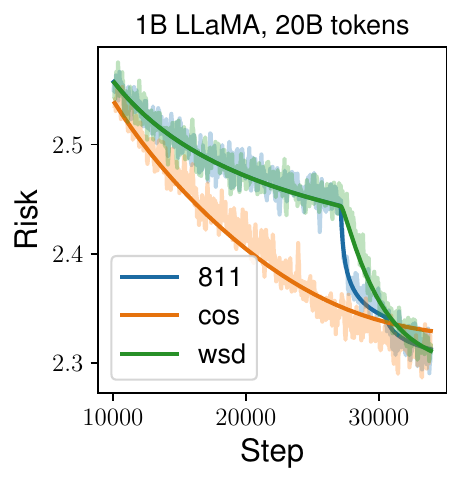}
        \label{fig: llm-fsl}
    }
    \hspace*{-1em}
    \subfloat[\small  The FSL-optimal LRS and its performance]{
        \includegraphics[width=0.26\textwidth]{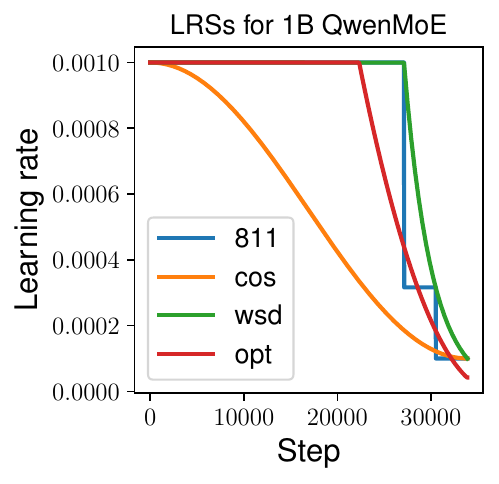}
        \hspace*{-.8em}
        \includegraphics[width=0.247\textwidth]{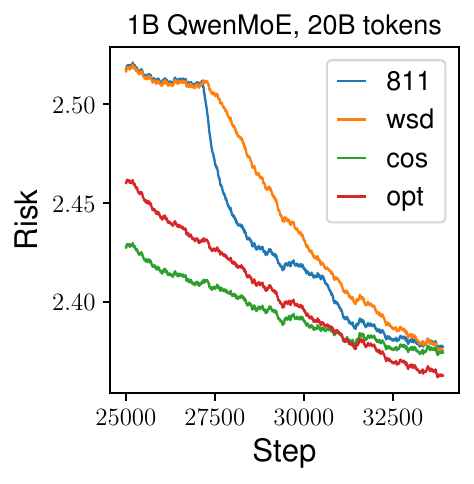}
        \label{fig: optimalr-moe}
    }
    %\vspace*{-.5em}
    \caption{\textbf{Experiment on LLMs.}  \textbf{(a)} Fitting and predictive accuracy of the FSL on dense LLaMA models. \textbf{(b)} Left: comparison of various LRSs. Right: loss trajectories of the FSL-optimal schedule versus baseline LRSs on a 1B QwenMoE model.
    }
\end{figure}

\paragraph*{FSL accurately fit and predict loss curves.}
We first quantify the descriptive and predictive power of FSL. Following the protocol of \cite{tissue2024scaling} and \cite{luo2024a}, we restrict attention to the post‑warmup portion of the loss trajectory. Two Llama~\citep{touvron2023llama} models (400\,M and 1\,B) are trained on 20\,B tokens under the three LRSs.
For each model we (i) fit the FSL parameters on the loss curve obtained using the 8‑1‑1 LRS and (ii) deploy the fitted FSL to \emph{predict} the loss curves of the cosine and WSD schedules. Figure~\ref{fig: llm-fsl} demonstrates that FSL not only fits the 8‑1‑1 trajectory accurately but also generalizes reliably to the unseen WSD and cosine schedules for both model sizes.

\paragraph*{The FSL‑optimal LRS is WSD-like.}
We next leverage the fitted FSL to \emph{design} improved LRSs. Specifically, we numerically minimize the final-step loss over the space of LRSs using the fitted FSL. This experiment employs a 1B-parameter QwenMoE model~\citep{yang2024qwen25}, trained on 20B tokens using the same three  LRSs. We fit the FSL using the trajectory from the 8-1-1 LRS and numerically solve for the FSL-optimal LRS. The model is then trained under this FSL-optimal LRS, using the same compute budget, and compared against the baseline LRSs. 

Figure~\ref{fig: optimalr-moe}(left) shows that surprisingly, 
the FSL-optimal LRS is WSD-like and the decay phase drives the learning rate far below the conventional $0.1\eta_{\max}$ threshold. This echos recent empirical recommendations by \cite{bergsma2025straight,hagele2024scaling,luo2024a}. Furthermore,  Figure~\ref{fig: optimalr-moe}(right) demonstrates that the FSL‑optimal schedule yields a strictly lower final loss than all baselines, substantiating its practical relevance.

Taken together, these results suggest that FSL offers a faithful  surrogate model for studying LLM training dynamics and holds promise as a principled tool for interpreting and designing LRSs in large-scale pre-training.

\section{Concluding Remarks}
\label{sec:Concluding Remarks}

In this paper, we present a systematic study of how LRS influences scaling laws by considering  a power-law kernel regression problem. Our theoretical framework yields a novel functional scaling law, which explicitly characterizes the impact of both learning rate and batch size schedules via a convolution-type functional term. The utility of our FSL is demonstrated through detailed analyses of three widely used LRSs, providing theoretical justification for several prevailing practices in LLM pre-training---most notably, offering an explanation for the effectiveness of the empirically popular but previously less-understood WSD schedules.

Looking ahead, several promising directions remain open. 
First, it is essential to develop a theoretical characterization of the optimal LRS under various resource constraints, which remains largely unresolved.
Second, leveraging FSL studies the effect of batch size scheduling---a relatively under-explored topic despite its practical importance. 
Third, to make FSL truly practical, it is crucial to  refine our FSL through extensive large-scale LLM pre-training experiments. 

% In particular, extending FSL to accommodate adaptive optimizers such as Adam poses an urgent challenge: Adam introduces implicit learning dynamics that substantially change how learning rate and batch size interact, potentially leading to different scaling behavior compared to SGD. Bridging this gap will be crucial for translating FSL theory into practice.

\section*{Acknowledgement}
Lei Wu is supported by the National Natural Science Foundation of China (NSFC12522120, NSFC92470122, and NSFC12288101). Binghui Li is supported by the Elite Ph.D.~Program in Applied Mathematics  at Peking University. We are grateful to Kaifeng Lyu, Kairong Luo, and Haodong Wen for generously sharing their work~\cite{luo2024a}, which greatly  inspired this study. We also thank Tingkai Yan, Yuhao Liu, Yunze Wu, and Zean Xu for many helpful discussions, and the anonymous reviewers for their valuable feedback.

\bibliography{ref}
\bibliographystyle{plain}

% \newpage

% \begin{appendices}
% \setcounter{tocdepth}{2}
%     \part{}
%     \localtableofcontents
%     \newpage

\newpage
\appendix
\addcontentsline{toc}{section}{Appendix} 
\part{Appendix} 
\parttoc 

\newpage

\section{Miscellanea}
\label{appendix:related_work}

\subsection{Empirical Fitting of LLM Pre-training Loss Trajectory} 
\label{sub:empirical_fitting_of_llm_pre_training_loss_trajectory}

% Scaling laws have emerged as a foundational principle in understanding the interplay between model size, dataset size, and performance in large-scale deep model pre-training. Initially proposed by \citep{kaplan2020scaling}, these laws have been further refined through a series of subsequent studies \citep{henighan2020scaling, hoffmann2022training, zhai2022scaling, kadra2023power, aghajanyan2023scaling, bi2024deepseek, shuai2024scaling, kumar2024scaling, tissue2024scaling, luo2024a}.

% Among these, 
The Chinchilla Law~\citep{hoffmann2022training}  describes the final-step loss $\cL$ as follows: 
\begin{equation} 
L(M,D) = L_0 + A_1 M^{-\kappa_1}  + A_2 D^{-\kappa_2} , \nonumber 
\end{equation} 
where $L_0$, $A_1$, $A_2$, $\kappa_1$, and $\kappa_2$ are constants, and $D$ and $M$ represent the amount of training data (tokens) and model size (number of parameters), respectively. 

Later, \citep{tissue2024scaling} proposed the {\bf Momentum Law}, a heuristic rule designed to capture the full loss trajectory. Given a learning rate schedule $\bm{\eta}:=\{\eta_j\}_j$, the loss at the \( k \)-th step is modeled as
\begin{align*}
\cL_k(\bm{\eta}) &= L_0 + A S_1^{-\kappa} - C S_2 , 
\end{align*}
where 
\[
  S_1 = \sum_{i=1}^{k}\eta_i, \quad S_2 = \sum_{i=2}^{k}\sum_{j=2}^{i}(\eta_{j-1} - \eta_{j}) \lambda^{i-j}.
\]
Here \( L_0 \), \( A \), \( C \), and \( \kappa \) are constants, and
$\lambda\in (0,1)$ is a hyperparameter representing the decay factor for learning rate annealing, which typically ranges from $0.99$ to $0.999$.

Subsequently, \citep{luo2024a} proposed the {\bf Multi-Power Law} (MPL),  which  replaces the $S_2$ in the Momentum Law with additional power laws to better capture the progressive loss reduction induced by learning-rate decay. Specifically, the MPL takes the following form:
\begin{equation}
\begin{aligned}
    &\cL_k(\bm{\eta}) = L_0 + A S_1^{-\kappa} - \operatorname{LD}(k),
\end{aligned}
\end{equation}
where
\[
    \operatorname{LD}(k) := C\sum_{i=2}^{k}(\eta_{i-1}-\eta_{i})G(\eta_{i}^{-\kappa'}S_i), \quad S_i := \sum_{j=1}^{i}\eta_{j}, \quad G(x) := 1 - (C'x+1)^{-\kappa''}. \nonumber
\]
Here $L_0, A, C, C', \kappa, \kappa', \kappa''$ are constants.

% However, both the Momentum Law and the Multi-Power Law remain empirical, lacking a rigorous theoretical foundation. In our work, we  theoretically derive a novel functional-form scaling law with arbitrary LRS, which also generalizes well for unseen larger model sizes.

\subsection{Popular Learning Rate Schedules in LLM Pre-training}
Here, we introduce some widely used LRSs in the context of LLM pre-training.
\label{sec:CLRS}

\begin{itemize}
    \item 
    \textbf{Cosine Schedule \citep{loshchilov2016sgdr}.} The schedule is given by $\eta_{k}=\frac{1+\rho}{2}\eta_{\operatorname{max}}+\frac{1-\rho}{2}\eta_{\operatorname{max}}\operatorname{cos}(\frac{k-1}{K-1})$, where $\eta_{\operatorname{max}}$ is the maximum learning rate and the hyper-parameter $\rho$ is usually chosen as $0.1$ such that the minimum learning rate is $\eta_{\operatorname{max}}/10$~\citep{touvron2023llama}.
    \item 
    \textbf{Warmup-Stable-Decay (WSD) Schedule \citep{zhai2022scaling,hu2024minicpm,hagele2024scaling}.} The schedule consists of three phases: a warm-up phase of \( K_{\operatorname{warm-up}} \) steps, followed by a stable phase maintaining the learning rate \( \eta_{k} = \eta_{\operatorname{max}} \), and finally a decay phase governed by \( \eta_{k} = h(k - K_{\operatorname{stable}})\eta_{\operatorname{max}} \) for \( K_{\operatorname{stable}} \leq k \leq K \), where \( K_{\operatorname{stable}} \) represents the total duration of the first two phases. Here, the decay function \( h(\cdot) \in (0,1) \) can be linear or exponential.
    \item 
    \textbf{Multi-Step Schedule \citep{bi2024deepseek}.} The entire schedule is divided into \( S \) stages, i.e., \( [K_0, K_1] \cup [K_1, K_2] \cup \cdots \cup [K_{S-1}, K_S] = [0, K] \), where \( 0 = K_0 < K_1 < \cdots < K_S = K \). The schedule satisfies that $\eta_{k} = \eta_{K_i}$ for $K_{i-1} < k \leq K_{i}$ ($1 \leq i \leq S$). In our LLM experiments, we consider a 8-1-1 LRS, corresponding to the case where $S=3$ with $\eta_{K_1}=\eta_{\max}, \eta_{K_2}=\eta_{\max}/\sqrt{10}$, and $\eta_{K_3}=\eta_{\max}/10$, and $K_1=0.8 K, K_2=0.9K$.
\end{itemize}

\subsection{Connections to Kernel Regression}
\label{appendix_rkhs}
% Our PLK regression is closely related to the classical kernel regression. In kernel regression, we aim to learn a function within a reproducing kernel Hilbert space (RKHS) from finite data samples. We begin by reviewing the foundational setup of kernel regression and subsequently show the connections between our assumptions and the classical ones.

% Let $\cX$ be a compact topological space. 
% Let $\rho$ be a probability measure on
% $\cX\times \RR$. We denote by $\rho_x$ the marginal distribution on $\cX$. 

In this section, we explain how our setup in Section~\ref{sec:problem_setup} are equivalent to learning with kernels. 

\begin{definition}[Positive semidefinite (PSD) kernel]
A function $K:\cX\times \cX \to \RR$ is called a \emph{positive semidefinite (PSD) kernel} if it satisfies:
\begin{itemize}
  \item Symmetry: $K(\bx,\bx') = K(\bx',\bx)$ for all $\bx,\bx' \in \cX$;
  \item Positive semidefiniteness: for any $\bx_1,\dots,\bx_n \in \cX$ and $a_1,\dots,a_n \in \RR$,  
  \[
    \sum_{i,j=1}^n a_i a_j\, K(\bx_i,\bx_j) \ge 0.
  \]
\end{itemize}
\end{definition}

\begin{definition}[Reproducing kernel Hilbert space (RKHS)]
Given a kernel $K:\cX\times\cX\to\RR$, the \emph{reproducing kernel Hilbert space} $\cH_K$ associated with $K$ is a Hilbert space of functions over $\cX$ such that
\[
  \langle f, K(\bx, \cdot) \rangle_{\cH_K} = f(\bx),
  \quad \forall f \in \cH_K,~\bx \in \cX.
\]
\end{definition}
One can show that for a given kernel $K$, the associated RKHS is unique.
Kernel methods consider the hypothesis space given by the associated RKHS. For instance, kernel ridge regression gives estimator:
\[
  \widehat{f}_\lambda = \argmin_{f\in \cH_K} \frac{1}{n}\sumin (f(\bx_i)-y_i)^2 + \lambda \|f\|_{\cH_K}^2.
\]
Hence, model capacity is determined by the size of the RKHS $\cH_K$.

Let $\cD$ be the input distribution. We define the integral operator $\cT_K: L^2(\cD)\to L^2(\cD)$ by 
\[
  \cT_K f(\cdot) = \EE_{\bx\sim\cD} [K(\cdot,\bx) f(\bx)].
\]
By assuming
$\EE_{\bx\sim\cD} [K(\bx,\bx)]< \infty$, 
the operator $\cT_K$ is compact (Mercer's theorem) and consequently, the kernel admits the following eigenvalue decomposition
\[
K(\bx, \bx') = \sum_{j=1}^{\infty} \lambda_j e_j(\bx)e_j(\bx'),
\]
where $\{\lambda_j\}_{j=1}^\infty$ and $\{e_j\}_{j=1}^\infty$ denotes the eigenvalues and eigenfunctions, respectively. Moreover, $\langle e_i,e_j\rangle_{L^2(\cD)}=\delta_{i,j}$, i.e., the eigenfunctions form an orthonormal basis of $L^2(\cD)$. 

Using the spectral decomposition, the RKHS admits the following representation:
\[
\mathcal{H}_K = \left\{\sum_{j=1}^{\infty} a_j e_j: \sum_{j=1}^{\infty} \frac{a_j^2}{\lambda_j} <
\infty\right\}.
\]
% with $\{\phi_j\}_{j=1}^\infty$ being an orthonormal basis of $\mathcal{H}_K$.
% Also we have $\EE [\phi_j(x)^2] = \lambda_j \EE [e_j(x)^2] = \lambda_j$,
% and $\{\phi_j\}_{j=1}^\infty$ is still orthogonal in $L^2_\rho(\mathcal{X})$.
To better quantify the smoothness of functions, one often consider the interpolation space $\mathcal{H}_K^s$ with $s \ge 0$, defined as
\[
\mathcal{H}_K^s 
= \left\{
  \sum_{j=1}^{\infty} a_j e_j 
  : 
  \sum_{j=1}^{\infty} \frac{a_j^2}{\lambda_j^{s}} < \infty
\right\}.
\]
Clearly, $\mathcal{H}_K^1 = \mathcal{H}_K$, and
\[
\mathcal{H}_K^{s_1} \subset \mathcal{H}_K^{s_2}, 
\qquad \forall\, s_1 > s_2 \ge 0.
\]
Hence, the index $s$ characterizes the smoothness of a function relative to the chosen kernel.

In the analysis of kernel methods, the following conditions are commonly used to describe the smoothness of the target function and the capacity of the kernel, respectively.
\begin{assumption}[Source condition]
\label{assump:source-kernel}
There exists some $s>0$ such that $f^* \in \mathcal{H}_K^s$.
\end{assumption}

\begin{assumption}[Capacity condition]
\label{assump:capacity-kernel}
There exists some $\beta>1$ such that $\lambda_j \eqsim j^{-\beta}$.
\end{assumption}

These conditions yield the following interpretation:
\begin{itemize}
  \item A smaller $s$ indicates that the target function $f^*$ belongs to a larger space, corresponding to a more difficult learning problem.
  \item A smaller $\beta$ implies a slower eigenvalue decay, meaning a richer hypothesis space $\mathcal{H}_K$ and thus higher model capacity.
\end{itemize}

Our formulation in Section~\ref{sec:problem_setup} is equivalent to the above setting, but expressed in terms of the feature map $\bphi$.  
Under Assumption~\ref{ass:diagonal}, we have
\[
K_\phi(\bx,\bx') 
= \sum_{j=1}^N \phi_j(\bx)\phi_j(\bx') 
= \sum_{j=1}^N \lambda_j\, \widehat{\phi}_j(\bx)\widehat{\phi}_j(\bx').
\]
In this case, Assumption~\ref{ass:capacity} corresponds exactly to the above capacity condition, while the task-difficulty assumption in Section~\ref{sec:problem_setup} can be viewed as a power-law version of the source condition.  
Specifically, under Assumption~\ref{ass:source},
\[
  f^* 
  = \sum_{j=1}^N \theta_j^* \phi_j
  = \sum_{j=1}^N j^{-1/2}\lambda_j^{\,s}\, \widehat{\phi}_j
  =: \sum_{j=1}^N a_j^*\, \widehat{\phi}_j.
\]
Hence, for any arbitrarily small $\delta \in (0,s)$, we have $f^* \in \mathcal{H}_{K_\phi}^{s-\delta}$, since
\[
  \sum_{j=1}^N \frac{|a_j^*|^2}{\lambda_j^{s-\delta}}
  = \sum_{j=1}^N j^{-1-\beta(s-\delta)} 
  < \infty.
\]

\subsection{The SDE Modeling}
\label{sub:the_sde_modeling}

\paragraph*{The physical-time SDE.}
In our setup, the SGD update can be written as
\[
\bv_{k+1} = \bv_{k} - \varphi_k \nabla \mathcal{R}(\bv_{k}) h
- \varphi_k h \bxi_k.
\]
The term $\bxi_k$ is the gradient noise, whose covariance
is $\frac{1}{B_k}\Sigma(\bv_{k})$. By assuming the gradient noise to be Gaussian, the SGD becomes
\[
\bv_{k+1} - \bv_{k} = -\varphi_k \nabla \mathcal{R}(\bv_{k}) h
+ \varphi_k \sqrt{h}\,\, \mathcal{N}\left(0, \tfrac{h}{B_k} \Sigma(\bv_{k})\right).
\]
It is exactly the Euler–Maruyama discretization of the It\^o-type SDE:
\[
\dd \bar{\bv}_\tau = - \varphi(\tau) \nabla \mathcal{R}(\bar{\bv}_\tau)
\dd t + \varphi(\tau) \sqrt{\frac{h}{b(\tau)} \Sigma(\bar{\bv}_\tau)}
\dd \mathbf{B}_\tau,
\]
where $\mathbf{B}_\tau \in \mathbb{R}^M$ denotes the $M$-dimensional Brownian
motion, and $\varphi(\cdot)$, $b(\cdot)$ are the continuous version of LRS function and
batch-size schedule function, respectively.

\paragraph*{The intrinsic-time SDE.}
Intuitively, the discrete update~\eqref{eqn: sgd2} can be  viewed as the Euler–Maruyama discretization of  SDE
\eqref{eqn: sde-instrinc} on the {\it non-uniform} grid \(\{t_k=\sum_{j=0}^{k}\eta_j\}_{k\in\NN}\) where the effective step size is $\Delta t_k=\eta_k$:
$$
    \bv_{k+1} - \bv_{k} = - \nabla \cR(\bv_{k}) (t_{k+1}-t_{k}) - \sqrt{t_{k+1}-t_{k}}\;\cN\left(0, \tfrac{\eta_k}{B_k}\Sigma(\bv_k)\right).
$$

\section{Experiment Details and Additional Results}
\label{appendix:experiment_details}

In this section, we present the details of our experiments as well as additional results.

\subsection{Power-Law  Kernel Regression}
\label{sub:the_kernel_teacher_student_setup}

\paragraph*{Physical-time FSL.}
The FSL~\eqref{eqn: fsl} is presented in terms of intrinsic time, but in practice, it is often more convenient to use physical time (training steps). By a suitable change of times, after $\tau$ steps (equivalently, $\tau/h$ discrete steps), the FSL maintains the form \eqref{eqn: fsl}, with adjustments:
\begin{align*}
    t^{-s} &= T(\tau)^{-s}, \\ 
    \cN(\varphi,b)&= \int_0^{\tau} \cK(T(\tau)-T(u))\, (e(T(u))+\sigma^2)\frac{h\varphi(u)^2}{b(u)}\dd u.
\end{align*}

\paragraph*{Fitting FSL on SGD Average-Risks.}
To validate that the Functional Scaling law (FSL) can accurately
capture the risk curve of SGD, we conducted a
series of SGD experiments under different configurations of $s$ and $\beta$.
Subsequently, we fitted the FSL to these risk curves.
Our results demonstrate that FSL indeed provides a close fit to the SGD
trajectories.

In each experiment, we adopt a PLKR  configuration with
$M=N=128$, $\sigma = 3$ and employ the top-$M$ projection matrix, thereby eliminating the
approximation error term $M^{-s\beta}$. We explore a range of values for
$s\in [0.5, 1]$ and $\beta\in [1.5, 5]$, encompassing both easy- ($s\ge 1-1/\beta$) and
hard-learning ($s< 1-1/\beta$) regimes. For each parameter configuration, we execute 200
independent SGD runs with a batch size of 1 over 10,000 steps. The resulting
average trajectory across these runs serves as the fitting target. The FSL
fitting is performed using the physical-time FSL formulation.
\[
\mathcal{E}_k = c_1 T(k)^{-s} + c_2 \sum_{i=1}^{k} \mathcal{K}(T(k)-T(i))
e(T(i)) \eta_i^2 + c_3 \sigma^2\sum_{i=1}^{k} \mathcal{K}(T(k)-T(i))
 \eta_i^2,
\]
where $c_1$, $c_2$, $c_3$ are constants to fit, $\{\eta_i\}_{i=1}^k$ is the
learning rate schedule, and $T(i) = \sum_{j=1}^{i} \eta_j$.

When fitting the SGD trajectory, we minimize the mean squared error (MSE) between the
empirical risk trajectory of SGD (without the irreducible risk
$\frac{\sigma^2}{2}$), denoted by $\mathcal{E}_{\text{SGD}}(k)$, and
the theoretical prediction from FSL, $\mathcal{E}_k$. Formally, we solve the
following optimization problem:
\[
\min_{c_1,c_2,c_3} \frac{1}{K} \sum_{k=1}^{K} \left(\mathcal{E}_{\text{SGD}}(k) - \mathcal{E}_k\right)^2,
\]
where $K$ represents the total number of training steps. This minimization is
performed using ordinary least squares (OLS), with the integrals in the FSL
expression $\mathcal{E}_k$ evaluated numerically via quadrature methods.

We display the learning rate schedules (LRSs) used in the SGD experiments in the top-left panel of Figure~\ref{fig: fsl-fit-sgd}. Complementing Figure~\ref{fig: sgd-fsl} (middle and right), additional experimental results for various values of $s$ and $\beta$ are presented in Figure~\ref{fig: fsl-fit-sgd}.
\begin{figure}[ht]
\centering 
\includegraphics[width=0.3\textwidth]{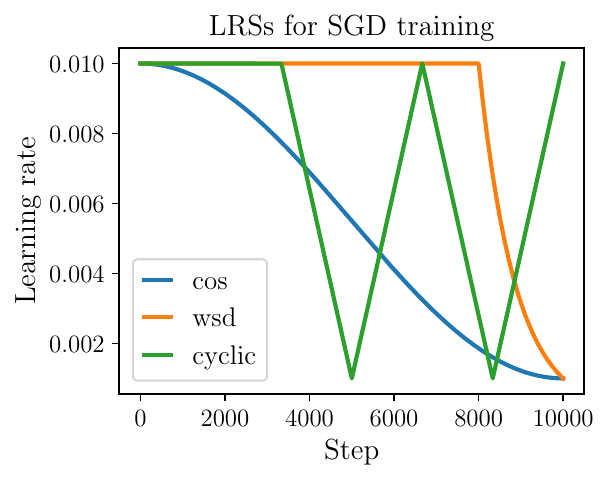}
\includegraphics[width=0.3\textwidth]{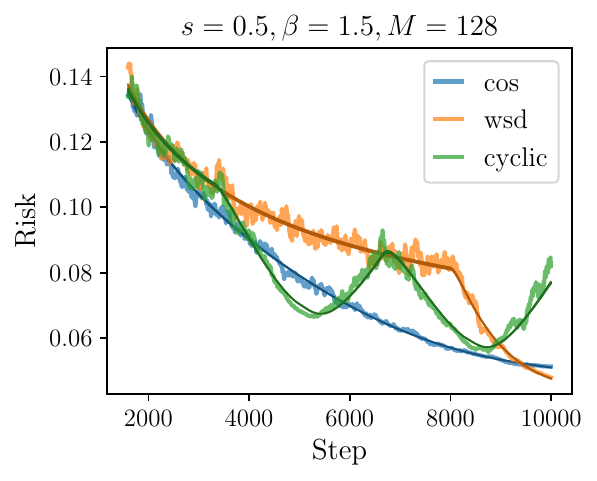}
\includegraphics[width=0.3\textwidth]{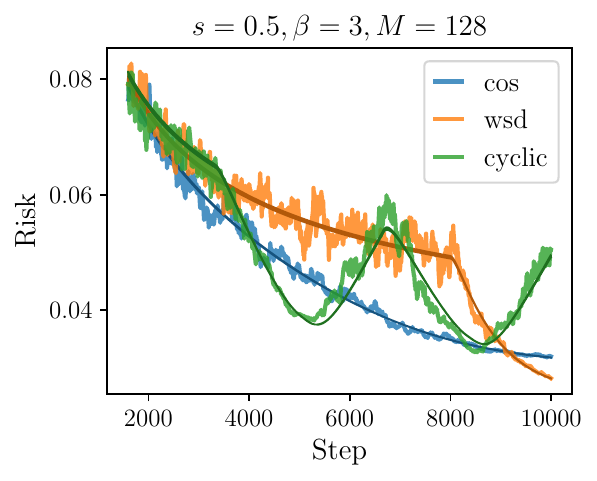}
\includegraphics[width=0.3\textwidth]{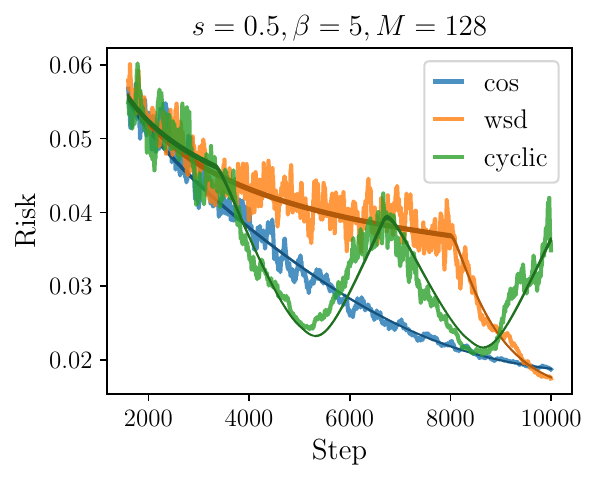}
\includegraphics[width=0.3\textwidth]{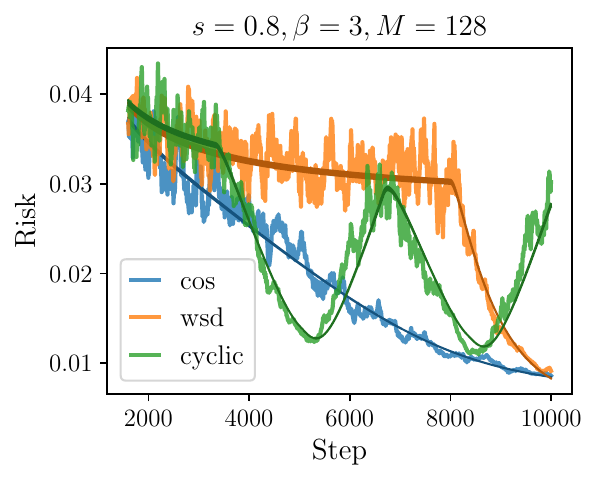}
\includegraphics[width=0.3\textwidth]{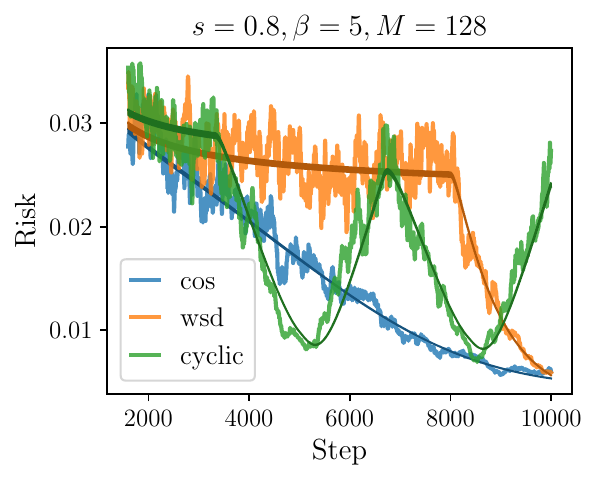}
\caption{\small \textbf{Fitting results of FSL on SGD trajectories.}
	The shaded curves are the average over 200 independent
	SGD runs, while the solid curves show the predictions of FSL.
}
\label{fig: fsl-fit-sgd}
\end{figure}

\paragraph*{Scaling law experiments.}  
These experiments are designed to evaluate the correctness of the scaling laws predicted by our analytical analysis.  
To this end, we conduct two complementary sets of experiments:

\begin{itemize}
  \item \textbf{FSL experiments.}   This experiment is intended to validate the theoretical predictions derived from the FSL. We compute the predicted risk by numerically discretizing the FSL~\eqref{eqn: fsl}, with all untracked constants set to 1.  
  For each LRS, following the theoretical analysis, we set \( \eta_{\max} = 0.05 D^{-r} \), where \( r = s / (1 + s) \) for the constant learning rate schedule, and \( r = 1 \) for exponential decay and WSD schedules.  
  We fix the batch size to \( B = 1 \); thus, for each data budget \( D \), we compute the intrinsic time and evaluate the final-step loss using the discretized FSL.

  \item \textbf{SGD experiments.}  This experiment serves to assess whether the scaling behavior predicted by our continuous-time FSL faithfully captures that of discrete-time SGD.
  We simulate stochastic gradient descent (SGD) with 200 independent trajectories and a fixed batch size \( B = 1 \).  
  For each data budget \( D \), the maximum learning rate is set as \( \eta_{\max} = 0.05 D^{-r} \), using the same theoretical values of \( r \) as in the FSL experiments.  
  We run SGD for \( D \) steps under each corresponding LRS and record the final-step excess risk.  
  
\end{itemize}

\subsection{LLM pre-training}
\label{appendix:llm_pre-training}

\paragraph*{Practical FSL Ansatz for LLM pre-training}
In this section, in order to fit real
LLM pre-training loss curves, we will derive an approximation form of the FSL in
Theorem \ref{thm: fsl-hard-regime}.

First, by the physical-time for of the FSL~\eqref{eqn: fsl} with $h=1$ and $B(u)\equiv B$, we have
\[
\cE_k \eqsim \frac{1}{T(k)^s}
+ M^{-s\beta} + \frac{1}{B} \int_{0}^{k} \mathcal{K}(
T(k)-T(u)) (\sigma^2 + e(T(u))) \cdot \varphi(u)^2\dd u.
\]
Here we focus on the integral term.
Since $\varphi(u) = \int_{0}^{u} \varphi'(r)\dd r + \varphi(0)$, and that
\begin{equation}
\label{eq:numerical_const_term}
\int_{0}^{k} \mathcal{K}(T(k)-T(u))(\sigma^2 + e(T(k))
\cdot \varphi(u)\varphi(0)\dd u
= \varphi(0) \int_{0}^{T(k)} \mathcal{K}(T(k)-t)(\sigma^2 + e(t))\dd
t.
\end{equation}
Note that this is exactly the SGD noise term at the constant LRS
$\eta(0)$ for a total intrinsic-time $T(\tau)$.
By results of constant LRS (as seen in the proof of Theorem \ref{thm:const_lrs_law}), we have
\[
\eqref{eq:numerical_const_term} \eqsim \varphi(0)(\sigma^2 + e(T(k))).
\]

As $\varphi(0) \lesssim 1$, we have
\begin{equation}
\label{eq:numerical_LRD_form}
\cE_k \eqsim \frac{1}{T(k)^s}
+ M^{-s\beta} - \mathrm{LRD}(k),
\end{equation}
where
\begin{align*}
\mathrm{LRD}(k) &:= -\frac{1}{B} \int_{0}^{k} \mathcal{K}(T(k)-T(u))
(\sigma^2 + e(T(u))) \varphi(u) \int_{0}^{u} \varphi'(r)\dd r \dd u\\
&= -\frac{1}{B} \int_{0}^{k} \varphi'(r)\int_{r}^{k} \mathcal{K}(T(k)-T(u))
(\sigma^2 + e(T(u))) \varphi(u) \dd u \dd r\\
&= -\frac{1}{B} \int_{0}^{k} \varphi'(r)\int_{T(r)}^{T(k)}
\mathcal{K}(T(k)-t) (\sigma^2 + e(t)) \dd t \dd r.
\end{align*}

We discretize the outer integral at integer nodes $r = 0, 1, \dots, k$,
\[
\mathrm{LRD}(k) \approx \frac{1}{B} \sum_{i=1}^{k} (\eta_{i-1}-\eta_i)
\int_{T(i)}^{T(k)} \mathcal{K}(T(k)-t)(\sigma^2 + e(t))\dd t.
\]
By the integral mean value theorem, we can take $(\sigma^2 + e(t))$
outside the integral, which gives
\[
\mathrm{LRD}(k) \approx \frac{1}{B} \sum_{i=1}^{k} (\eta_{i-1} - \eta_i)
(\sigma^2 + e(\xi_i)) \int_{T(i)}^{T(k)} \mathcal{K}(T(k)-t)\dd t,
\]
where $\xi_i \in [T(i), T(k)]$. Now since
\[
\int_{T(i)}^{T(k)} \mathcal{K}(T(k)-t)\dd t
\approx \int_{T(i)}^{T(k)}
\frac{1}{(1+ct)^{2-1/\beta}}\dd t \dd u
\eqsim 1 - \frac{1}{(1+c(T(k)-T(i))^{1-1/\beta}},
\]
we then further simplify it as 
% (see the definition of $g$ function in Section~\ref{sub: g-functions})
\begin{align*}
\mathrm{LRD}(k) 
&\approx \frac{1}{B} \sum_{i=1}^{k} (\eta_{i-1}-\eta_i)
(\sigma^2 + e(T(i))) (1 - (1+c (T(k)-T(i)))^{-\gamma}).
\end{align*}
Here, we approximate $\xi_i$ as $T(i)$
and introduce a new parameter $\gamma$ to replace
$1-\frac{1}{\beta}$ for simplicity.

Therefore, combining with \eqref{eq:numerical_LRD_form},
when the batch size $B$ is fixed, after renaming some constants,
the final discrete ansatz can be written as
\begin{equation}
\label{eqn: fsl-ansatz}
\begin{aligned}
\cR_k &\approx
c_0 + \frac{c_1}{T(k)^s} + c_2 M^{-s\beta}\\ 
&\qquad 
- c_3\sum_{i=1}^{k} (\eta_{i-1} - \eta_i)
\left(c_4 + \frac{1}{T(i)^s}\right)
\left(1 - (1+c_5(T(k)-T(i)))^{-\gamma}\right),
\end{aligned}
\end{equation}
where $c_0$, $c_1$, $c_2$, $c_3$, $c_4$, $c_5$, $s$, $\beta$, $\gamma$ are
constants to fit.

\paragraph*{Fitting the Practical FSL}
The objective of this experiment is to analyze and fit the loss function using
our functional scaling law, by \eqref{eqn: fsl-ansatz}, since we do not
explore the effect of varying the model size $M$ in our experiments,
we drop the term $M^{-s\beta}$ and get
\[
\mathcal{L}_\Theta(k) =  L_{0} + \frac{c_1}{T(k)^s} - \operatorname{LRD}(k)
\]
where $T(k) = \sum_{i=1}^{k} \eta_i$ and
\[
\operatorname{LRD}(k):= c_2 \sum_{i=1}^{k}(\eta_{i-1} - \eta_{i})
\left( c_3 + \frac{1}{T(i)^s} \right) \left( 1 - \frac{1}{(1 + c_4(T(k)
- T(i)))^{\gamma}} \right),
\]
and $\Theta = (L_0, c_1, c_2, c_3, c_4, s, \gamma)$.

Following \citep{tissue2024scaling}, we utilize the Huber loss as the objective function. 
\[
\min_{\Theta} \sum_{k=1}^K \text{Huber}_{\delta}\left(\log \mathcal{L}_{\Theta}(k)
	- \log \mathcal{L}_{\text{gt}}(k)\right),
\]
where $\delta = 1\times 10^{-3}$, $\mathcal{L}_{\text{gt}}$
denotes the ground truth of the validation
losses.
We adopt the Adam optimizer, with a learning rate of $5 \times 10^{-2}$ for the
index parameters in our law and $5 \times 10^{-3}$ for the coefficient or
constant parameters. Each optimization takes over 10,000 steps.

We fit the law on the 400M model and 1B model trained with 20B tokens and an 8-1-1
LRS
We then predict the loss curve for the 400M model and 1B model with cosine LRS
and WSD LRS.
The experiment result is present in Figure \ref{fig: llm-fsl}.

\paragraph*{FSL-optimal LRS via numerical variation.} \
We propose to obtain a numerical optimal LRS by directly minimizing the final-step loss over the space of LRS using the fitted FSL, termed FSL-optimal LRS.

\textbf{Step 1: Fitting FSL.}
Fit  FSL on the loss
curve of a 1B QwenMoE model trained on  20B tokens with batch size 288, maximum learning rate
$\eta_0 = 0.001$, and the 8-1-1 scheduler over a total step of $K = 33907$
, following the same procedure described earlier. 

\textbf{Step 2: Optimize LRS.}
To improve optimization stability, we reparameterize the learning rate schedule by defining 
$$
\delta_i = \eta_i - \eta_{i+1},\text{ for } i = 0, 1, \dots, K-1.
$$
Then, the $i$-th step learning rate can be recovered by  $\eta_i = \eta_0 - \sum_{k=0}^{i-1} \delta_k$, which defines  a one-to-one
correspondence between the learning rate schedule $\{\eta_i\}$ and
$\{\delta_i\}$.  The optimization problem is
\begin{equation}\label{eqn: FSL-optimization}
    \min_{\{\delta_i\}_{i=1}^K} \mathcal{L}_{\Theta}(\{\eta_i\}_{i=1}^K),
    \quad \st\, \sum_{k=0}^{K-1} \delta_k \le \eta_0,
    \, d\eta_i \ge 0, i=0,1,\dots,K-1.
\end{equation}
To solve the above constraint optimization, we use the  projected gradient descent (PGD) {\citep{bubeck2015convex}}. The learning rate of PGD is searched
ranging from $1\times 10^{-8}$ to $5\times 10^{-10}$,
and the optimization step number ranges from 50,000 to 100,000.

The resulting FSL-optimal LRS is presented in Figure~\ref{fig: optimalr-moe} (left), where cosine, WSD, and 8-1-1 LRSs are also given for a comparison.

\textbf{Step 3: Evaluate our LRS.} We then evaluate the performance of the resulting FSL-optimal LRS, and the three LRSs in Figure~\ref{fig: optimalr-moe} (left) are used as baseline. 
All comparisons are conducted on the same 1B QwenMoE model under identical training conditions: 33,907 total steps, batch size 288, and 20B training tokens. Full loss curves are shown in Figure~\ref{fig: optimalr-moe}.

\paragraph*{Additional Experiments}
We have further conducted ablation experiments with different model sizes and architectures, different total steps and different WSD schedules.

We validate our functional scaling law in models with various sizes, ranging from 100M to 1B, and diverse architectures including GPT-2 \citep{radford2019language}, LLaMA \citep{touvron2023llama} and QwenMoE \citep{yang2024qwen25}.
For each model, we first fit the FSL using the 8-1-1 LRS and subsequently employ it to predict the loss curve under a WSD LRS.
Next we numerically solve the FSL-optimal LRS and empirically validate its efficacy by comparing the final pre-training loss against those obtained using other commonly adopted learning rate schedules.
We present the results in Figure~\ref{fig:1B-LLaMA} for the 1B LLaMA dense model, Figure~\ref{fig:100M-GPT2} for the 100M GPT-2 dense model.
The consistent alignment between predicted and observed performance across architectures and sizes underscores the robustness and generalizability FSL.

We further validate the applicability of our functional scaling law (FSL) across varying training durations. Using a 100M LLaMA dense model, we conduct experiments with total training steps set to 17k, 34k, 68k, and 134k. As demonstrated in Figures~\ref{fig:different-step-fit-lrs} and \ref{fig:different-step-loss}, our FSL accurately models the loss trajectories across all evaluated step counts, confirming its robustness to different total training steps.

Finally, we conduct a comprehensive empirical comparison between our FSL-optimal learning rate schedule and various WSD baselines, examining different decay ratios and minimum learning rate configurations. As evidenced by Figures~\ref{fig:wsd-decay-ratio} and \ref{fig:wsd-to-zero}, our FSL-optimal LRS consistently outperforms all WSD variants, achieving superior final pre-training loss across all experimental conditions. This systematic evaluation demonstrates both the effectiveness of our theoretically-derived schedule and its practical advantages over conventional heuristic approaches.

% 1B model
    \begin{figure}
    \centering
    \begin{subfigure}[b]{0.3275\textwidth}
        \includegraphics[width=\textwidth]{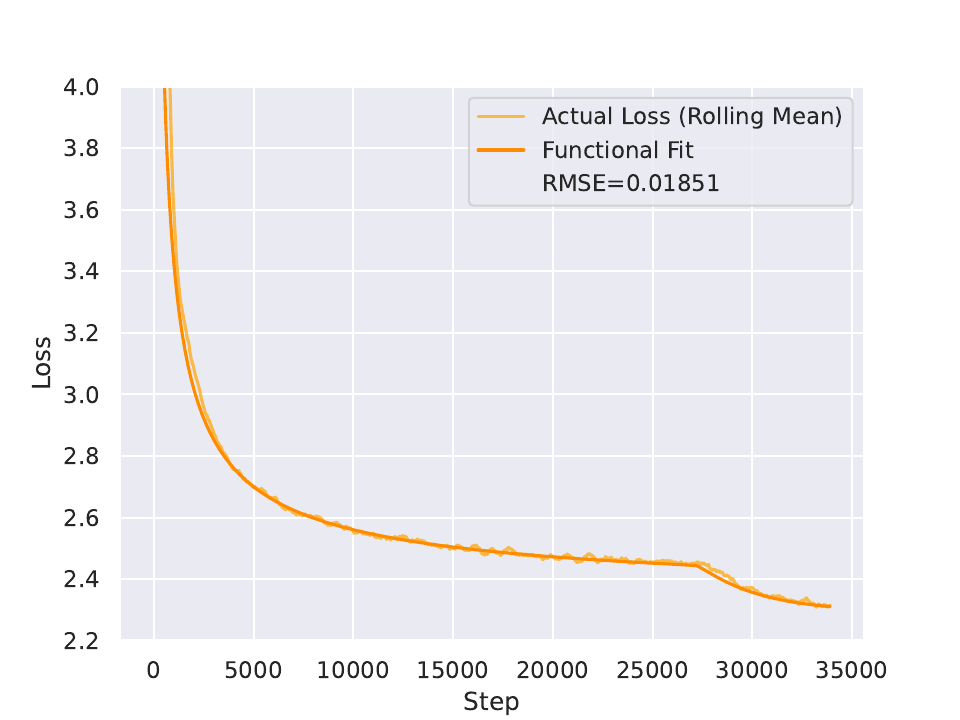}
        \caption{Prediction on 
        WSD LRS}
    \end{subfigure}
    \hfill
    \begin{subfigure}[b]{0.3275\textwidth}
        \includegraphics[width=\textwidth]{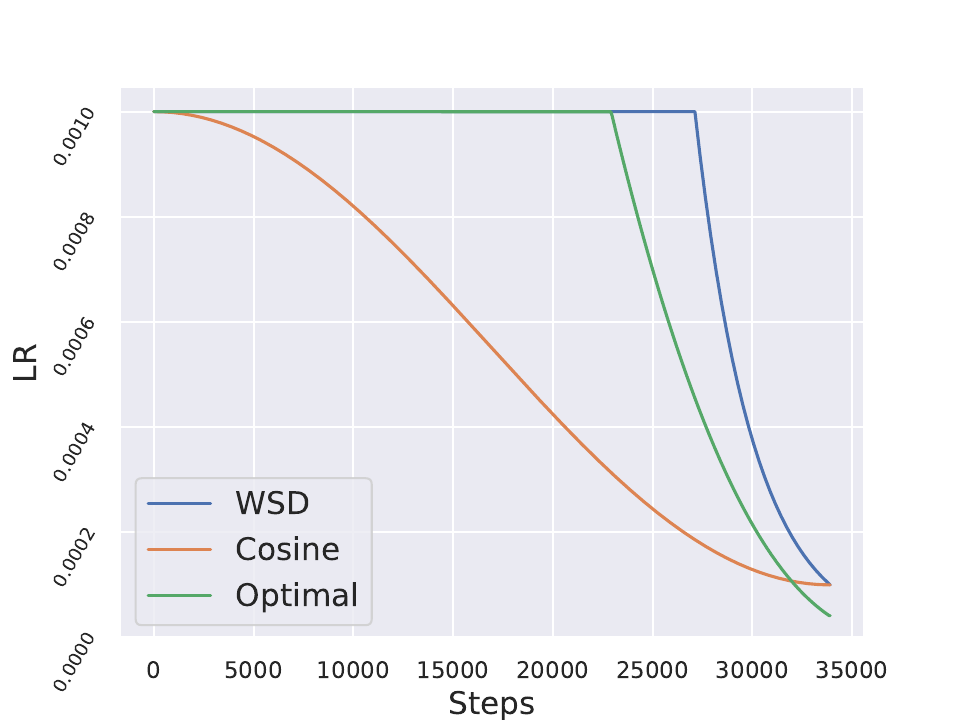}
        \caption{Optimal and Existing LRSs}
    \end{subfigure}
    \hfill
    \begin{subfigure}[b]{0.3275\textwidth}
        \includegraphics[width=\textwidth]{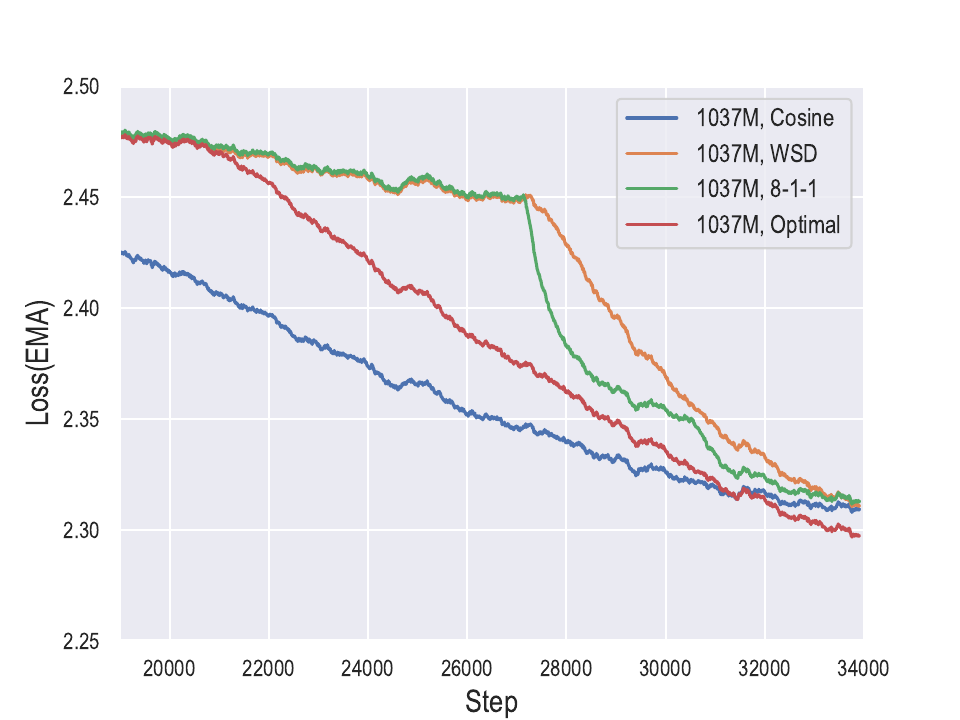}
        \caption{Loss curve of LRSs}
    \end{subfigure}
    \hfill
    \caption{\textbf{Experiment on the 1B LLaMA (dense) model.} Figure (a): We fit our functional scaling law on the loss curve of 1B LLaMA (dense) model with 20B tokens training data and 
    8-1-1 LRS. Figures (b)(c): The comparison on the 1B model between the optimal LRS, cosine LRS, WSD LRS with exponential decay and 8-1-1 LRS.}
    \label{fig:1B-LLaMA}
    \end{figure}

% gpt
    \begin{figure}
    \centering
    \begin{subfigure}[b]{0.3275\textwidth}
        \includegraphics[width=\textwidth]{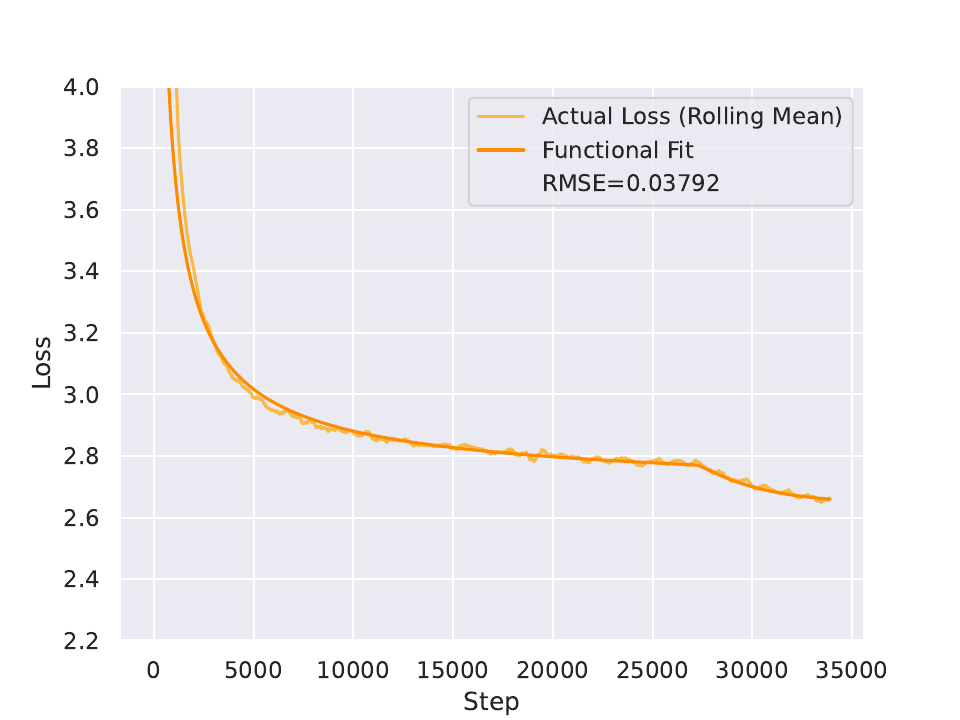}
        \caption{Prediction on 
        WSD LRS}
    \end{subfigure}
    \hfill
    \begin{subfigure}[b]{0.3275\textwidth}
        \includegraphics[width=\textwidth]{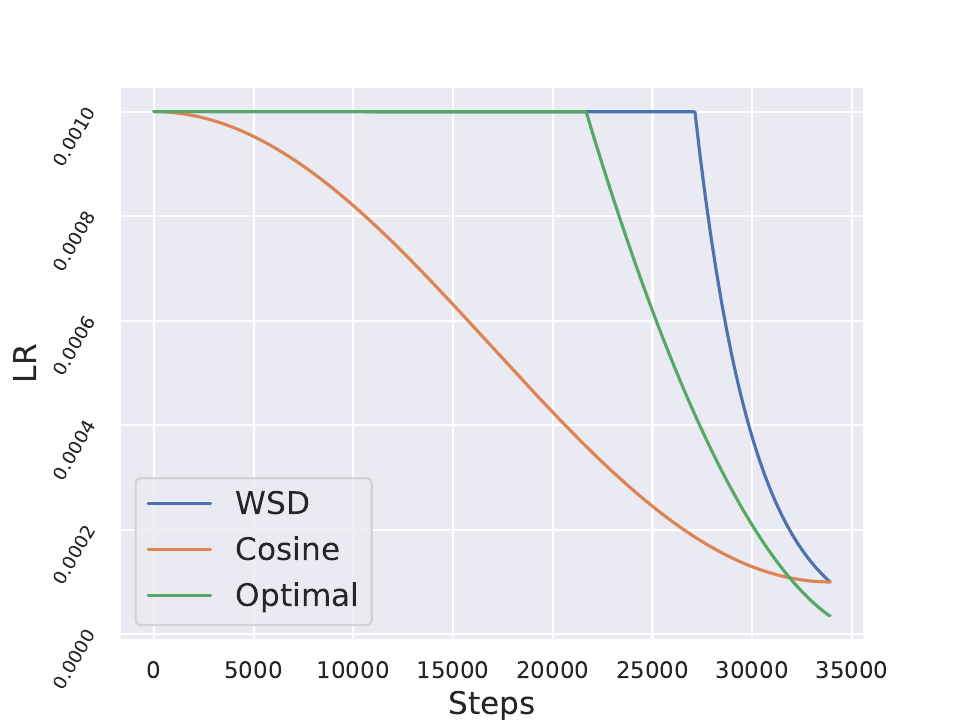}
        \caption{Optimal and Existing LRSs}
    \end{subfigure}
    \hfill
    \begin{subfigure}[b]{0.3275\textwidth}
        \includegraphics[width=\textwidth]{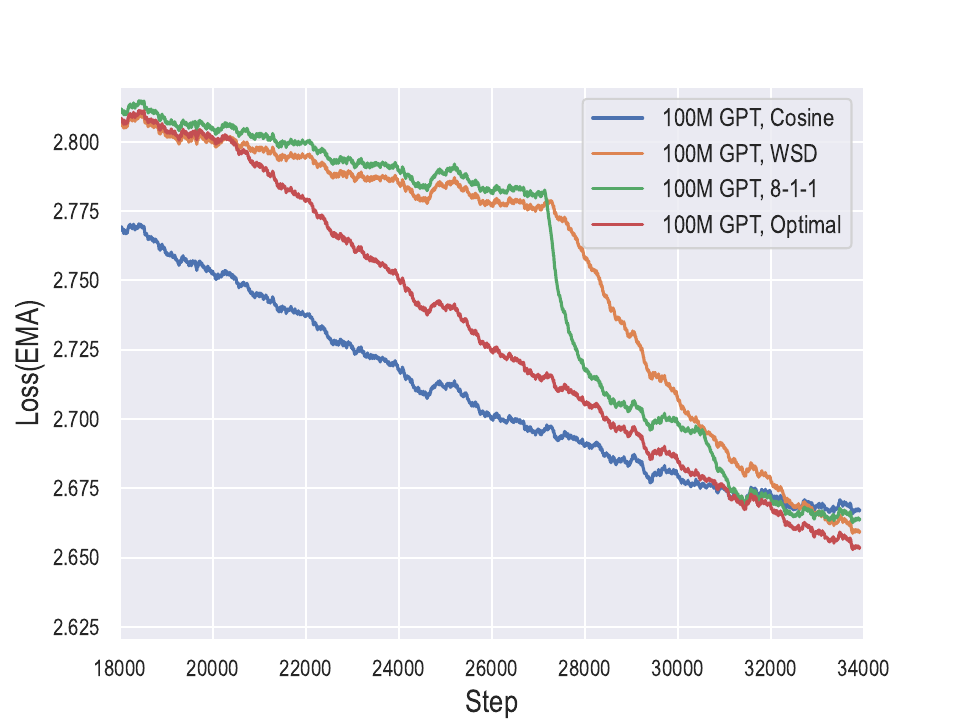}
        \caption{Loss curve of LRSs}
    \end{subfigure}
    \hfill
    \caption{\textbf{Experiment on the 100M GPT2 (dense) model.} Figure (a): We fit our functional scaling law on the loss curve of 100M GPT2 (dense) model with 20B tokens training data and 
    8-1-1 LRS. Figures (b)(c): The comparison on the 100M model between the optimal LRS, cosine LRS, WSD LRS with exponential decay and 8-1-1 LRS.}
    \label{fig:100M-GPT2}
    \end{figure}

% moe
    % \begin{figure}
    % \centering
    % \begin{subfigure}[b]{0.3275\textwidth}
    %     \includegraphics[width=\textwidth]{figs/Additional/moe_wsd_fit.pdf}
    %     \caption{Fitting on 
    %     WSD LRS}
    %     \label{fig:5a}
    % \end{subfigure}
    % \hfill
    % \begin{subfigure}[b]{0.3275\textwidth}
    %     \includegraphics[width=\textwidth]{figs/Additional/moe_optimal_lrs.pdf}
    %     \caption{Optimal and Existing LRSs}
    %     \label{fig:5b}
    % \end{subfigure}
    % \hfill
    % \begin{subfigure}[b]{0.3275\textwidth}
    %     \includegraphics[width=\textwidth]{figs/Additional/loss_curve_1B_moe.pdf}
    %     \caption{Loss curve of LRSs}
    %     \label{fig:5c}
    % \end{subfigure}
    % \hfill
    % \caption{\textbf{Experiment on the 1B Qwen (MoE) Model.} Figure (a): We fit our functional scaling law on the loss curve of 1B Qwen (MoE with 230M activated parameters) model with 20B tokens training data and 
    % WSD LRS. Figures (b)(c): The comparison on the 1B model between the optimal LRS, cosine LRS, WSD LRS with exponential decay and 8-1-1 LRS.}
    % \label{fig:1B-MoE}
    % \end{figure}

% different steps fit and lrs
    \begin{figure}
    \centering
    \begin{subfigure}[b]{0.45\textwidth}
        \includegraphics[width=\textwidth]{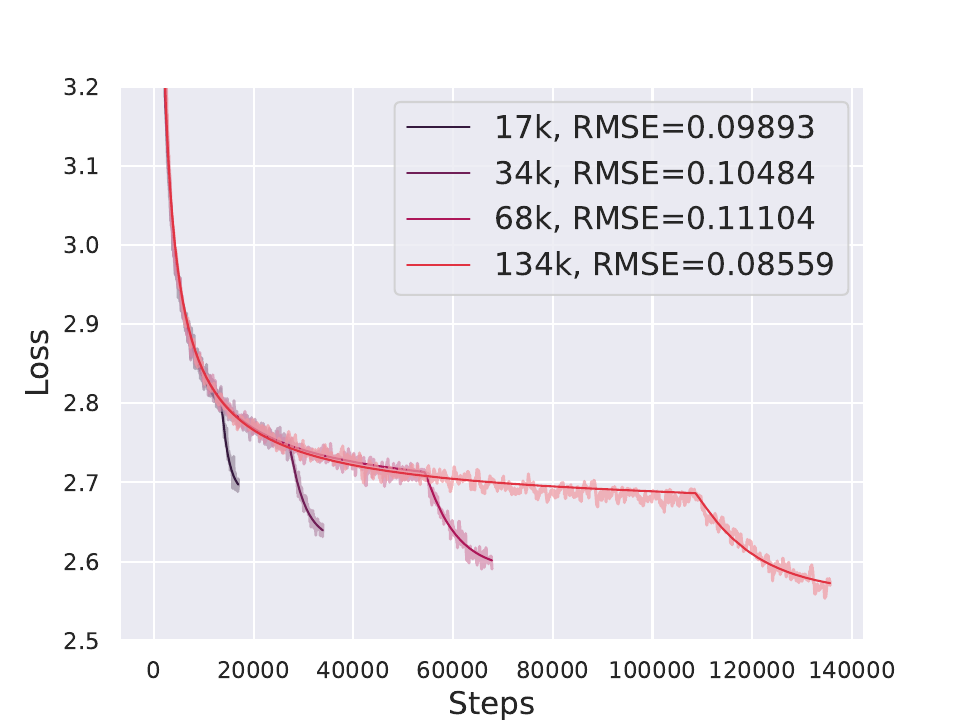}
        \caption{Prediction result}
        \label{fig:different-step-fit}
    \end{subfigure}
    \hfill
    \begin{subfigure}[b]{0.45\textwidth}
        \includegraphics[width=\textwidth]{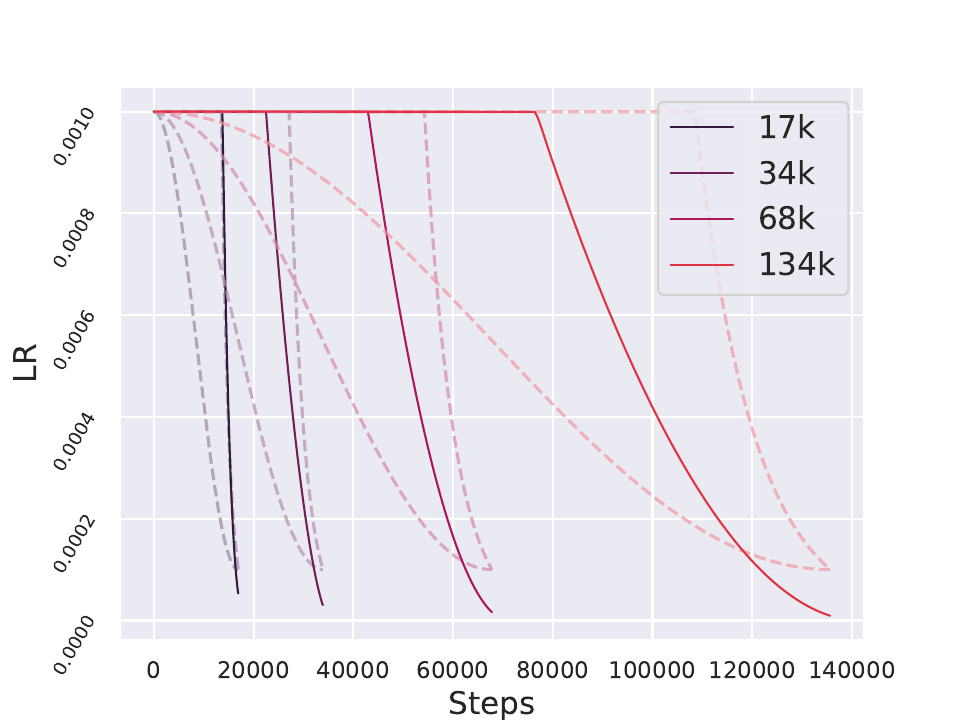}
        \caption{Optimal LRS}
        \label{fig:different-step-lrs}
    \end{subfigure}
    \hfill
    \caption{\textbf{Experiments with different total steps.} Figure (a): Fitted functional scaling laws on 100M LLaMA model with different total training steps 17k, 34k, 68k and 134k (corresponding to 10B, 20B, 40B and 80B tokens respectively). Figure (b): Optimal LRSs compared with cosine and WSD LRSs. The solid lines are optimal LRSs, and the dashed lines are cosine/WSD LRSs.}
    \label{fig:different-step-fit-lrs}
    \end{figure}
    
% different steps loss curve
    \begin{figure}
    \centering
    \begin{subfigure}[b]{0.3275\textwidth}
        \includegraphics[width=\textwidth]{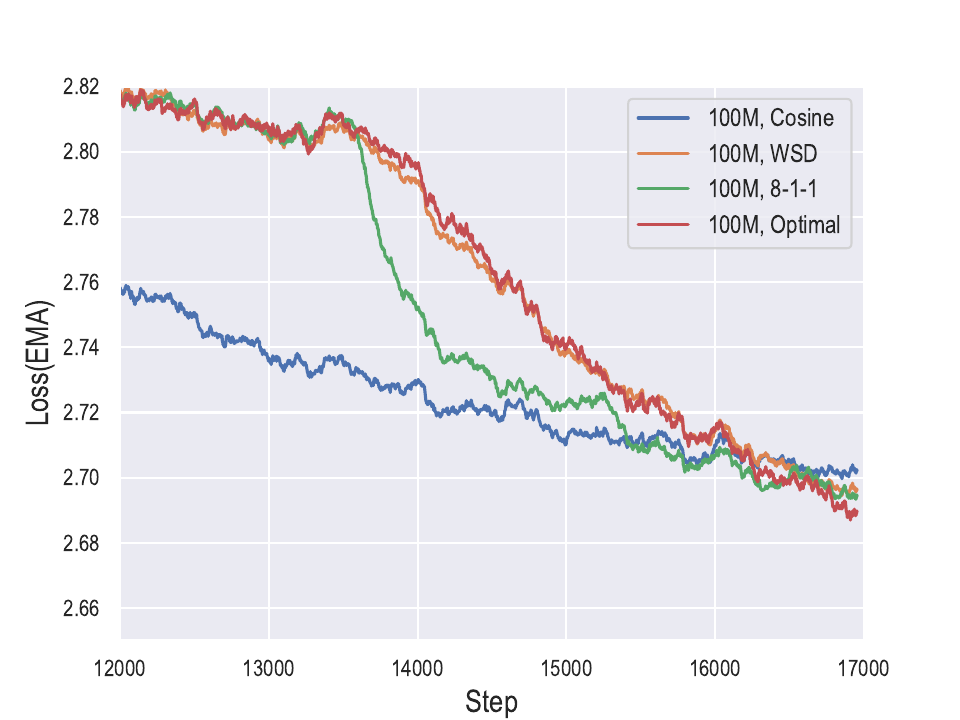}
        \caption{17k steps}
    \end{subfigure}
    \hfill
    \begin{subfigure}[b]{0.3275\textwidth}
        \includegraphics[width=\textwidth]{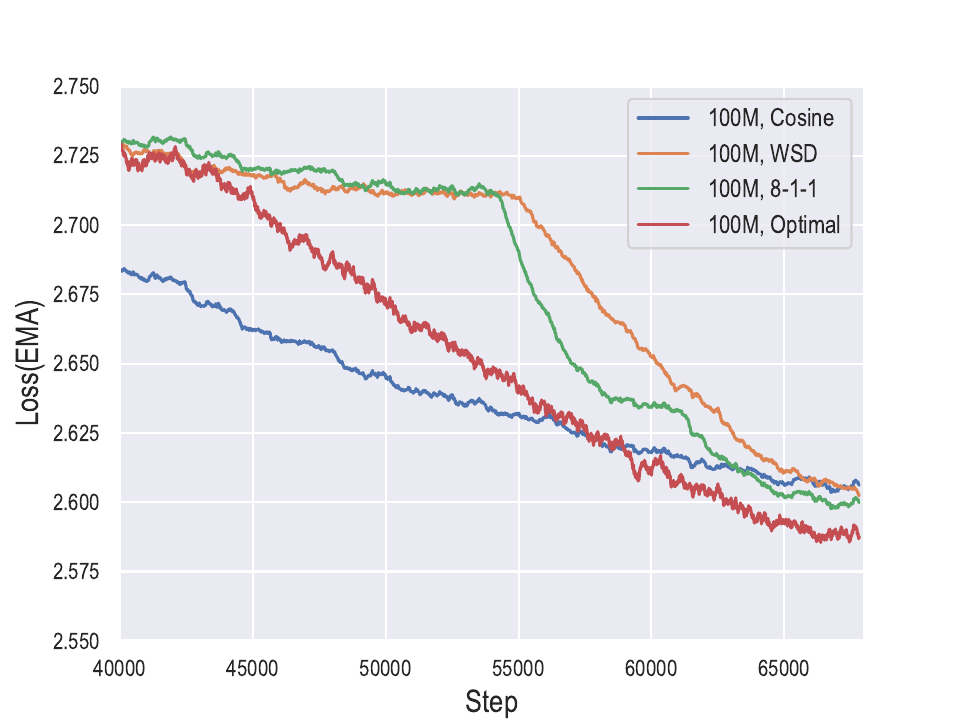}
        \caption{68k steps}
    \end{subfigure}
    \hfill
    \begin{subfigure}[b]{0.3275\textwidth}
        \includegraphics[width=\textwidth]{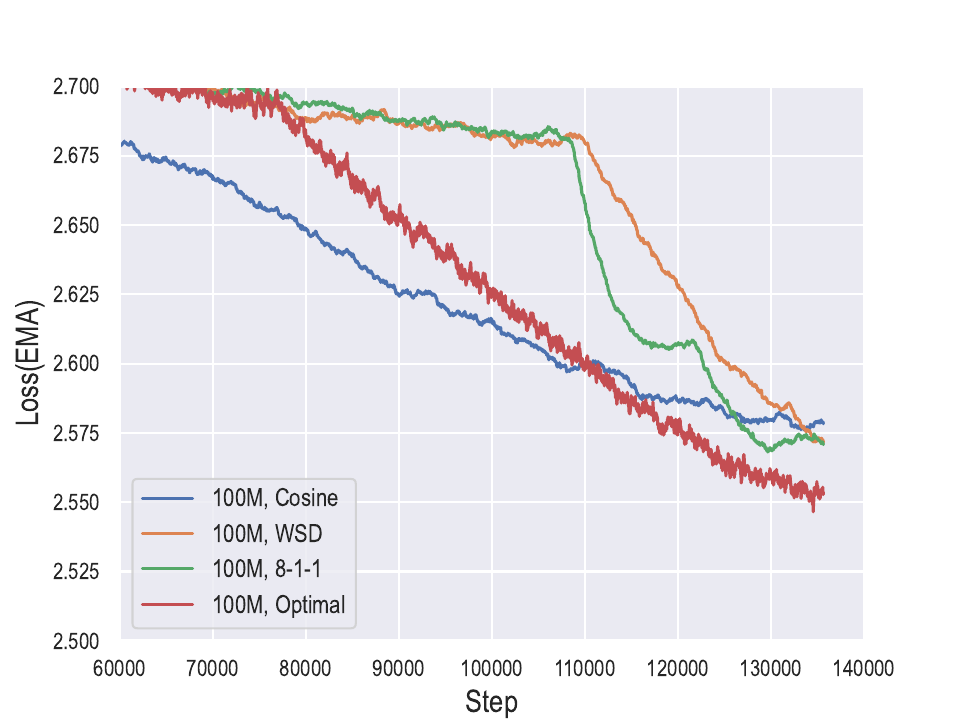}
        \caption{134k steps}
    \end{subfigure}
    \hfill
    \caption{\textbf{Experiments with different total steps.} We compare loss curves of existing LRSs and optimal LRS on the 100M LLaMA model with different total training steps 17k, 68k and 134k.}
    \label{fig:different-step-loss}
    \end{figure}

% different wsd ratio
    \begin{figure}
        \centering
        \includegraphics[width=0.45\linewidth]{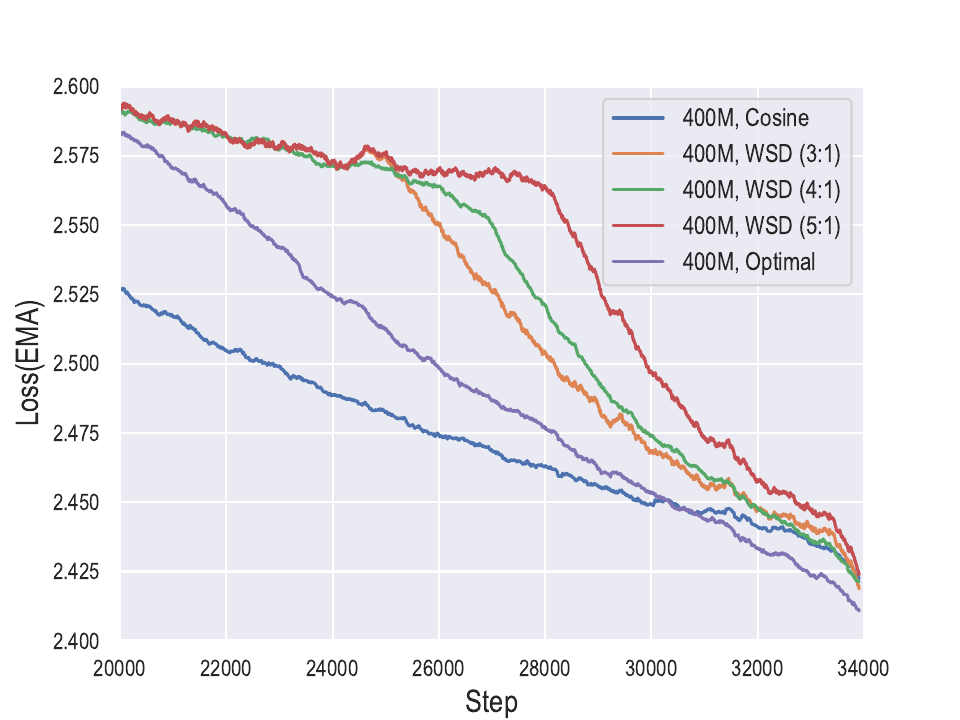}
        \caption{\textbf{WSD with different decay ratios:} We train a 400M LLaMA (dense) model with 20B tokens of training data and WSD LRSs with the ratios between stable time and decay time of 3:1, 4:1, and 5:1. All WSD LRSs exhibit a final loss similar to that of the Cosine LRS, and the optimal LRS derived from our functional scaling law outperforms all other LRSs by a loss gap of approximately 0.01.}
        \label{fig:wsd-decay-ratio}
    \end{figure}

% wsd to zero
    \begin{figure}
        \centering
        \includegraphics[width=0.45\linewidth]{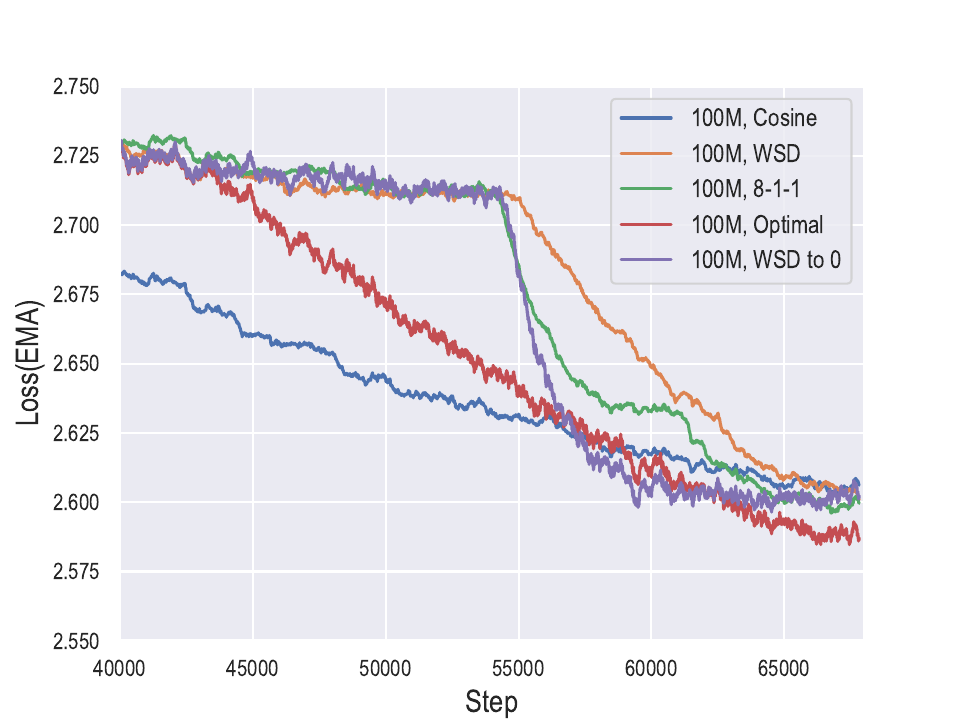}
        \caption{\textbf{Comparison between optimal LRS and WSD with a near-zero final learning rate:} We train a 100M LLaMA (dense) model with 40B tokens training data and various LRSs with the same $\eta_{\max} = 10^{-3}$, including WSD LRS with $\eta_{\min}$ = $\frac{1}{10}\eta_{\max}$, WSD LRS with $\eta_{\min} = 10^{-7}$, cosine LRS with $\eta_{\min} = \frac{1}{10} \eta_{\max}$, 8-1-1 LRS with $\eta_{\min} = \frac{1}{10}\eta_{\max}$, and optimal LRS. The experimental results show that decaying to (near) zero does not result in significant loss reduction.}
        \label{fig:wsd-to-zero}
    \end{figure}

%\newpage

\section{Proofs for Section \ref{sec:main_results}}
\label{sec:proof_of_theorem_ref_thm_fsl}

\subsection{Volterra Integral Equation Governing the Loss Dynamics}
\label{sub:Analysis of SDE}
In this section, we derive a Volterra-type integral equation that exactly characterizes 
the evolution of expected loss under the intrinsic-time SDE.
This equation serves as the starting point for all subsequent theoretical analysis.
Recall the intrinsic-time SDE:
\begin{equation}
\label{eq:sde_raw}
\dd \bnu_t = - \nabla \cR(\bnu_t) \dd t
+ \sqrt{\gamma_t \bSigma(\bnu_t)} \dd \mathbf{B}_t,
\end{equation}
where we write $\gamma_t = \gamma(t)$ for simplicity.

By the definition of $\cR(\bv)$, we have 
$
    \nabla \cR(\bv) = \bW\bH(\bW^\top \bv - \theta^*).
$
Let $\bu_t = \bW ^\top \bnu_t-\btheta^*$. Then, we have 
\[
    \cE_t=\cE(\bu_t) = \half\|\bu_t\|_{\bH}^2
\]
To obtain the estimate of $\cE_t$, we consider
the intrinsic-time SDE for $\bu_t$  given by: 
\begin{lemma}
    We have
    \begin{equation}
    \label{eq:sde_u}
        \dd \bu_t = -\bW ^\top \bW  \bH 
        \bu_t \dd t + \sqrt{\gamma_t \bW ^\top
        \bSigma_t \bW } \dd \mathbf{B}_t,
    \end{equation}
    where $\bSigma_t :=
    \bSigma(\bnu_t)$.
\end{lemma}
\begin{proof}[Proof]
    By Eq.~\eqref{eq:sde_raw},
    \begin{align*}
        \dd \bu_t
        &= \dd (\bW ^\top \bnu_t
        - \bv^*)= \bW ^\top \dd \bnu_t = -\bW ^\top \bW  \bH  \bu_t
        \dd t + \bW ^\top \sqrt{\gamma_t \bSigma_t}
        \dd \Tilde{\mathbf{B}}_t.
    \end{align*}
    Here $\Tilde{\mathbf{B}}_t$ is an $N$ dimensional standard Brownian motion,
    we are going to replace it with an $M$ dimensional standard Brownian motion
    ${\mathbf{B}}_t$.

    It is easy to see that the diffusion term $\bW ^\top \sqrt{\gamma_t
    \bSigma_t} \dd \Tilde{\mathbf{B}}_t$ has the same distribution
    as $\sqrt{\gamma_t \bW ^\top
    \bSigma_t \bW } \dd {\mathbf{B}}_t$,
    hence the SDE can be written in ${\mathbf{B}}_t$ as
    \[
    \dd \bu_t = -\bW ^\top \bW  \bH 
    \bu_t \dd t + \sqrt{\gamma_t \bW ^\top
    \bSigma_t \bW }\dd {\mathbf{B}}_t.
    \]
\end{proof}

A key insight for tractability is that the gradient noise exhibits the following anisotropic structure:
% \begin{lemma}[Noise structure]\label{lemma: noise-covariance-0}
% For any $\bv \in \RR^M$, it holds that
% $$
% (2C_{1}\cE(\bv) + \sigma^2)\, \bW\bH\bW^\top
% \preceq \bSigma(\bv) \preceq
% (4C_{2}\cE(\bv) + \sigma^2)\, \bW\bH\bW^\top.
% $$
% \end{lemma}

Our analytic analysis also relies on the noise structure characterized by the following lemma.
\begin{lemma}[Noise Structure]\label{lemma: noise-covariance}
For any $\bv \in \RR^M$, it holds that
$$
(2\rho_{-}\cE(\bv) + \sigma^2)\, \bW\bH\bW^\top
\preceq \bSigma(\bv) \preceq
(2\rho_{+}\cE(\bv) + \sigma^2)\, \bW\bH\bW^\top.
$$
\end{lemma}

Noting  $\nabla^2 \cR(\bv)=\bW\bH\bW^\top$ and $\cR(\bv) = \cE(\bv) + \frac{1}{2} \sigma^2$, this lemma means
$
\bSigma(\bv) \eqsim \cR(\bv) \nabla^2 \cR(\bv).
$
That is, the gradient noise  scales proportionally with the population risk and aligns with the local curvature. Notably, the  noise has two distinct sources: (i) the  fit-dependent term  $\cE(\bv)$, which arises purely from minibatching and persists even in the absence of label noise;
(ii) the   $\sigma^2$ term, which captures the contribution from label noise. This anisotropic structure of SGD noise -- scaling with risk and shaped by curvature -- has also been observed in prior work~\citep{wu2022alignment,wu2023implicit}.

\begin{proof}
Noting $\ell(\bz;\bv)=\half(\bv^\top \bW\bphi(\bx)-y)^2$, we have
\begin{align*}
    \nabla \ell(\bz;\bv) &= \bW\bphi(\bx) \bphi(\bx)^\top\left(\bW^\top \bv - \btheta^*\right) - \bW\bphi(\bx)\epsilon\\ 
    \nabla \cR(\bv) &= \EE[\nabla \ell(\bz;\bv)] = \bW\bH\left(\bW^\top\bv - \btheta^*\right).
\end{align*}
Hence, the covariance matrix of the noise $\bxi:=\nabla \ell(\bz;\bv) - \nabla \cR(\bv)$ is given by
\begin{align*}
    \bSigma(\bv)
    &= \EE[\bxi\bxi^\top | \bv]\\ 
    &=  \bW  \left(\EE\left[
    \bphi(\bx) \bphi(\bx)^\top \bu
    \bu^\top \bphi(\bx)
    \bphi(\bx)^\top\right] - \bH \bu
    \bu^\top \bH \right) \bW ^\top
    + \sigma^2 \bW  \bH  \bW ^\top.
\end{align*}
Noting
\begin{align*}
    \EE\left[\bphi(\bx) \bphi(\bx)^\top \bu\bu^\top \bphi(\bx) \bphi(\bx)^\top\right] - \bH\bu\bu^\top \bH  &= \EE\left[\bphi(\bx)^\top \bu\bu^\top\bphi(\bx) \bphi(\bx)\bphi(\bx)^\top\right] - \bH\bu\bu^\top \bH,
\end{align*}
then applying Assumption \ref{assumption: hypercontractivity}, we have 
\[
    \bSigma(\bv) \preceq \rho_{+}\bW\tr(\bH\bu\bu^\top) \bH\bW^\top+ \sigma^2 \bW\bH\bW^\top = \left(2\rho_{+}\cE(\bu) + \sigma^2\right) \bW\bH\bW^\top,
\] 
where the last step follows from $\tr(\bH\bu\bu^\top) = \|\bu\|_{\bH}^2=2\cE(\bv)$. The lower bound follows the same proof.
\end{proof}

\textbf{The excess-risk dynamics} is then given by the following Volterra integral
equation:
\begin{proposition}\label{pro:evolution-pro}
For the intrinsic-time SDE, we have
    \begin{equation}
        2\EE[\cE_t] = \bu_0^\top \bA^\top_t
        \bH  \bA_t \bu_0 +
        \int_{0}^{t}\tr(\mathbf{S} \bA^\top_{t-\tau}
        \bH  \bA_{t-\tau} \mathbf{S}) \cdot
        \gamma_\tau(c_\tau \EE[\cE_\tau]+ \sigma^2) \dd \tau, 
        \label{eq:risk-recursion}
    \end{equation}
    where $\bA_t := e^{-\bW ^\top \bW 
    \bH  t}$, $\mathbf{S} := \sqrt{\bW ^\top
    \bW \bH \bW ^\top \bW }$, and $c_\tau\in [2\rho_{-},2\rho_{+}]$ for any $\tau\geq 0$.
\end{proposition}

\begin{proof}
    By It\^o's formula,
    \begin{align*}
        \dd (e^{\bW ^\top \bW  \bH  t}
        \bu_t)
        &= e^{\bW ^\top \bW  \bH  t}
        (\bW ^\top \bW  \bH  \bu_t\dd t +
        \dd \bu_t)\\
        &= e^{\bW ^\top \bW  \bH  t}
        \sqrt{\gamma_t \bW ^\top \bSigma_t \bW }
        \dd \mathbf{B}_t.
    \end{align*}

    Integrating both sides, we get
    \[
    e^{\bW ^\top \bW  \bH  t} \bu_t
    - \bu_0 = \int_{0}^{t} e^{\bW ^\top \bW 
    \bH  \tau} \sqrt{\gamma_\tau \bW ^\top
    \bSigma_\tau \bW } \dd \mathbf{B}_\tau.
    \]
    Now write $\bA_t = e^{-\bW ^\top \bW 
    \bH  t}$, we have
    \[
    \bu_t = \bA_t\bu_0
    + \int_{0}^{t} \bA_{t-\tau} \sqrt{\gamma_\tau \bW ^\top
    \bSigma_\tau \bW } \dd \mathbf{B}_\tau.
    \]
    Note that the integral with respect to $\mathbf{B}_t$ always has
    zero expectation, therefore we have
    \begin{align*}
        2\EE \cE_t
        &= \EE (\bu_t^\top \bH  \bu_t)\\
        &= \bu_0^\top \bA^\top_{t} \bH 
        \bA_t \bu_0 + \EE \int_{0}^{t}
        \gamma_\tau \tr\left(\sqrt{\bW ^\top \bSigma_\tau
        \bW } \bA^\top_{t-\tau} \bH 
        \bA_{t-\tau} \sqrt{\bW ^\top \bSigma_\tau
        \bW }\right) \dd \tau.
    \end{align*}

    By Lemma \ref{lem:spd_matrix_trace} and Lemma \ref{lemma: noise-covariance},
    we have
    \[
    \tr\left(\sqrt{\bW ^\top \bSigma_\tau
    \bW } \bA^\top_{t-\tau} \bH 
    \bA_{t-\tau}\sqrt{\bW ^\top \bSigma_\tau
    \bW }\right) \le (2\rho_{+} \cE_\tau + \sigma^2)
    \tr(\mathbf{S} \bA^\top_{t-\tau}\bH 
    \bA_{t-\tau} \mathbf{S}),
    \]
    \[
    \tr\left(\sqrt{\bW ^\top \bSigma_\tau
    \bW } \bA^\top_{t-\tau} \bH 
    \bA_{t-\tau} \sqrt{\bW ^\top \bSigma_\tau
    \bW }\right) \ge  (2\rho_{-}\cE_\tau + \sigma^2)
    \tr(\mathbf{S} \bA^\top_{t-\tau}\bH 
    \bA_{t-\tau} \mathbf{S}).
    \]
    Hence there exists some constant $c_\tau \in [2\rho_{-}, 2\rho_{+}]$ such that
    \[
    \EE \tr\left(\sqrt{\bW ^\top \bSigma_\tau
    \bW } \bA^\top_{t-\tau} \bH 
    \bA_{t-\tau} \sqrt{\bW ^\top \bSigma_\tau
    \bW }\right) = (c_\tau \EE[\cE_\tau] + \sigma^2)
    \tr(\mathbf{S} \bA^\top_{t-\tau}\bH 
    \bA_{t-\tau} \mathbf{S}),
    \]
    from which the lemma follows.
\end{proof}

\subsection{Power-law Decay of the $e_M$ and $\cK_M$ Functions}

In this section, we rigorously establish the power-law decay behavior of the functions $e_M$ and $\cK_M$.
\begin{lemma}
    \label{lemma:ke-power}
    For $e_M(t)$, $\mathcal{K}_M(t)$ we have the following estimation:
    \[
    e_M(t) \lesssim t^{-s}, \quad \mathcal{K}_M(t) \lesssim
    t^{-(2-\frac{1}{\beta})}.
    \]
    Moreover,
    \[
    e_M(t) \eqsim t^{-s}, \quad \mathcal{K}_M(t) \eqsim
    t^{-(2-\frac{1}{\beta})}, \quad 1\lesssim t\lesssim M^{\beta}.
    \]
\end{lemma}
\begin{proof}[Proof]
    Recall that
    \[
    e_M(t) = \sum_{j=1}^{M} e^{-2\lambda_jt} \lambda_j|\theta_j^*|^2 \eqsim
    \sum_{j=1}^{M} e^{-2\lambda_jt} j^{-1-s\beta}.
    \]
    Since $\lambda_j \eqsim j^{-\beta}$, there exists a constant $c$ such that
    \[
    \lambda_j \le c j^{-\beta} \implies 2\lambda_j t \le 1 \text{ when } j\ge
    (2ct)^{\frac{1}{\beta}}.
    \]
    We have
    \begin{align*}
        e_M(t) &\eqsim \sum_{j=1}^{[(2ct)^{\frac{1}{\beta}}]} e^{-2\lambda_j
        t}j^{-1-s\beta} + \sum_{j=[(2ct)^{\frac{1}{\beta}}]+1}^{M}
        e^{-2\lambda_j t} j^{-1-s\beta}\\
        &\lesssim (2ct)^{\frac{1}{\beta}} \max_j e^{-2\lambda_j t} j^{-1-s\beta}
        + \sum_{j=[(2ct)^{\frac{1}{\beta}}]+1}^{M} j^{-1-s\beta}.
    \end{align*}
    Now we bound the two terms separately:
    \[
    \max_j e^{-2\lambda_j t}j^{-1-s\beta} \eqsim \max_j e^{-2\lambda_j t}
    (\lambda_jt)^{s+\frac{1}{\beta}} \cdot t^{-s-\frac{1}{\beta}}
    \le t^{-s-\frac{1}{\beta}}\sup_{x>0} e^{-2x}x^{s+\frac{1}{\beta}}
    \lesssim t^{-s-\frac{1}{\beta}}.
    \]
    where the last inequality is because $\sup_{x>0}
    e^{-2x}x^{s+\frac{1}{\beta}}$ is a constant only depending on $s$ and
    $\beta$.
    
    We have
    \[
        \sum_{j=[(2ct)^{\frac{1}{\beta}}]+1}^{M} j^{-1-s\beta}
        \lesssim \int_{[(2ct)^{\frac{1}{\beta}}]}^{\infty} x^{-1-s\beta}\dd x
        \eqsim [(2ct)^{\frac{1}{\beta}}]^{-s\beta} \eqsim t^{-s}.
    \]
    Therefore we have for all $t>0$,
    \[
    e_M(t) \lesssim (2ct)^{\frac{1}{\beta}} t^{-s-\frac{1}{\beta}} + t^{-s}
    \eqsim t^{-s}.
    \]

    On the other hand, when $M > 2([(2ct)^{\frac{1}{\beta}}+1])$, i.e.,
    $t\lesssim M^{\beta}$,
    \begin{align*}
        e_M(t) &\eqsim \sum_{j=1}^{[(2ct)^{\frac{1}{\beta}}]} e^{-2\lambda_j
        t}j^{-1-s\beta} + \sum_{j=[(2ct)^{\frac{1}{\beta}}]+1}^{M}
        e^{-2\lambda_j t} j^{-1-s\beta}\\
        &\gtrsim e^{-1}\sum_{j=[(2ct)^{\frac{1}{\beta}}]+1}^{M} j^{-1-s\beta}\\
        &\gtrsim \int_{[(2ct)^{\frac{1}{\beta}}]+1}^{2[(2ct)^{\frac{1}{\beta}}]+1}
        x^{-1-s\beta} \dd x \\
        &= \frac{1}{s\beta}
        (([(2ct)^{\frac{1}{\beta}}+1])^{-s\beta} -
        (2[(2ct)^{\frac{1}{\beta}}]+1)^{-s\beta}) \\
        &\eqsim t^{-s},
    \end{align*}
    where the last line requires $t\gtrsim 1$.
    Therefore we have shown that $e_M(t) \eqsim t^{-s}$ when $1\lesssim
    t\lesssim M^{\beta}$.

    The proof for $\mathcal{K}_M(t)$ follows the same way.
\end{proof}

\begin{lemma}
    \label{lemma:em-gap}
    We can bound the gap between $e_M(t)$ and $e_\infty(t)$ as follows:
    \[
    e_\infty(t) - e_M(t) \lesssim M^{-s\beta}.
    \]
    As a consequence, we have
    \[
    e_M(t) + M^{-s\beta} \eqsim e_\infty(t) + M^{-s\beta}.
    \]
\end{lemma}
\begin{proof}[Proof]
    We have
    \begin{align*}
        e_\infty(t) - e_M(t)
        &= \sum_{j=M+1}^{\infty} e^{-2\lambda_j t}\lambda_j |\theta_j^*|^2\\
        &\lesssim \sum_{j=M+1}^{\infty} j^{-1-s\beta}\\
        &\le \int_{M}^{\infty} x^{-1-s\beta}\dd x\eqsim M^{-s\beta}.
    \end{align*}
\end{proof}

\subsection{The Case of Top-$M$ Features}
First, we prove the general label-noise case of FSL (Theorem~\ref{thm:
fsl-general-noise}). Applying this result, we then derive the constant
label-noise case (Theorem~\ref{thm:fsl-const-noise}), the noiseless case
(Theorem~\ref{thm:fsl-zero-noise}) and the hard regime case (Theorem~\ref{thm:
fsl-hard-regime})

% \begin{theorem}
% 	[Functional Scaling Law, general case]
% 	\label{thm:fsl-topm}
% 	In the top-$M$ case,
% 	we have the lower bound
% 	\[
% 		\mathbb{E}[\mathcal{E}_t] \gtrsim M^{-s\beta} + e_M(t)
% 		+ \int_{0}^{t} \mathcal{K}_M(t - \tau) (e_M(\tau) + \sigma^2) \gamma_\tau
% 		\dd \tau.
% 	\]
% 	For the upper bounds
% 	\begin{itemize}
% 		\item When $t \le M^{\beta}$, we have the matching upper bound
% 	\[
% 		\mathbb{E}[\mathcal{E}_t] \lesssim M^{-s\beta} + e_M(t)
% 		+ \int_{0}^{t} \mathcal{K}_M(t - \tau) (e_M(\tau) + \sigma^2) \gamma_\tau
% 		\dd \tau.
% 	\]
% 		\item When $t \geq M^{\beta}$,
% 	\[
% 		\mathbb{E}[\mathcal{E}_t] \lesssim M^{-s\beta} + e_M(t)
% 		+ \int_{0}^{t} \mathcal{K}_\infty(t - \tau) (e_M(\tau) + \sigma^2) \gamma_\tau
% 		\dd \tau.
% 	\]
% 	\end{itemize}
% \end{theorem}

\subsubsection{Proof of Theorem~\ref{thm: fsl-general-noise}}

In the top-$M$ feature case, the matrix $\bW$ satisfies
$\bw_j = \be_j$ for each $j\in [M]$, therefore we can simplify the equation
for $\mathbb{E}[\mathcal{E}_t]$ as follows.
\begin{theorem}
	[Volterra equation of the top-$M$ case]
	\label{thm:volterra-topm}
	In the top-$M$ case, we have
	\begin{equation}
    \label{eq:volterra-topm}
		\mathbb{E}[\mathcal{E}_t] \eqsim M^{-s\beta} + e_M(t) +
		\int_{0}^{t} \mathcal{K}_M(t - \tau) \gamma_\tau (\mathbb{E}[\mathcal{E}_\tau]
		+ \sigma^2) \dd \tau,
	\end{equation}
	where the function $e_M$ and $\mathcal{K}_M$ are defined as
	\begin{equation}
		\label{eq:Ke-def}
	e_M(t) := \sum_{j=1}^{M} \lambda_j (\theta_j^*)^2 e^{-2\lambda_j t},
	\quad \mathcal{K}_M(t) := \sum_{j=1}^{M} \lambda_j^2 e^{-2\lambda_j t}.
	\end{equation}
\end{theorem}
\begin{proof}[Proof]
	By \eqref{eq:risk-recursion} in Proposition~\ref{pro:evolution-pro},
	note that $\bW^\top \bW \bH = \bH_{0:M} \in \mathbb{R}^{N\times N}$ is
	the top-$M$ part of the matrix $\bH$, i.e. $\bH_{0:M} =
	\diag\{\lambda_1, \dots, \lambda_M, 0, \dots, 0\}$, we get
	\[
		\mathbb{E}[\mathcal{E}_t] \eqsim \bu_0^\top \bH e^{-2\bH_{0:M} t}
		\bu_0 + \int_{0}^{t} \tr(\bH_{0:M}^2 e^{-2\bH_{0:M}})
		\gamma_\tau (\mathbb{E}[\mathcal{E}_t] + \sigma^2) \dd \tau,
	\]
	which can be further written in terms of the eigenvalues $\{\lambda_j\}$ as
	\[
		\mathbb{E}[\mathcal{E}_t] \eqsim \sum_{j=1}^{M} \lambda_j (u_0^{(j)})^2
		e^{-2\lambda_j t} + \sum_{j=M+1}^{\infty} \lambda_j (u_0^{(j)})^2 +
		\int_{0}^{t} \sum_{j=1}^{M} \lambda_j^2
		e^{-2\lambda_j t} \gamma_\tau (\mathbb{E}[\mathcal{E}_t] + \sigma^2)
		\dd \tau.
	\]
	Note that $u_0^{(j)}$, the $j$-th component of $\bu_0$, is equal to
	$\theta_j^*$ because of the zero initialization of $\bnu_0$.

	Therefore by the definition of $e_M$ and $\mathcal{K}_M$ we arrive at
	the Volterra-type integral equation of $\mathbb{E}[\mathcal{E}_t]$.
\end{proof}

\begin{lemma}
	\label{lemma:convolution-kernel}
	For the forgetting kernel $\mathcal{K}_{M}$ and $t \le M^\beta$,
	there exists a constant $C$
	independent of $t$, such that
	\[
	\mathcal{K}_M * \mathcal{K}_M(t) \leq C \mathcal{K}_M(t),
	\quad \forall t \le M^\beta.
	\]
	where $*$ denotes convolution:
	 \[
	\mathcal{K}_M * \mathcal{K}_M(t)
	:= \int_{0}^{t} \mathcal{K}_M(\tau) \mathcal{K}_M(t - \tau) \dd \tau.
	\]
\end{lemma}
\begin{proof}[Proof]
    Observe that $\mathcal{K}_M(t)$ is a monotonically decreasing function. 
By the symmetry of the convolution, we can write
\begin{align*}
    \mathcal{K}_M * \mathcal{K}_M(t)
    &= \int_0^t \mathcal{K}_M(\tau)\mathcal{K}_M(t-\tau)\,\dd \tau \\
    &= 2\int_0^{t/2} \mathcal{K}_M(\tau)\mathcal{K}_M(t-\tau)\,\dd \tau.
\end{align*}
Since $\mathcal{K}_M$ is decreasing, for $0 \le \tau \le t/2$ we have 
$\mathcal{K}_M(t-\tau) \le \mathcal{K}_M(t/2)$. 
This observation yields the upper bound
\[
    \mathcal{K}_M * \mathcal{K}_M(t)
    \le 2\,\mathcal{K}_M\!\left(\tfrac{t}{2}\right)
    \int_0^\infty \mathcal{K}_M(\tau)\,\dd \tau
    \le C\,\mathcal{K}_M\!\left(\tfrac{t}{2}\right),
\]
where the constant $C$ is finite and given by
\begin{align*}
    C 
    &= 2\int_0^\infty \mathcal{K}_M(\tau)\,\dd \tau 
     \le 2\int_0^\infty \sum_{j=1}^\infty \lambda_j^2 e^{-2\lambda_j \tau}\,\dd \tau 
      = \sum_{j=1}^\infty \lambda_j 
      = \mathrm{tr}(\mathbf{H}) < \infty.
\end{align*}

    It remains to show that when $t \le M^\beta$,
    \[
        \mathcal{K}_M\left(\frac{t}{2}\right)
        \le C'\mathcal{K}_M(t)
    \]
    for some constant $C'>0$. Recall that
	by Lemma~\ref{lemma:ke-power} we have
	\[
	\mathcal{K}_M(t) \eqsim t^{-(2-\frac{1}{\beta})},\quad
	\mathcal{K}_M\left(\frac{t}{2}\right) \eqsim
	\left(\frac{t}{2}\right)^{-(2-\frac{1}{\beta})}.
	\]
	Therefore when $1\lesssim t \lesssim M^{\beta}$ the conclusion follows.
	When $t\lesssim 1$, the function $\mathcal{K}_M(t)$ is decreasing
	and $\mathcal{K}_M(t) \le \sum_{j=1}^{M} \lambda_j^2 < \infty$, we can
	always find a constant $C'$ satisfying the condition.
	
    Combining the above estimates, we obtain
    \[
        (\mathcal{K}_M * \mathcal{K}_M)(t)
        \le C\,\mathcal{K}_M\left(\tfrac{t}{2}\right)
        \le CC_1\,\mathcal{K}_M(t)
        =: C'\mathcal{K}_M(t).
    \]
    The proof is complete.
\end{proof}

We now prove Theorem~\ref{thm: fsl-general-noise}.

\begin{proof}[Proof]
	The lower bound is trivial by $\mathbb{E}[\mathcal{E}_t] \geq e_M(t)$ and
	the Volterra equation \eqref{eq:volterra-topm}:
    \begin{align*}
    \mathbb{E}[\mathcal{E}_t] &\eqsim M^{-s\beta} + e_M(t) +
		\int_{0}^{t} \mathcal{K}_M(t - \tau) \gamma_\tau (\mathbb{E}[\mathcal{E}_\tau]
		+ \sigma^2) \dd \tau \\
        &\gtrsim M^{-s\beta} + e_M(t) +
		\int_{0}^{t} \mathcal{K}_M(t - \tau) \gamma_\tau (e_M(\tau)
		+ \sigma^2) \dd \tau.
    \end{align*}
	For the upper bounds, we first prove a slightly stronger bound
	with $\mathcal{K}_\infty$, that is,
    \begin{proposition}
    \label{prop:fsl-weak-upper-bound}
    For all $t>0$, $s>0$ and $\beta > 1$, we have
    \begin{equation}
 		\mathbb{E}[\mathcal{E}_t] \lesssim M^{-s\beta} + e_M(t)
 		+ \int_{0}^{t} \mathcal{K}_\infty(t - \tau) (e_M(\tau) + \sigma^2) \gamma_\tau
 		\dd \tau.
 	\end{equation}
    \end{proposition}
    \begin{proof}[Proof of \ref{prop:fsl-weak-upper-bound}]
    \renewcommand{\qedsymbol}{\ensuremath{\blacksquare}}
        By Equation~\eqref{eq:volterra-topm}, we have
        \begin{align*}
			\mathbb{E}[\mathcal{E}_t] &\eqsim M^{-s\beta} + e_M(t) +
		\int_{0}^{t} \mathcal{K}_M(t - \tau) \gamma_\tau
		(\mathbb{E}[\mathcal{E}_\tau] + \sigma^2) \dd \tau \\
		&\leq M^{-s\beta} + e_M(t) + \int_{0}^{t} \mathcal{K}_\infty(t - \tau)
		\gamma_\tau (\mathbb{E}[\mathcal{E}_\tau] + \sigma^2) \dd \tau.
        \end{align*}

		Define $f(t) := \int_{0}^{t} \mathcal{K}_\infty(t-\tau) \gamma_\tau
		(\mathbb{E}[\mathcal{E}_\tau] + \sigma^2) \dd \tau$ and $\gamma_{\max} = \sup_t \gamma(t)$,
        substituting the above inequality into the definition of $f(t)$, we get
		\begin{align*}
			f(t)
			&\lesssim \int_{0}^{t} \mathcal{K}_\infty(t-\tau) (M^{-s\beta} +
			e_M(\tau) + \sigma^2)\gamma_\tau \dd \tau\\
			&\quad + \int_{0}^{t} \int_{0}^{\tau}
			\mathcal{K}_\infty(t-\tau) \mathcal{K}_\infty(\tau-r)
			\gamma_\tau\gamma_r (\mathbb{E}[\mathcal{E}_r]+\sigma^2) \dd r \dd \tau\\
			&\lesssim \gamma_{\max} M^{-s\beta} + \int_{0}^{t} \mathcal{K}_\infty(t-\tau)
			(e_M(\tau)+\sigma^2) \gamma_\tau \dd \tau\\
			&\quad+ \gamma_{\max} \int_{0}^{t} \int_{r}^{t}
			\mathcal{K}_\infty(t-\tau)\mathcal{K}_\infty(\tau-r) \dd \tau
			\gamma_r (\mathbb{E}[\mathcal{E}_r]+\sigma^2) \dd r\\
			&= \gamma_{\max} M^{-s\beta} + \int_{0}^{t} \mathcal{K}_\infty(t-\tau) (e_M(\tau)
			+ \sigma^2) \gamma_\tau \dd \tau\\ &\quad+ \gamma_{\max} \int_{0}^{t}
			(\mathcal{K}_\infty * \mathcal{K}_\infty)(t-r) \gamma_r
			(\mathbb{E}[\mathcal{E}_r]+\sigma^2) \dd r \\
			&\overset{\text{Lemma~\ref{lemma:convolution-kernel}}}{\lesssim}
			\gamma_{\max} M^{-s\beta} + \int_{0}^{t} \mathcal{K}_\infty(t-\tau)
			(e_M(\tau)+\sigma^2) \gamma_\tau \dd \tau + \gamma_{\max} f(t)
		\end{align*}
        Therefore when $\gamma_{\max}$ is sufficiently small ($\gamma_{\max}
		\leq \frac{c}{\tr(\bH)}$ for some absolute constant $c$), the constant
		factor of $f(t)$ on the right-hand side will be less
		than $\frac{1}{2}$, hence we may subtract $\gamma_{\max}f(t)$ from both sides and get
		\[
			\int_{0}^{t} \mathcal{K}_\infty(t-\tau) \gamma_\tau
			\mathbb{E}[\mathcal{E}_\tau] \dd \tau
			\lesssim \gamma_{\max}M^{-s\beta} + \int_{0}^{t} \mathcal{K}_\infty(t-\tau)
			(e_M(\tau)+\sigma^2)\gamma_\tau \dd \tau.
		\]
		Therefore substituting this back to the Volterra equation yields
		\[
			\mathbb{E}[\mathcal{E}_t] \lesssim M^{-s\beta} + e_M(t)
			+ \int_{0}^{t} \mathcal{K}_\infty(t-\tau)
			(e_M(\tau)+\sigma^2) \gamma_\tau\dd \tau.
		\]
    \end{proof}
	
    The preceding proposition relies exclusively on two properties of $\mathcal{K}_\infty$: the identity established in Lemma~\ref{lemma:convolution-kernel} and the finiteness of its integral, $\int_{0}^{\infty} \cK_\infty(t) \dd t < \infty$.
    Since the kernel $\mathcal{K}_M$ also possesses these properties for $t \leq M^\beta$, the derivation of the upper bound
    \[
    \mathbb{E}[\mathcal{E}_t] \lesssim M^{-s\beta} + e_M(t)
 		+ \int_{0}^{t} \mathcal{K}_M(t - \tau) (e_M(\tau) + \sigma^2) \gamma_\tau
 		\dd \tau, \quad t\le M^{\beta}
    \]
    follows analogously by substituting $\mathcal{K}_\infty$ with $\mathcal{K}_M$.
\end{proof}

\subsubsection{Proof of Theorem~\ref{thm:fsl-const-noise}}
% \begin{theorem}
% 	[Functional Scaling Law in the noisy case]
% 	In the top-$M$ case, when $\sigma \geq 1$,
% 	we have
% 	\[
% 		\mathbb{E}[\mathcal{E}_t] \eqsim M^{-s\beta} + e_M(t)
% 		+ \int_{0}^{t} \mathcal{K}_M(t - \tau) (e_M(\tau) + \sigma^2) \gamma_\tau
% 		\dd \tau.
% 	\]
% \end{theorem}
\begin{proof}[Proof]
    It is clear that, by our assumption, $\sup_t\gamma_t \lesssim 1$, and that
    \[
    e_M(t) \eqsim \sum_{j=1}^M j^{-1-s\beta} e^{-2\lambda_j t}
    \le \sum_{j=1}^M j^{-1-s\beta} \eqsim 1.
    \]
    By Proposition~\ref{prop:fsl-weak-upper-bound}, we have
	\begin{align}
		\mathbb{E}[\mathcal{E}_t]
        &\lesssim M^{-s\beta} + e_M(t) + \int_0^t \cK_\infty(t-\tau)
        (e_M(t) + \sigma^2)\gamma_\tau \dd \tau\\
        &\lesssim 1 + \int_{0}^{t}
		\mathcal{K}_\infty(t-\tau)(\sigma^2+1) \dd \tau \lesssim \sigma^2,
	\end{align}
	where the last inequality is by $\sigma^2 \gtrsim 1$ and that
    \[
        \int_{0}^t \cK_\infty(\tau) \dd \tau
        = \int_0^t \sum_{j=1}^\infty \lambda_j^2 e^{-2\lambda_j \tau} \dd \tau =
        \frac{1}{2} \sum_{j=1}^\infty \lambda_j \eqsim 1.
    \]

	Therefore by the Volterra equation \eqref{eq:volterra-topm},
	note that $\mathbb{E}[\mathcal{E}_\tau] + \sigma^2 \eqsim \sigma^2
	\eqsim e_M(\tau) + \sigma^2$, we have
	\[
		\mathbb{E}[\mathcal{E}_t] \eqsim M^{-s\beta} + e_M(t)
		+ \int_{0}^{t} \mathcal{K}_M(t-\tau) (e_M(\tau) + \sigma^2) \gamma_\tau
		\dd \tau.
	\]
	
\end{proof}

\subsubsection{Proof of Theorem~\ref{thm:fsl-zero-noise}}
% \begin{theorem}
% 	[Functional Scaling Law in the noiseless case]
% 	In the top-$M$ case, when $s \leq 2 - \frac{1}{\beta}$, and $\sigma=0$,
% 	we have
% 	\[
% 		\mathbb{E}[\mathcal{E}_t] \eqsim M^{-s\beta} + e_M(t)
% 		+ \int_{0}^{t} \mathcal{K}_M(t - \tau) e_M(\tau) \gamma_\tau
% 		\dd \tau.
% 	\]
% \end{theorem}

\begin{proof}[Proof]
	When $\sigma^2 = 0$,
	by FSL in general case (Proposition~\ref{prop:fsl-weak-upper-bound}), we have
	\[
		\mathbb{E}[\mathcal{E}_t] \lesssim M^{-s\beta} + e_M(t)
		+ \int_{0}^{t} \mathcal{K}_\infty(t-\tau) e_M(\tau) \gamma_\tau \dd
		\tau.
	\]

    In order to get the upper bound with $\cK_M$, we first prove a lemma.
    \begin{lemma}
        We will bound the gap introduced by $\cK_\infty$ and $\cK_M$:
	\begin{equation}
    \label{eq:Ke-gap}
    	\int_{0}^{t} (\mathcal{K}_{\infty}(t - \tau) - \mathcal{K}_M(t - \tau)) e_M(\tau)
    	\gamma_\tau \dd \tau \lesssim M^{\max\{-s\beta, -2\beta+1\}}.
	\end{equation}
    \end{lemma}
    \begin{proof}[Proof of Lemma]
    \renewcommand{\qedsymbol}{\ensuremath{\blacksquare}}
	First note that
	\[
	\mathcal{K}_{\infty}(t) - \mathcal{K}_M(t) = \sum_{j=M+1}^{\infty}
	\lambda_j^2 e^{-2\lambda_j t} \lesssim \sum_{j=M+1}^{\infty} j^{-2\beta}
	\eqsim M^{-2\beta + 1}.
	\]
	Therefore we can bound the integral as
	\begin{align*}
		&\int_{0}^{t} (\mathcal{K}_\infty(t - \tau) - \mathcal{K}_M(t - \tau))
		e_M(\tau) \gamma_\tau \dd \tau\\
		&\lesssim M^{-2\beta+1} \int_{0}^{t} e_M(\tau)\gamma_\tau\dd \tau\\
		&\leq \gamma_{\max} M^{-2\beta+1} \int_{0}^{\infty} e_M(\tau) \dd \tau\\
		&= \gamma_{\max} M^{-2\beta+1} \int_{0}^{\infty} \sum_{j=1}^{M}
		\lambda_j (\theta_j^*)^2 e^{-2\lambda_j \tau} \dd \tau\\
		&\eqsim \gamma_{\max} M^{-2\beta+1} \sum_{j=1}^{M}j^{-s\beta-1+\beta}\\
		&\lesssim \gamma_{\max} M^{\max\{-s\beta-\beta+1, -2\beta+1\}}
		\lesssim \gamma_{\max} M^{\max\{-s\beta, -2\beta+1\}}
	\end{align*}
    \end{proof}
	
	By $s \leq 2 - \frac{1}{\beta}$ and $\gamma_{\max}$ is
	sufficiently small, we can combine
	the upper bound with \eqref{eq:Ke-gap}, and
	directly conclude that
	\[
		\mathbb{E}[\mathcal{E}_t] \lesssim M^{-s\beta} + e_M(t)
		+ \int_{0}^{t} \mathcal{K}_M(t-\tau)e_M(\tau)\gamma_\tau \dd \tau.
	\]
	Now with the lower bound in FSL Theorem~\ref{thm: fsl-general-noise},
	the result follows.
\end{proof}

\subsubsection{Proof of Theorem~\ref{thm: fsl-hard-regime}}
\begin{proof}[Proof]
	In the hard regime $s\le 1 - \frac{1}{\beta}$,
	by Propostion~\ref{prop:fsl-weak-upper-bound}, we have
	\[
		\mathbb{E}[\mathcal{E}_t] \lesssim M^{-s\beta} + e_M(t)
		+ \int_{0}^{t} \mathcal{K}_\infty(t-\tau) (e_M(\tau) + \sigma^2) \gamma_\tau \dd
		\tau.
	\]

    \begin{lemma}
        It holds for $s\le 1-\frac{1}{\beta}$ that
	\begin{equation}
    \label{eq:Ksigma-gap}
    	\int_{0}^{t} (\mathcal{K}_{\infty}(t - \tau) - \mathcal{K}_M(t - \tau))
    	\gamma_\tau \dd \tau \lesssim M^{-\beta+1}.
	\end{equation}
    \end{lemma}
    \begin{proof}[Proof of Lemma]
    \renewcommand{\qedsymbol}{\ensuremath{\blacksquare}}
	First from the definition of $\cK$ we obtain
	\[
	\mathcal{K}_{\infty}(t) - \mathcal{K}_M(t) = \sum_{j=M+1}^{\infty}
	\lambda_j^2 e^{-2\lambda_j t}.
	\]
	Therefore, we have
	\begin{align*}
		&\int_{0}^{t} (\mathcal{K}_\infty(t - \tau) - \mathcal{K}_M(t - \tau))
		\gamma_\tau \dd \tau\\
		&= \gamma_{\max} \int_{0}^{\infty} \sum_{j=1}^{M}
		\lambda_j^2 e^{-2\lambda_j \tau} \dd \tau\\
		&\eqsim \gamma_{\max} \sum_{j=1}^{M}j^{-\beta}
		\lesssim \gamma_{\max} M^{-\beta+1}.
	\end{align*}
    \end{proof}
	
	By $s \leq 1 - \frac{1}{\beta}$ and that $\gamma_{\max}$ is
	sufficiently small, we can combine
	the upper bound with \eqref{eq:Ksigma-gap}, \eqref{eq:Ke-gap}, and
	directly conclude that
	\begin{align*}
		\mathbb{E}[\mathcal{E}_t]
        &\lesssim M^{-s\beta} + e_M(t) + \int_{0}^{t} \mathcal{K}_\infty(t-\tau)(e_M(\tau)+\sigma^2)\gamma_\tau \dd \tau\\
        &\lesssim M^{-s\beta} + e_M(t)
		+ \int_{0}^{t} \mathcal{K}_M(t-\tau)(e_M(\tau)+\sigma^2)\gamma_\tau \dd \tau \\
        &\qquad+ \int_{0}^{t} (\mathcal{K}_{\infty}(t - \tau) - \mathcal{K}_M(t - \tau)) (e_M(\tau) + \sigma^2)
    	\gamma_\tau \dd \tau\\
        {\small (\text{by }\eqref{eq:Ke-gap},\eqref{eq:Ksigma-gap})}\quad
        &\lesssim M^{-s\beta} + e_M(t)
		+ \int_{0}^{t} \mathcal{K}_M(t-\tau)(e_M(\tau)+
        \sigma^2)\gamma_\tau \dd \tau\\
        &\qquad + M^{\max\{-s\beta, -2\beta+1\}} + \sigma^2 M^{-\beta+1}\\
        &\lesssim M^{-s\beta} + e_M(t)
		+ \int_{0}^{t} \mathcal{K}_M(t-\tau)(e_M(\tau)+
        \sigma^2)\gamma_\tau \dd \tau.
	\end{align*}
	Combining with the lower bound in FSL Theorem~\ref{thm: fsl-general-noise},
	the result follows.
\end{proof}

\subsection{The Case of Random-$M$ Features}
\label{appendix:proof_random}

\textbf{Notations.} Throughout this section, for a power series
$
P(x) = \sum_{i=0}^{\infty} a_i x^i,
$
we write
$
P(\bA) = \sum_{i=0}^{\infty} a_i \bA^i
$
to denote the result of substituting a square matrix $\bA\in \mathbb{R}^{d\times d}$ into $P(\cdot)$, where the power $\bA^i$ is obtained through matrix multiplications.

For a vector $\bv\in \mathbb{R}^d$, we write the norm $\lVert \bv\rVert$ for $\lVert \bv\rVert_2$ if not otherwise specified.
For a matrix $\bA\in \mathbb{R}^{d\times d}$, we use $\lVert \bA\rVert_2$ to denote the operator norm induced by the vector 2-norm, and $\mu_1(\bA) \ge \mu_2(\bA) \ge \dots \ge \mu_d(\bA)$ to denote the eigenvalues of $\bA$.
For a positive semi-definite matrix $\bC\in \mathbb{R}^{d\times d}$ and a vector $\bv\in \mathbb{R}^d$,
we define $\lVert \bv\rVert_{\bC}$ as
\[
\lVert \bv\rVert_{\bC} = \bv^\top \bC \bv.
\]

First, for the random-$M$ feature setting, we can establish the following Volterra equation.

\begin{proposition}
\label{prop:volterra-random}
Suppose $0 < s \le 1$. Then, with probability at least $1 - e^{-\Omega(M)}$,  
a similar Volterra equation as derived in Theorem~\ref{thm:volterra-topm}
continues to hold for the random-feature case; that is,
\begin{equation}
\EE[\cE_t] 
\;\eqsim\;
M^{-s\beta} + \widehat{e}_M(t)
+ \int_{0}^{t} \widehat{\cK}_M(t-\tau)\,
\gamma_\tau \bigl(\EE[\cE_\tau] + \sigma^2\bigr)\, \dd \tau,
\end{equation}
where
\[
\widehat{e}_M(t) := \sum_{j=1}^M \hlambda_j |\theta^*_j|^2 e^{-2\hlambda_jt}, \quad \widehat{\cK}_M := \sum_{j=1}^M \hlambda_j^2 e^{-2\hlambda_jt},
\]
and $\hlambda_j := \mu_j(\bW\bH\bW^\top)$ is the $j$-th eigenvalue
of the random matrix.
\end{proposition}

Theorem~\ref{thm: fsl-random} then follows by applying exactly the same argument as in the top-$M$ case.  
It therefore remains to prove Proposition~\ref{prop:volterra-random}.  
The key idea is to show that the spectrum of the random matrix $\bW \bH \bW^\top \in \RR^{M\times M}$  
closely matches that of the top-$M$ truncation, namely,
\[
  \hlambda_j := \mu_j(\bW\bH\bW^\top) \;\eqsim\; \lambda_j, 
  \qquad \forall\; 1 \le j \le M.
\]

\subsubsection{Concentration Inequalities}
Recall that we derived the following recursive equation in Eq.~\eqref{eq:risk-recursion}:
\[
2\EE \cE_t = \bu_0^\top \bA^\top_t
\bH  \bA_t \bu_0 +
\int_{0}^{t}\tr(\mathbf{S} \bA^\top_{t-\tau}
\bH  \bA_{t-\tau} \mathbf{S}) \cdot
\gamma_\tau(c_\tau \EE[\cE_\tau]+ \sigma^2) \dd \tau, 
\]
where $\bA_t = e^{-\bW ^\top \bW  \bH  t}$ and $\mathbf{S} = (\bW ^\top \bW  \bH  \bW ^\top \bW )^{\frac{1}{2}}$.

 We first introduce the following notation:
for integers $0\le a < b\le N$ (we allow $b = \infty$, in this case we regard it as
the same as $b = N$),
\[
\quad \bH _{a:b} = \operatorname{diag}\{\lambda_{a+1}, \dots, \lambda_b\}
\in \mathbb{R}^{(b-a)\times (b-a)},
\quad \bu_{a:b} = ((\bu)_{a+1}, \dots,
(\bu)_{b}) \in \mathbb{R}^{b-a},
\]
while
\[
\bW _{a:b} = [\bW _{a+1}, \dots, \bW _b]
\in \mathbb{R}^{M\times (b-a)}
\]
is the $(a+1)$-th to $b$-th columns of $\bW $.

To understand this equation with random projection matrix $\bW $,
we leverage the following concentration results developed in~\citep{lin2024scaling}.
\begin{lemma}
    [Lemma G.4 in \citep{lin2024scaling}]
    \label{lem:eig_concentration_appendix}
    For $\bH = \diag\{\lambda_1, \lambda_2, \dots, \lambda_N\}$ such that $\lambda_j \eqsim j^{-\beta}$ where $\beta>1$,
    there exists $\beta$-dependent constants $0 < c_1 < c_2$ such that it holds with probability at least $1 - e^{-\Omega(M)}$ for all $j\in [M]$ that
    \[
    c_1 \lambda_j \le \mu_j(\bW \bH \bW ^\top) \le c_2\lambda_j
    \]
\end{lemma}

\begin{lemma}
    [Lemma G.5 in \citep{lin2024scaling}]
    \label{lem:eig_ratio}
    There exists some $\beta$-dependent constant $c$ such that
    for all $k\ge 1$, the ratio between the $\frac{M}{2}$-th
    and $M$-th eigenvalue
    \[
    \frac{\mu_{\frac{M}{2}}(\bW _{k:\infty}\bH _{k:\infty}\bW _{k:\infty}^\top)}
    {\mu_{M}(\bW _{k:\infty}\bH _{k:\infty}\bW _{k:\infty}^\top)} \le c
    \]
    with probability at least $1 - e^{-\Omega(M)}$.
\end{lemma}

\subsubsection{Upper and Lower Bounds}

% Let $\hlambda_j=\mu_j(\bW \bH \bW ^\top)$.
\begin{lemma}
    \label{lem:trace_analysis_appendix}
    With probability at least $1 - e^{-\Omega(M)}$, for $s>0$ we have
    \[
    \tr(\mathbf{S}\bA^\top_{t-\tau} \bH 
    \bA_{t-\tau} \mathbf{S}) = \sum_{j=1}^{M}
    e^{-2(t-\tau)\hlambda_j} \hlambda_j^2
    =: \widehat{\cK}_M(t-\tau).
    \]
\end{lemma}
\begin{proof}[Proof]
We can compute that
\begin{align*}
    \tr(\mathbf{S}\bA^\top_{t-\tau} \bH  \bA_{t-\tau}\mathbf{S})
    &= \tr(\bW ^\top \bW \bH \bW ^\top \bW  \bA^\top_{t-\tau}\bH \bA_{t-\tau})\\
    &= \tr\left( \bW ^\top \bW \bH \bW ^\top \bW  \sum_{a, b=0}^{\infty} \frac{1}{a!b!}
        (-(t-\tau))^{a+b}(\bH \bW ^\top \bW )^a \bH  (\bW ^\top \bW  \bH )^b \right) \\
    &= \tr\left(\sum_{a,b=0}^{\infty} \frac{1}{a!b!} (-t+\tau)^{a+b}\cdot
    \bW ^\top (\bW  \bH  \bW ^\top)^{a+b+1} \bW \bH \right)\\
    &= \tr\left( \sum_{a,b=0}^{\infty} \frac{1}{a!b!}(-t+\tau)^{a+b}\cdot
    (\bW \bH \bW ^\top)^{a+b+2}\right) \\
    &= \sum_{a,b=0}^{\infty} \frac{1}{a!b!}(-t+\tau)^{a+b}
    \sum_{j=1}^{M} \hlambda_j^{a+b+2} \\
    &= \sum_{j=1}^{M} e^{-2(t-\tau)\hlambda_j} \hlambda_j^2,
\end{align*}
which completes the proof.
\end{proof}

For the first term in Eq.~\eqref{eq:risk-recursion}, following
\citep{lin2024scaling}, we have
\begin{lemma}
    \label{lem:bias_upper_bound}
    With probability at least $1 - e^{-\Omega(M)}$, for $0 < s\le 1$ we have
    \[
    \bu_0^\top \bA^\top_t\bH  \bA_t
    \bu_0 \lesssim M^{-s\beta} + \widehat{e}_M(t).
    \]
\end{lemma}
\begin{proof}
Note 
\begin{align*}
    \bA^\top_t \bH  \bA_t
    &= \sum_{a,b=0}^{\infty} \frac{1}{a!b!}(-t)^{a+b}
    (\bH \bW ^\top \bW )^{a}\bH (\bW ^\top \bW \bH )^{b}\\
    &= \sum_{a,b=0}^{\infty} \frac{1}{a!b!}(-t)^{a+b}
    \bH \bW ^\top
    (\bW \bH \bW ^\top)^{a+b-1}
    \bW \bH \\
    &= \bH \bW ^\top (\bW \bH \bW ^\top)^{-1}\mathbf{M}_t(\bW \bH \bW ^\top)^{-1} \bW \bH 
\end{align*}
where $\mathbf{M}_t = P_t(\bW \bH \bW ^\top)$ with $P_t$ being the power series
\begin{align*}
P_t(x) 
&:= \sum_{a+b\ge 0} \frac{1}{a!b!}(-t)^{a+b}x^{a+b+1}.
\end{align*}
Note that when $x\in \mathbb{R}$, we have 
$$
P_t(x) =  x \sum_{a, b\ge 0} \frac{(-tx)^a}{a!}\cdot \frac{(-tx)^b}{b!}
= x e^{-tx} \cdot e^{-tx} = xe^{-2tx}.
$$
Hence the eigenvalues of $\mathbf{M}_t$ is exactly $P_t(\hlambda_j)$.
Since
\[
\bu_0^\top \bA^\top_t\bH \bA_t\bu_0
= \bu_0^\top (\bH \bW ^\top (\bW \bH \bW ^\top)^{-1}\mathbf{M}_t(\bW \bH \bW ^\top)^{-1} \bW \bH )\bu_0,
\]
for any positive integer $k\le \frac{M}{2}$,
note that $\bW \bH \bu =
\bW _{0:k}\bH _{0:k}\bu_{0:k}
+ \bW _{k:\infty}\bH _{k:\infty}\bu_{k:\infty}$,
we have
\[
\bu_0^\top \bA^\top_t\bH 
\bA_t\bu_0 \le 2(T_1 + T_2),
\]
where
\begin{align*}
    T_1 &= \bu_{0:k}^\top (\bH _{0:k}\bW _{0:k}^\top (\bW \bH \bW ^\top)^{-1}\mathbf{M}_t(\bW \bH \bW ^\top)^{-1} \bW _{0:k}\bH _{0:k})\bu_{0:k},\\
    T_2 &= \bu_{k:\infty}^\top (\bH _{k:\infty} \bW _{k:\infty}^\top
    (\bW \bH \bW ^\top)^{-1}\mathbf{M}_t(\bW \bH \bW ^\top)^{-1}\bW _{k:\infty}\bH _{k:\infty})\bu_{k:\infty}.
\end{align*}

Then by Lemma \ref{lem:T_estimate}, we can derive an upper bound.
Since $s \le 1$,
\begin{align*}
    T_1 + T_2 &\lesssim \frac{1}{t}\lVert \bu_{0:k} \rVert ^2_2
    + \lVert \bu_{k:\infty} \rVert _{\bH _{k:\infty}}^2\\
    &\eqsim \frac{1}{t} \sum_{j=1}^{k} j^{-1 +\beta(1-s)}
    + \sum_{j=k+1}^{N} j^{-1 - s\beta}
    \eqsim \frac{k^{\beta(1-s)}}{t} + k^{- s\beta}.
\end{align*}

By setting $k = \min\{t^{1/\beta}, \frac{M}{3}\}$, we have
\[
\bu_0^\top \bA^\top_t \bH  \bA_t\bu_0 \lesssim \max\{t^{-s},
M^{-s\beta}\} \eqsim t^{-s} + M^{-s\beta}.
\]
Clearly $\bu_0^\top \bA^\top_t \bH \bA_t\bu_0 \lesssim 1$,
therefore by Lemma~\ref{lemma:em-gap},
\[
\bu_0^\top \bA^\top_t \bH  \bA_t\bu_0 \lesssim \min\{t^{-s},
1\}+M^{-s\beta} \eqsim \widehat{e}_\infty(t) + M^{-s\beta}
\eqsim \widehat{e}_M(t) + M^{-s\beta}.
\]
\end{proof}

\begin{lemma}
    \label{lem:bias_lower_bound}
    For $s>0$, it holds with probability at least $1 - e^{-\Omega(M)}$ that
    \[
    \bu_0^\top \bA^\top_t \bH 
    \bA_t \bu_0 \gtrsim \max\{\widehat{e}_M(t),
    M^{-s\beta}\},
    \]
    where $\widehat{e}_M(t) := \sum_{j=1}^M \hlambda_j|\theta_j^*|^2 e^{-2t\hlambda_j}$.
\end{lemma}
\begin{proof}[Proof]
    Following the proof of Lemma \ref{lem:bias_upper_bound}, we have
    \begin{align*}
        \bu_0^\top \bA^\top_t \bH 
        \bA_t \bu_0
        &= \bu_0 \bH \bW ^\top
        (\bW \bH \bW ^\top)^{-1} \mathbf{M}_t
        (\bW \bH \bW ^\top)^{-1} \bW 
        \bH  \bu_0\\
        &= \tr\left((\bW \bH \bW ^\top)^{-1}
        \mathbf{M}_t (\bW \bH \bW ^\top)^{-1}
        \cdot \bW \bH \bu_0\bu_0^\top
        \bH \bW ^\top\right)\\
        &\ge \sum_{i=1}^{M} \mu_{M-i+1}\left((\bW \bH 
        \bW ^\top)^{-1}\mathbf{M}_t (\bW \bH 
        \bW ^\top)^{-1}\right) \cdot \mu_i\left(\bW \bH 
        \bu_0\bu_0^\top
        \bH \bW ^\top\right),
    \end{align*}
    where the last inequality is by Von Neumann's trace inequality:  For two symmetric matrices $\bA$, $\bB\in \mathbb{R}^{d\times d}$ with eigenvalues $a_1\ge a_2 \ge \dots \ge a_d$ and $b_1 \ge b_2 \ge \dots \ge b_d$, we have
    \[
    \tr(\bA \bB) \ge \sum_{i=1}^d a_i b_{d+1-i}.
    \]
    Note that $\mathbf{M}_t =
    P_t(\bW \bH \bW ^\top)$, we then get
    \begin{align*}
        \bu_0^\top \bA^\top_t \bH 
        \bA_t \bu_0
        &\ge \sum_{i=1}^{M}
        \mu_i\left((\bW \bH \bW ^\top)^2
        \mathbf{M}_t^{-1}\right)^{-1}
        \mu_i\left(\bW \bH 
        \bu_0\bu_0^\top
        \bH \bW ^\top\right)\\
        &= \sum_{i=1}^{M} e^{-2t\hlambda_i} \hlambda_i^{-1}
        \mu_i\left(\bW \bH 
        \bu_0\bu_0^\top
        \bH \bW ^\top\right).
    \end{align*}
    Note that $\bu_0 = \btheta^*$, by Assumption
    \ref{ass:capacity} and \ref{ass:source},
    \begin{align*}
        \bu_0^\top \bA^\top_t \bH 
        \bA_t \bu_0
        &\ge \sum_{i=1}^{M} e^{-2t\hlambda_i} \hlambda_i^{-1}
        \mu_i\left(\bW \bH 
        \bu_0\bu_0^\top
        \bH \bW ^\top\right)\\
        &\eqsim \sum_{i=1}^{M} e^{-2t\hlambda_i} \hlambda_i^{-1}
        i^{-1-\beta-s\beta}\\
        &\eqsim \sum_{i=1}^{M} e^{-2t\hlambda_i}
        \hlambda_i|\theta^*_j|^{2} =  \widehat{e}_M(t).
    \end{align*}
    Here in the second line we used Lemma \ref{lem:eig_concentration_appendix} for matrix $\bH\bu_0\bu_0^\top \bH$ as
    \[
    \mu_j(\bH\bu_0\bu_0^\top \bH) = \lambda_j^2 |\theta_j^*|^2 \eqsim j^{-1-\beta-s\beta} \implies \mu_j(\bW \bH 
        \bu_0\bu_0^\top
        \bH \bW ^\top) \eqsim j^{-1-\beta-s\beta},
    \]
    and the third line follows from again from Lemma~\ref{lem:eig_concentration_appendix} that $\hlambda_j \eqsim j^{-\beta}$.

    On the other hand, we prove the lower bound on $M^{-s\beta}$. First, we claim that
    \[
    \bu_0^\top \bA^\top_t \bH 
    \bA_t \bu_0 
    \ge 
    \|
    (\textbf{I} - \bH ^{\frac{1}{2}} \bW ^\top (\bW \bH \bW ^\top)^{-1} \bW \bH ^{\frac{1}{2}})
    \bH ^{\frac{1}{2}}\bu_0
    \|^2 =: T_3,
    \]
    which will be proved in the Lemma \ref{lm:matrix_analysis}

    Notice that
    \begin{align*}
        T_3
        &= \left\langle 
            \mathbf{I}_{N} -
            \bH^{1/2} \bW^\top (\bW\bH\bW^\top)^{-1} \bW \bH^{1/2},\,
            \bH^{1/2}\bu_0\bu_0^\top \bH^{1/2}
        \right\rangle,
    \end{align*}
    where the inner product
    $
    \langle \mathbf{A}, \mathbf{B} \rangle 
    = \operatorname{tr}(\mathbf{A}^\top \mathbf{B})
    $
    denotes the trace inner product between matrices.
    Therefore note that $\mu_i(\bH^{\frac{1}{2}}\bu_0\bu_0^\top
	\bH^{\frac{1}{2}}) = i^{-1-s\beta}$ by source and capacity conditions,
    \begin{align*}
        T_3&\ge \sum_{i=1}^{N} \mu_i \left( \textbf{I}_{N} - \bH
		^{1/2} \bW^\top (\bW\bH\bW^\top)^{-1} \bW \bH
	^{1/2} \right) \cdot \mu_{N+1-i} (\bH^{\frac{1}{2}}\bu_0\bu_0^\top
	\bH^{\frac{1}{2}})\\
        &\gtrsim \sum_{i=1}^{N} \mu_i(\mathbf{M}) \cdot (N+1-i)^{-1-s\beta}
    \end{align*}
    where the second line
	follows again from Von Neumann’s trace inequality. Since $\mathbf{M} = \mathbf{I}_{N} -
	\bH^{1/2} \bW^\top (\bW\bH\bW^\top)^{-1} \bW \bH
	^{1/2}$ is a projection matrix such that $\mathbf{M}^2 =
	\mathbf{M}$ and $\operatorname{rank}(\mathbf{I}_{N} - \mathbf{M}) = M$ with
	probability $1$,
	$\mathbf{M}$ must have $M$ eigenvalues $0$ and $N-M$ eigenvalues $1$.

    Hence we have
    \[
    T_3\gtrsim \sum_{i=M}^{N} i^{-1-s\beta} \gtrsim M^{-s\beta}.
    \]
\end{proof}

\begin{lemma}
    \label{lm:matrix_analysis}
    \[
    \bu_0^\top \bA^\top_t \bH 
    \bA_t \bu_0 
    \ge 
    \|
    (\mathbf{I} - \bH ^{\frac{1}{2}} \bW ^\top (\bW \bH \bW ^\top)^{-1} \bW \bH ^{\frac{1}{2}})
    \bH ^{\frac{1}{2}}\bu_0
    \|^2
    \]
\end{lemma}

\begin{proof}
    By the definition of positive semi-definite, we only need to prove that
    \[
    \bA^\top_t \bH 
    \bA_t 
    \succeq 
    \bH ^{\frac{1}{2}} 
    (\mathbf{I} - \bH ^{\frac{1}{2}} \bW ^\top (\bW \bH \bW ^\top)^{-1} \bW \bH ^{\frac{1}{2}})^2
    \bH ^{\frac{1}{2}}
    \]
    Notice that
    \begin{align*}
        \bA^\top_t \bH 
        \bA_t
        &= e^{-\bH \bW ^\top\bW t} \bH  e^{-\bW ^\top\bW \bH t}\\
        &= \bH ^{\frac{1}{2}}
        \left(\mathbf{I} + \sum_{a+b\ge 1} \frac{1}{a!b!} (-t)^{a+b}\bH ^{\frac{1}{2}}\bW ^\top        (\bW \bH \bW ^\top)^{-1} 
        \bW \bH ^{\frac{1}{2}}
        \right)
        \bH ^{\frac{1}{2}}\\
    \end{align*}
    Notice that $\bH $ is a positive definite matrix, and now we only need to prove
    \[
    \mathbf{I} + \sum_{a+b\ge 1} \frac{1}{a!b!} (-t)^{a+b}\bH ^{\frac{1}{2}}\bW ^\top        (\bW \bH \bW ^\top)^{a+b-1} 
        \bW \bH ^{\frac{1}{2}}
    \succeq 
    (\mathbf{I} - \bH ^{\frac{1}{2}} \bW ^\top (\bW \bH \bW ^\top)^{-1} \bW \bH ^{\frac{1}{2}})^2.
    \]
    Let $\mathbf{P} = \bW \bH ^{\frac{1}{2}}$. After simplification, we only need to prove that
    \[
    \mathbf{I} + \sum_{a+b\ge 1} \frac{1}{a!b!} (-t)^{a+b}\mathbf{P}^\top(\mathbf{P}\mathbf{P}^\top)^{a+b-1}\mathbf{P}
    \succeq
    \mathbf{I} - \mathbf{P}^\top(\mathbf{P}\mathbf{P}^\top)^{-1}\mathbf{P}.
    \]
    Notice that, by the definition of matrix exponential, we have
    \begin{align*}
        &\mathbf{I} + \sum_{a+b\ge 1} \frac{1}{a!b!} (-t)^{a+b}\mathbf{P}^\top(\mathbf{P}\mathbf{P}^\top)^{a+b-1}\mathbf{P}\\
        &= \mathbf{I} - \mathbf{P}^\top(\mathbf{P}\mathbf{P}^\top)^{-1}\mathbf{P}
        + \mathbf{P}^\top(\mathbf{P}\mathbf{P}^\top)^{-1}
        \left(\sum_{a+b\ge 0} \frac{2^{a+b}}{(a+b)!} (-t)^{a+b}(\mathbf{P}\mathbf{P}^\top)^{a+b}\right)
        \mathbf{P}\\
        &= \mathbf{I} - \mathbf{P}^\top(\mathbf{P}\mathbf{P}^\top)^{-1}\mathbf{P}
        + \mathbf{P}^\top(\mathbf{P}\mathbf{P}^\top)^{-1}
        e^{-2\mathbf{P}\mathbf{P}^\top t}
        \mathbf{P}.
    \end{align*}
    Notice that the matrix $\mathbf{P}^\top\mathbf{P}$ and $e^{-2\mathbf{P}\mathbf{P}^\top t}$ are both positive semi-definite, we have $\mathbf{P}^\top(\mathbf{P}\mathbf{P}^\top)^{-1}
    e^{-2\mathbf{P}\mathbf{P}^\top t}
    \mathbf{P}$ is positive semi-definite. As a result,
    \[
    \mathbf{I} + \sum_{a+b\ge 1} \frac{1}{a!b!} (-t)^{a+b}\mathbf{P}^\top(\mathbf{P}\mathbf{P}^\top)^{a+b-1}\mathbf{P}
    \succeq
    \mathbf{I} - \mathbf{P}^\top(\mathbf{P}\mathbf{P}^\top)^{-1}\mathbf{P}.
    \]
    which completes the proof.
\end{proof}

\begin{lemma}
    \label{lem:T_estimate}
    With probability $1 - e^{-\Omega(M)}$, we have
    \[
    T_1 \le c \frac{\lVert \bu_{0:k} \rVert _2^2}{t}
    \left( \frac{\mu_{\frac{M}{2}}(\bW _{0:k}\bH _{0:k}\bW _{0:k}^\top)}
    {\mu_M(\bW _{0:k}\bH _{0:k}\bW _{0:k}^\top)} \right) ^2,
    \quad T_2 \le \lVert \bu_{k:\infty} \rVert _{\bH _{k:\infty}}^2.
    \]
    where $c$ is some constant.
\end{lemma}
\begin{proof}[Proof]
    Recall that the eigenvalues of $\mathbf{M}_t = P_t(\bW\bH\bW^\top)$ is $P_t(\hlambda_j)$,
    so we have
    \[
    \|\mathbf{M}_t\|_2 = \max_{1\le j \le M} P_t(\hlambda_j) =
    \max_{1\le j\le M} \hlambda_j e^{-2t\hlambda_j} = \frac{1}{2t} \max_{1\le j\le M} 2t\hlambda_j e^{-2t\hlambda_j}
    \le \frac{1}{2t}\sup_{x>0} xe^{-x} \le \frac{1}{2et},
    \]
    where the last step follows from $\sup_{x>0}xe^{-x}\leq e^{-1}$.
    
    By the definition of $T_1$, we have
    \begin{align*}
        T_1 &\le \|\bH _{0:k}\bW _{0:k}^\top (\bW \bH \bW ^\top)^{-1}\mathbf{M}_t(\bW \bH \bW ^\top)^{-1} \bW _{0:k}\bH _{0:k}\| \|\bu_{0:k}\|_2^2 \\
        &\le \|\mathbf{M}_t\|_2 \|(\bW \bH \bW ^\top)^{-1}\bW _{0:k}\bH _{0:k}\|_2^2\|\bu_{0:k}\|_2^2\\
        &\le \frac{c}{t}\|(\bW \bH \bW ^\top)^{-1}\bW _{0:k}\bH _{0:k}\|_2^2\|\bu_{0:k}\|_2^2.
    \end{align*}
    We only need to show
    \[
    \|(\bW \bH \bW ^\top)^{-1}\bW _{0:k}\bH _{0:k}\|_2\le c\left( \frac{\mu_{\frac{M}{2}}(\bW _{k:\infty}\bH _{k:\infty}\bW _{k:\infty}^\top)}{\mu_M(\bW _{k:\infty}\bH _{k:\infty}\bW _{k:\infty}^\top)} \right).
    \]
    We denote $\bA_k = \bW _{k:\infty}\bH _{k:\infty}\bW _{k:\infty}^\top$, and since $\bW \bH \bW ^\top = \bW _{0:k}\bH _{0:k}\bW _{0:k}^\top + \bA_k$, we have
    \begin{align*}
    (\bW \bH \bW ^\top)^{-1}\bW _{0:k}\bH _{0:k} &= (\bA_k^{-1} - \bA_k^{-1}\bW _{0:k}[\bH _{0:k}^{-1} + \bW _{0:k}^\top \bA_k^{-1}\bW _{0:k}]^{-1}\bW _{0:k}^\top \bA_k^{-1}\bW _{0:k})\bW _{0:k}\bH _{0:k} \\
    &= \bA_k^{-1}\bW _{0:k}\bH _{0:k} - \bA_k^{-1}\bW _{0:k}[\bH _{0:k}^{-1} + \bW _{0:k}^\top \bA_k^{-1}\bW _{0:k}]^{-1}\bW _{0:k}^\top \bA_k^{-1}\bW _{0:k}\bH _{0:k} \\
    &= \bA_k^{-1}\bW _{0:k}[\bH _{0:k}^{-1} + \bW _{0:k}^\top \bA_k^{-1}\bW _{0:k}]^{-1}\bH _{0:k}^{-1}\bH _{0:k} \\
    &= \bA_k^{-1}\bW _{0:k}[\bH _{0:k}^{-1} + \bW _{0:k}^\top \bA_k^{-1}\bW _{0:k}]^{-1}
    \end{align*}
    where the first line uses Sherman–Morrison-Woodbury’s identity: For matrices $\bA$, $\bC$, $\bU$, $\bV$ with suitable shapes, we have
    \[
    (\bA + \bU\bC\bV)^{-1} = \bA^{-1} - \bA^{-1} \bU (\bC^{-1} + \bV \bA^{-1}\bU)^{-1} \bV\bA^{-1}.
    \]
    Since
    \[
    \bH _{0:k}^{-1} + \bW _{0:k}^\top \bA_k^{-1}\bW _{0:k} \succeq \bW _{0:k}^\top \bA_k^{-1}\bW _{0:k}.
    \]
    it follows that
    \begin{align*}
    \|[\bH _{0:k}^{-1} + \bW _{0:k}^\top \bA_k^{-1} \bW _{0:k}]^{-1}\|_2 \leq \|[\bW _{0:k}^\top \bA_k^{-1} \bW _{0:k}]^{-1}\|_2.
    \end{align*}
    Therefore, with probability at least $1 - e^{-\Omega(M)}$
    \begin{align*}
    \|\bA_k^{-1} \bW _{0:k} [\bH _{0:k}^{-1} + \bW _{0:k}^\top \bA_k^{-1} \bW _{0:k}]^{-1}\|_2 &\leq \|\bA_k^{-1}\|_2 \cdot \|\bW _{0:k}\|_2 \cdot \|[\bH _{0:k}^{-1} + \bW _{0:k}^\top \bA_k^{-1} \bW _{0:k}]^{-1}\|_2 \\
    &\leq \|\bA_k^{-1}\|_2 \cdot \|\bW _{0:k}\|_2 \cdot \|[\bW _{0:k}^\top \bA_k^{-1} \bW _{0:k}]^{-1}\|_2 \\
    &\leq \frac{\|\bA_k^{-1}\|_2 \cdot \|\bW _{0:k}\|_2}{\mu_{\min}(\bW _{0:k}^\top \bA_k^{-1} \bW _{0:k})}.
    \end{align*}
    Assume $k \le \frac{M}{2}$ and with probability at least $1-e^{-\Omega(M)}$ for some constant $c>0$, $\|\bW _{0:k}\|_2 \le c$.
    
    We may write $\bW _{0:k}^\top \bA_k^{-1} \bW _{0:k} = \sum_{i=1}^M \frac{1}{\widehat{\lambda}_{M-i}} \boldsymbol{s}_i \boldsymbol{s}_i^\top$, where $\boldsymbol{s}_i \overset{\text{i.i.d.}}{\sim} \mathcal{N}(0, \mathbf{I}_k/N)$ and $(\widehat{\lambda}_i)_{i=1}^M$ are eigenvalues of $\bA_k$ in non-increasing order.
    Therefore, for $k \leq M/3$,
    \[
    \sum_{i=1}^{M} \frac{1}{\widehat{\lambda}_{M-i}} \boldsymbol{s}_i \boldsymbol{s}_i^\top \succeq \sum_{i=1}^{M/2} \frac{1}{\widehat{\lambda}_{M-i}} \boldsymbol{s}_i \boldsymbol{s}_i^\top \succeq \frac{1}{\widehat{\lambda}_{M/2}} \sum_{i=1}^{M/2} \boldsymbol{s}_i \boldsymbol{s}_i^\top \succeq \frac{c \mathbf{I}_k}{\widehat{\lambda}_{M/2}}.
    \]
    with probability at least $1-e^{-\Omega(M)}$, where in the last inequality we again use the concentration properties of Gaussian covariance matrices (see e.g., Theorem 6.1 in \citep{wainwright2019high}).
    \begin{align*}
        \|\bA_k^{-1} \bW _{0:k} [\bH _{0:k}^{-1} + \bW _{0:k}^\top \bA_k^{-1} \bW _{0:k}]^{-1}\|_2 &\lesssim \frac{\|\bA_k^{-1}\|_2}{\mu_{\min}(\bW _{0:k}^\top \bA_k^{-1} \bW _{0:k})}\\
        &\le \frac{\mu_{M/2}(\bA_k)}{\mu_M(\bA_k)}.
    \end{align*}

    Now we focus on $T_2$, by definition of $T_2$ we have
    \begin{align*}
    T_2 &= \bu_{k:\infty}{}^\top \bH _{k:\infty} \bW _{k:\infty}^\top (\bW \bH \bW ^\top)^{-1/2} \exp(-2t\bW \bH \bW ^\top) (\bW \bH \bW ^\top)^{-1/2} \bW _{k:\infty} \bH _{k:\infty} \bu_{k:\infty} \\
    &\leq \bu_{k:\infty}{}^\top \bH _{k:\infty} \bW _{k:\infty}^\top (\bW \bH \bW ^\top)^{-1} \bW _{k:\infty} \bH _{k:\infty} \bu_{k:\infty} \\
    &\leq \|\bH _{k:\infty}^{1/2} \bW _{k:\infty}^\top (\bW \bH \bW ^\top)^{-1} \bW _{k:\infty} \bH _{k:\infty}^{1/2} \| \cdot \|\bu_{k:\infty}\|^2_{\bH _{k:\infty}} \\
    &\leq \|\bu_{k:\infty}\|^2_{\bH _{k:\infty}},
    \end{align*}
    
    where the last line follows from
    
    \begin{align*}
    \|\bH _{k:\infty}^{1/2} \bW _{k:\infty}^\top &(\bW \bH \bW ^\top)^{-1} \bW _{k:\infty} \bH _{k:\infty}^{1/2} \|_2 \\ 
    &= \|\bH _{k:\infty}^{1/2} \bW _{k:\infty}^\top (\bW _{0:k} \bH _{0:k} \bW _{0:k}^\top + \bW _{k:\infty} \bH _{k:\infty} \bW _{k:\infty}^\top)^{-1}
     \bW _{k:\infty} \bH _{k:\infty}^{1/2} \|_2 \\
    &\leq \|\bH _{k:\infty}^{1/2} \bW _{k:\infty}^\top \bA_k^{-1} \bW _{k:\infty} \bH _{k:\infty}^{1/2} \|_2= 1.
    \end{align*}
	The last line is because a nonzero projection matrix has norm $1$.
\end{proof}

\subsubsection{Proof of Proposition~\ref{prop:volterra-random}}
Combining Lemma \ref{lem:trace_analysis_appendix}, \ref{lem:bias_upper_bound}
and \ref{lem:bias_lower_bound}, we get
with probability at least $1 - e^{-\Omega(M)}$,
\begin{align}
%\label{eq:random_risk_recursion}
\EE[\cE_t] &\eqsim \bu_0^\top\bA_t^\top\bH\bA_t\bu_0 + \int_{0}^{t} \tr(\bS\bA_{t-\tau}^\top\bH\bA_{t-\tau}\bS)
\gamma_\tau (\EE[\cE_\tau]+ \sigma^2) \dd \tau, \\
&\eqsim  M^{-s\beta}+\widehat{e}_M(t) + \int_{0}^{t} \widehat{\cK}_M(t-\tau)
\gamma_\tau (\EE[\cE_\tau]+ \sigma^2) \dd \tau,
\end{align}
Here we arrive at the Volterra equation for random feature case.
Since we have $\hlambda_j \eqsim \lambda_j$, the proof in the top-$M$ case also holds for $\widehat{e}_M$ and $\widehat{\cK}_M$.

% This equation, \eqref{eq:random_risk_recursion}
% , is of the same form as in Theorem \ref{thm:volterra-topm}.
% From here following the same proof as in the top-$M$ feature case, we derive the
% functional scaling laws for random projection matrices.
% \begin{theorem}
%     When each entry of $\bW $ is i.i.d. sampled from the
%     normal distribution $\mathcal{N}(0, \frac{1}{M})$,
%     the population risk (without irreducible error) at the $r$-th step satisfies:
%     \[
%     \EE[\cE_{T(r)}] \eqsim g_M(s, T(r))
%     + \frac{1}{B} \int_{0}^{r} g_M(2-1/\beta, T(r)-T(\tau))
%     (\sigma^2 + g_M(s, T(\tau)))(\eta(\tau))^2 \dd \tau +
%     M^{-s\beta}
%     \]
%     with probability at least $1 - e^{-\Omega(M)}$.
% \end{theorem}

% Note that $\Tilde{\cR}_{T(r)} = \cE_{T(r)} + \frac{1}{2}\sigma^2$,
% thus we have fully proved our main theorem, Theorem \ref{thm:main_func},
% the functional scaling laws.

%\newpage
\section{Proofs for Section \ref{sec: effect-lrs}}
\label{appendix:compute}

When the condition $\sigma \gtrsim 1$ holds -- indicating a constant label-noise level -- the FSL simplifies to
\begin{equation}
\label{eq: simplified_fsl_sec_5}
    \EE[\cE(\bnu_t)] \eqsim   \frac{1}{M^{s\beta}} + \frac{1}{t^s} + \frac{\sigma^2}{B}\cN(\varphi), \quad \text{ with }\quad \cN(\varphi)=\int_0^t \cK_M(t-r)\varphi(T^{-1}(r))\dd r,
\end{equation}
where the fit-dependent noise term $e(r)$  is absorbed by the full-batch GD term due to $\sigma \gtrsim 1$.
Extending to the full range $\sigma \geq 0$  is possible but makes the statements and derivations much more involved. We therefore focus on the above case to streamline the exposition. 

\subsection{Proofs for Constant LRS}
\label{sub:Proofs for The Constant LRS}

In this section, we prove Theorem \ref{thm:chinchilla} and present the data-optimal scaling strategy, as well as some results related to the compute-optimal allocation.

\begin{theorem}
    [Restatement of Theorem \ref{thm:chinchilla}]
    \label{thm:const_lrs_law}
    Under Assumption~\ref{assump: task-difficulty},
    when the learning rate $\eta(k) \equiv \eta$, for the top-$M$ selection of the projection matrix
    $\bW $ or for the random case with probability at least
    $1 - e^{-\Omega(M)}$, we have
    \[
    \EE [\cR_K] - \frac{\sigma^2}{2} \eqsim \frac{1}{(\eta K)^s}
    + \frac{\eta}{B}\sigma^2 + M^{-s\beta}.
    \]
\end{theorem}
\begin{proof}[Proof]
    By our main Theorem \ref{thm: fsl-hard-regime}, when the learning rate $\eta(k) \equiv
    \eta$, denote $\gamma:= \frac{\eta}{B}$, by
    \eqref{eq: simplified_fsl_sec_5} we have
    %by Assumption \ref{ass:small_lr}, we can further simplify it as
    
    \[
    \EE[\cE_{K}] \eqsim \gamma\sigma^2 + \frac{1}{t^s}
    + M^{-s\beta}.
    \]

    Now we may write it as
    \[
    \EE [\cR_K] -\frac{\sigma^2}{2} \eqsim \gamma\sigma^2 +
    \frac{1}{t^s} + M^{-s\beta}.
    \]
    Notice that $t = \eta K$ and $\gamma = \frac{\eta}{B}$, we have
    \[
    \EE [\cR_K] - \frac{\sigma^2}{2} \eqsim \frac{1}{(\eta K)^s}
    + \frac{\eta}{B}\sigma^2 + M^{-s\beta}.
    \]
\end{proof}

\begin{theorem}
    Given a total data size of $D \gg 1$, the optimal strategy for
    minimizing the final population risk, in terms of the effective learning
    rate $\gamma$ and model size $M$ is:
\begin{equation}\label{eqn: data-optim-scaling-const-lrs-1}
    \gamma_{\opt} \eqsim D^{-\frac{s}{s+1}},\quad M_{\opt}\gtrsim D^{\frac{1}{(1+s)\beta}},\qquad \cE_{\opt} \eqsim D^{-\frac{s}{s+1}}.
\end{equation}
\end{theorem}
\begin{proof}[Proof]
    Since we have
    \[
    \EE[\cE_{K}] \eqsim \gamma\sigma^2 + \frac{1}{(\gamma D)^s}
    + M^{-s\beta},
    \]
    By weighted AM-GM inequality, we have that when $\cE_{K}$ is minimized, it must hold that
    \[
    \gamma \sigma^2 \eqsim \frac{1}{(\gamma D)^s}
    \]
    which gives
    \[
    \gamma_{\opt} \eqsim D^{-\frac{s}{s+1}}.
    \]
    Substituting this into the error expression yields
    \[
    \cE_{\opt} \eqsim D^{-\frac{s}{s+1}} + M^{-s\beta}.
    \]
    To balance the two terms and achieve the optimal rate, we require
    \[
    M_{\opt} \gtrsim D^{\frac{1}{(1+s)\beta}}.
    \]
    Consequently, the optimal loss rate becomes
    \[
    \cE_{\opt} \eqsim D^{-\frac{s}{s+1}}.
    \]
\end{proof}

Next we consider the compute optimal strategy for constant learning rates.
We define the compute $C = MKB$ to be the product of the model size,
training steps and batch size.
\begin{theorem}
    Given a total compute budget of $C \gg 1$, the optimal strategy for
    minimizing the final population risk, in terms of the effective learning
    rate $\gamma$, model size $M$, and data size $D := BK$, is:
    \begin{equation}
        \gamma_{\opt}\eqsim C^{-\frac{s\beta}{1+\beta + s\beta}}, M_{\opt}\eqsim C^{\frac{1}{1+\beta+s\beta}}, D_{\opt}\eqsim C^{\frac{\beta+s\beta}{1+\beta+s\beta}}, \nonumber
    \end{equation}
\end{theorem}
\begin{proof}[Proof]
    Since we have
    \[
    \EE[\cE_{K}] \eqsim \gamma\sigma^2 + \frac{1}{(\eta K)^s}
    + M^{-s\beta},
    \]
    substituting $K = \frac{C}{MB}$, we get
    \[
    \EE[\cE_{K}] \eqsim \gamma\sigma^2 + \frac{M^s}{(C\gamma)^s}
    + M^{-s\beta}.
    \]
    By weighted AM-GM inequality, we have that when $\cE_{K}$ is
    minimized, it must hold that
    \[
    \gamma\sigma^2 \eqsim \frac{M^s}{(C\gamma)^s},
    \quad \frac{M^s}{(C\gamma)^s} \eqsim M^{-s\beta},
    \]
    which gives
    \[
    \gamma_{\opt} \eqsim C^{-\frac{s\beta}{1+\beta+s\beta}},
    \quad M_{\opt} \eqsim C^{\frac{1}{1+\beta+s\beta}}.
    \]
    Now we can further compute $D = BK = CM^{-1} \eqsim
    C^{\frac{\beta+s\beta}{1+\beta+s\beta}}$.
\end{proof}

\subsection{Proof for The Exponential-Decay LRS}
\label{sub:proof_for_the_exponential_decay_lrs}

Recall that the LRS given by 
\[
    \varphi(\tau) = a e^{-\zeta \tau}, \text{ with } \varphi(K) = b,
\]
where $\zeta = \log(a/b)/K=:1/\bar{K}$. Note that the intrinsic-time transform is given by 
\[
    T(\tau) = \int_0^{\tau} \varphi(r)\dd r=\frac{a}{\zeta} \left(1-e^{-\zeta \tau}\right).
\]
Thus, we have
\begin{itemize}
    \item The total intrinsic time is: 
    \[
       T(K) = \frac{a}{\zeta}(1-e^{-\zeta K}) = \frac{K}{\log(a/b)}(a-b)=:\bar{K} (a-b).
    \]
  For simplicity, we shall write $T=T(K)$ in what follows.
    \item The LRS-adjusted function in intrinsic time is given by 
\end{itemize}
\[
    \gamma(t) = \varphi(T^{-1}(t)) =a-\zeta t.
\]

\begin{lemma}\label{lemma: N-exp-decay}
When $T \gtrsim 1$, the noise term satisfies 
\[
\cN(\varphi) := \int_{0}^T \cK_M(T-t)\gamma(t)\dd t = b I_1 + (a-b) I_2,
\]
where
\[
I_1 = \sum_{j=1}^M \frac{1-e^{-2\lambda_jT}}{2}\lambda_j, \qquad 
I_2 = \sum_{j=1}^M\left( -\frac{\lambda_j e^{-2\lambda_jT}}{2} +\frac{1-e^{-2\lambda_jT}}{4T}\right).
\]
\end{lemma}

\begin{proof}
We approximate the forgetting kernel by summation:
\[
\cK_M(t) = \sum_{j=1}^M \lambda_j^2 e^{-2\lambda_j t}.
\]
Noticing $b = a - \zeta T$ and $\zeta T = a - b$, we have
\begin{align*}
\int_0^{T}\cK_M(T-t)\gamma(t)\dd t 
&= \sum_{j=1}^M \lambda_j^2\int_0^{T} e^{-2\lambda_j(T-t)}(a-\zeta t)\dd t \\
&= \sum_{j=1}^M \lambda_j^2 e^{-2\lambda_jT}
\left[a\!\int_0^T e^{2\lambda_jt}\dd t - \zeta\!\int_0^T t e^{2\lambda_jt}\dd t\right]\\
&= \sum_{j=1}^M \left[\frac{a(1-e^{-2\lambda_jT})}{2}\lambda_j 
- \frac{\zeta(1-(1+2\lambda_jT)e^{-2\lambda_jT})}{4}\right]\\
&=(a-\zeta T)\sum_{j=1}^M \frac{1-e^{-2\lambda_jT}}{2}\lambda_j
+ \zeta T\sum_{j=1}^M\left( -\frac{\lambda_je^{-2\lambda_jT}}{2} + \frac{1-e^{-2\lambda_jT}}{4T}\right).
\end{align*}
Thus, we complete the proof.
\end{proof}

We next bound $I_1$ and $I_2$ separately. 
\begin{lemma}
If $T$ and $M$ are sufficiently large, then 
\[
I_1 = C_\beta + o_{T,M}(1),
\]
where $C_\beta$ is a constant depending only on $\beta$.
\end{lemma}

\begin{proof}
Recall that
\[
I_1 = \sum_{j=1}^M \frac{1-e^{-2\lambda_jT}}{2}\lambda_j
= \frac{1}{2}\sum_{j=1}^M \lambda_j - \frac{1}{2}\sum_{j=1}^M \lambda_j e^{-2\lambda_j T}
=: A - B.
\]

For the first term,
\[
A = \frac{1}{2}\sum_{j=1}^M \lambda_j
= C_\beta - \frac{1}{2}\sum_{j=M+1}^{\infty} \lambda_j,
\qquad C_\beta := \frac{1}{2}\sum_{j=1}^{\infty} \lambda_j.
\]
By the integral bound for the tail,
\[
\sum_{j=M+1}^{\infty} \lambda_j
\le c\sum_{j=M+1}^{\infty} j^{-\beta}
\le c\int_{M}^{\infty} x^{-\beta}\,dx
= \frac{cM^{-(\beta-1)}}{\beta-1},
\]
so \(A = C_\beta + O(M^{-(\beta-1)}) = C_\beta + o_M(1)\).

For the second term, following the same proof in Lemma~\ref{lemma:ke-power}, we have
\[
B=\frac{1}{2}\sum_{j=1}^M \lambda_j e^{-2\lambda_jT}
\lesssim T^{-(1-1/\beta)}.
\]
Combining the two estimates gives
\[
I_1 = C_\beta + o_M(1) + O(T^{-(1-1/\beta)}) = C_\beta + o_{T,M}(1).
\]
\end{proof}

% \begin{lemma}
% If $T$ and $M$ is sufficiently large, then $$I_2 \eqsim \frac{\beta \min(M, T^{1/\beta})}{4T}.$$
% \end{lemma}
% \begin{proof}
% Let $r = uT$. Then, by a change of variable, we obtain 
% \[
%     I_2 = \frac{1}{4T^{1-1/\beta}}\int_{\frac{T}{M^\beta}}^T \frac{1-e^{-2r}-2re^{-2r}}{r^{1+1/\beta}}\dd r =: \frac{1}{4T^{1-1/\beta}}\int_{\frac{T}{M^\beta}}^T q_\beta(r)\dd r.
% \]
% It is easy to verify that for any $\beta\geq 1$,  $\inf_{r\geq 0}q_\beta(r)\geq 0$ and $q_\beta(r)\eqsim r^{-1-1/\beta}$ when $r$ is sufficiently large. We refer to Figure~\ref{fig: qbeta} for an illustration of $q_\beta(\cdot)$.

% \begin{figure}[!ht]
%     \centering 
%     \includegraphics[width=0.4\textwidth]{figs/qbeta.pdf}
%     %\vspace*{-1em}
%     \caption{Illustration of the function $q_\beta(\cdot)$.}
%     \label{fig: qbeta}
% \end{figure}

% \begin{itemize}
% \item When $T/M^\beta\leq 1$ and $T$ is sufficiently large such that  $\int_{\frac{T}{M^\beta}}^T q_\beta(r)\dd r\eqsim \beta$ and thus
%  we have 
% \[
%     I_2 = \frac{\beta+ o_{M,T}(1)}{4T^{1-1/\beta}} .
% \]
% \item When $T/M^{\beta}>1$, it holds for all $r\geq 1$ that $0.5\leq 1-e^{-2r}-2re^{-2r}\leq 1$ . Thus, there exists a $C_{T,M}\in [0.5, 1]$ such that 
% \begin{align*}
%  I_2 &= C_{T,M} \frac{1}{4T^{1-1/\beta}} \int_{\frac{T}{M^\beta}}^Tr^{-1-1/\beta}\dd r \\ 
%  &= \frac{C_{T,M}\beta}{4T^{1-1/\beta}}\left(\left(\frac{T}{M^\beta}\right)^{-1/\beta}-T^{-1/\beta}\right)=\frac{C_{T,M}\beta(M-1)}{4T}.
% \end{align*}
% \end{itemize}
% Combining the two cases, we complete the proof.
% \end{proof}
\begin{lemma}
If $T$ and $M$ are sufficiently large, then 
\[
I_2 \eqsim \frac{\min(M,\,T^{1/\beta})}{T}.
\]
\end{lemma}

\begin{proof}
Recall that
\[
I_2 = \sum_{j=1}^M\!\left(-\frac{\lambda_j e^{-2\lambda_j T}}{2} + \frac{1-e^{-2\lambda_j T}}{4T}\right)
    = \frac{1}{4T}\sum_{j=1}^M\!\bigl(1 - e^{-2\lambda_j T} - 2\lambda_j T e^{-2\lambda_j T}\bigr),
\]
where $\lambda_j \eqsim j^{-\beta}$. Let $c_1, c_2$ be constants such that $c_1j^{-\beta} \le\lambda_j \le c_2j^{-\beta}$.

\paragraph{Case 1. $T \ge M^\beta$.}
In this regime, for every $j\le M$ we have $\lambda_j T \ge c_1 M^{-\beta}T \ge c_1$.  
Since $h(r) = 1 - e^{-2r} - 2re^{-2r}$ is bounded on $[c_1,\infty)$, 
there exist constants $c_3 < c_4$ such that 
$h(r) \in [c_3, c_4]$ for all $r \ge c_1$. Thus
\[
\frac{c_3}{4} \frac{M}{T} \le I_2 \le \frac{c_4}{4} \frac{M}{T}.
\]

\paragraph{Case 2. $T < M^\beta$.} 
Since $g(r) = 1 - e^{-2r} - 2re^{-2r} \ge 2r^2 e^{-2r}$ for $r>0$, we have 
\[
I_2 \ge \frac{1}{4T} \sum_{j=1}^M 2 \lambda_j^2T^2 e^{-2\lambda_j T} = \frac{T}{2}\cK_M(T) \eqsim \frac{1}{T^{1-\frac{1}{\beta}}},
\]
where the last equality follows from Lemma~\ref{lemma:ke-power} that $\cK_M(t) \eqsim t^{-(2-1/\beta)}$ when $1\lesssim t\lesssim M^\beta$.

On the other hand, since the term $1 - e^{-2\lambda_j T} - 2\lambda_j T e^{-2\lambda_j T}$ is increasing in $\lambda_j$ and $\lambda_j\le c_2 j^{-\beta}$, we have
\begin{align*}
I_2 &\le \frac{1}{4T}\sum_{j=1}^M \left(1 - e^{-2c_2j^{-\beta}T} - 2c_2j^{-\beta} e^{-2c_2j^{-\beta}T}\right)\\
&\le \frac{1}{4T} \int_0^{\infty}\left(
1-e^{-2c_2x^{-\beta}T}-2c_2x^{-\beta}Te^{-2c_2x^{-\beta}T}
\right) \dd x\\
{\tiny (u=2c_2Tx^{-\beta})\quad}&= \frac{(2c_2)^{\frac{1}{\beta}}}{4T^{1-\frac{1}{\beta}}} \int_0^\infty \frac{1-e^{-u}-ue^{-u}}{u^{1+\frac{1}{\beta}}}\dd u\\
&\lesssim \frac{1}{T^{1-\frac{1}{\beta}}}.
\end{align*}
where the last inequality follows from the integral
\[
\int_0^\infty \frac{1-e^{-u}-ue^{-u}}{u^{1+\frac{1}{\beta}}}\dd u < \infty,
\]
which is because the function $q(u) := \frac{1-e^{-u}-ue^{-u}}{u^{1+\frac{1}{\beta}}}$ satisfies $q(u) \sim u^{1-\frac{1}{\beta}}$ as $u\to 0$, and $q(u) \sim u^{-(1+\frac{1}{\beta})}$ as $u\to \infty$.

Combining the two cases, we obtain that
\[
I_2 \eqsim \frac{\min(M,\,T^{1/\beta})}{T}.
\]
\end{proof}

\begin{theorem}
    [Theorem \ref{thm:geometric} in the main paper]
    We consider the exponentially decaying learning rate schedule
	\[
		\varphi(\tau) = a e^{-\zeta \tau}, \text{ with } \varphi(K) = b,
	\]
    Under this learning rate schedule, for the top-$M$ projection matrix
    or the random projection with probability at least $1 - e^{-\Omega(M)}$,
    we have
    \begin{equation}
        \cE_K \eqsim M^{-s\beta} + T^{-s} + \frac{\sigma^2}{B}
		\left(b+(a-b)\frac{\min\{M, T^{1/\beta}\}}{T}\right), \nonumber
    \end{equation}
    where $T=(a-b)K/\log(a/b)$ is the total intrinsic training time.
\end{theorem}
\begin{proof}[Proof]
    By the functional scaling laws~\eqref{eq: simplified_fsl_sec_5},
	\[
	\cE_K \eqsim M^{-s\beta} + T^{-s} + \frac{\sigma^2}{B}\cN(\varphi).
	\]
	The noise term $\cN(\varphi)$ is estimated by Lemma~\ref{lemma:
	N-exp-decay} and the bound on $I_1$, $I_2$ as
	\[
	\cN(\varphi) = bI_1 + (a - b)I_2
	\eqsim b + (a - b) \frac{\min(M, T^{1/\beta})}{T},
	\]
	which gives
	\[
	\cE_K \eqsim M^{-s\beta} + T^{-s} + \frac{\sigma^2}{B}
	\left( b + (a - b) \frac{\min(M, T^{1/\beta})}{T} \right),
	\]
	so we complete the proof.
\end{proof}

\begin{theorem}
    Given a total data size $D \gg 1$, the optimal strategy for minimizing
	the final population risk when $b = \frac{a}{K}$ is given
	by $M_{\opt} = \infty$ and
	\begin{itemize}
		\item If $s> 1-\frac{1}{\beta}$, then $\gamma_\opt \eqsim (D/\log D)^{-\frac{1+s\beta-\beta}{1+s\beta}}$ and $\cE_{\opt}\eqsim (D/\log D)^{-\frac{s\beta}{s\beta+1}}$.
		\item If $s\leq 1-\frac{1}{\beta}$, then $\gamma_\opt\eqsim 1$ and $\cE_\opt \eqsim (D/\log D)^{-s}$.
	\end{itemize}
\end{theorem}
\begin{proof}[Proof]
    Denote $\tilde{D} := \frac{D}{\log K}$, then by Theorem~\ref{thm:geometric},
	\[
	\mathcal{E}_K \eqsim M^{-s\beta} + (\gamma \tilde{D})^{-s} +
	\frac{\min(M, (\gamma \tilde{D})^{\frac{1}{\beta}})}{\tilde{D}}.
	\]
	\paragraph{Case 1.} When $M^\beta \le \gamma \tilde{D}$,
	\[
	\mathcal{E}_K \eqsim M^{-s\beta} + (\gamma \tilde{D})^{-s}
	+ \frac{M}{\tilde{D}}.
	\]
	We see that in this case $\gamma$ should be as large as possible,
	since $a \lesssim 1$, we set $\gamma\eqsim 1$ accordingly.

	In this case $M^{-s\beta} + \frac{M}{\tilde{D}} \gtrsim
	\tilde{D}^{-\frac{s\beta}{1+s\beta}}$, with equality at $M \eqsim
	\tilde{D}^{\frac{1}{1+s\beta}}$.

	When $s > 1-\frac{1}{\beta}$, the above equality condition can be acheived
	as $M^\beta = \tilde{D}^{\frac{\beta}{1+s\beta}} < \tilde{D}$. Hence we have
	that
	\[
	M_{\opt} \eqsim \tilde{D}^{\frac{1}{1+s\beta}},\quad
	\gamma_{\opt} \eqsim 1, \quad \mathcal{E}_{\opt} \eqsim
	\tilde{D}^{-\frac{s\beta}{1+s\beta}}.
	\]
	Note that $\gamma = \frac{a}{B} \eqsim 1$ and $a\lesssim 1$,
	which forces $B \eqsim 1$, hence $\tilde{D}\eqsim \frac{D}{\log D}$.

	When $s \le 1-\frac{1}{\beta}$, the quantity $M^{-s\beta} +
	\frac{M}{\tilde{D}}$ is decreasing with respect to $M$,
	hence the optimal $M$ in this case is $M = (\gamma
	\tilde{D})^{\frac{1}{\beta}}$, which transfers to case 2.

	\paragraph{Case 2.} When $M^\beta > \gamma \tilde{D}$,
	\[
	\mathcal{E}_K \eqsim M^{-s\beta} + (\gamma \tilde{D})^{-s}
	+ \gamma^{\frac{1}{\beta}} \frac{1}{\tilde{D}^{1-\frac{1}{\beta}}}.
	\]
	Clearly in this case $M_{\opt} = \infty$, and by AM-GM inequality,
	\[
		(\gamma \tilde{D})^{-s} + \gamma^{\frac{1}{\beta}}
		\frac{1}{\tilde{D}^{1-\frac{1}{\beta}}} \gtrsim
		\tilde{D}^{-\frac{s\beta}{1+s\beta}},
	\]
	with equality at $\gamma \eqsim \tilde{D}^{\frac{\beta - 1 -
	s\beta}{1+s\beta}}$.

	When $s > 1-\frac{1}{\beta}$, the equality can be achieved, hence we have
	\[
	M_{\opt} = \infty, \quad \gamma_{\opt} \eqsim
	\tilde{D}^{-\frac{1+s\beta-\beta}{1+s\beta}}, \quad \mathcal{E}_{\opt}
	\eqsim \tilde{D}^{-\frac{s\beta}{1+s\beta}}.
	\]

	When $s\le 1-\frac{1}{\beta}$, since $\gamma \lesssim 1$, we must have
	\[
	M_{\opt} = \infty, \quad \gamma_{\opt} \eqsim 1,\quad
	\mathcal{E}_{\opt} \eqsim \tilde{D}^{-s}.
	\]
	Similarly, as $a\lesssim 1$, we have $B\lesssim \tilde{D}^{1 -
	\frac{\beta}{1+s\beta}}$, which means $K \gtrsim
	\tilde{D}^{\frac{\beta}{1+s\beta}}$, hence $\log K \eqsim \log D$,
	$\tilde{D}\eqsim \frac{D}{\log D}$.

	\paragraph{Summary.} Combining the two cases together, we see that $M_{\opt}
	= \infty$ can always achieves the optimal rate, hence the conclusion
	follows.
\end{proof}

\begin{theorem}
    Given a large total compute budget \( C \gg 1 \), the optimal strategy for
	minimizing the final population risk -- expressed in terms of the effective
	maximum learning rate \( \gamma \), model size \( M \), and data size \( D \) -- is
	given by:
	\begin{itemize}
	\item When \( s > 1 - \frac{1}{\beta} \), the optimal scaling laws are:
    \begin{equation}
        \gamma_{\opt} \eqsim (C/\log C)^{-\frac{1+\beta(s-1)}{2+s\beta}}, \quad
        M_{\opt} \eqsim (C/\log C)^{\frac{1}{2+s\beta}}, \quad
        D_{\opt} \eqsim C^{\frac{1+s\beta}{2+s\beta}}(\log C)^{\frac{1}{2+s\beta}}, \nonumber
    \end{equation}
    which leads to the following optimal final population risk:
    \begin{equation}
        \cE_{\opt}(C) \eqsim (C/\log C)^{-\frac{s\beta}{2+s\beta}}. \nonumber
    \end{equation}
    
	\item When \( s \le 1 - \frac{1}{\beta} \), the optimal scaling laws are
    \begin{equation}
        \gamma_{\opt} \eqsim 1, \quad
        M_{\opt} \eqsim (C/\log C)^{\frac{1}{1+\beta}}, \quad
        D_{\opt} \eqsim C^{\frac{\beta}{1+\beta}}(\log C)^{\frac{1}{1+\beta}}, \nonumber
    \end{equation}
    which leads to the following optimal final population risk:
    \begin{equation}
        \cE_{\opt}(C) \eqsim (C/\log C)^{-\frac{s\beta}{1+\beta}}. \nonumber
    \end{equation}
	\end{itemize}
\end{theorem}
\begin{proof}
    Denote $\tilde{D} = D/\log K$. For similar reasons as in the derivation of data-optimal scaling, we may assume $\log K\eqsim \log C$ to simplify the proof. At this point, the loss can be reformulated as follows.
    \[
    \cE_K \eqsim M^{-s\beta} +
    \frac{1}{(\gamma \tilde D)^s} + \sigma^2
    \frac{\min\{M, (\gamma \tilde D)^{1/\beta}\}}{\tilde D}.
    \]
    \paragraph{Case 1.}
    $M^\beta<\gamma \tilde D$ and we have
    \[
    \cE_K \eqsim M^{-s\beta} +
    \frac{1}{(\gamma \tilde D)^s} + \sigma^2
    \frac{M}{\tilde D}
    \]
    As $\gamma$ only appears in the second term, and $\frac{1}{(\gamma \tilde D)^s}$ is
	monotone decreasing with $\gamma$, we have that when $\cE_K$ is minimized,
	it must hold that
    \[
    M = (\gamma \tilde D)^{1/\beta}.
    \]
    When $s>1-\frac{1}{\beta}$, we
    then consider a weighted AM-GM inequality, we have
    \[
    M^{-s\beta} = \sigma^2 \frac{M}{\tilde D}.
    \]
    Combining with $C= MD$ and $M=(\gamma \tilde D)^{1/\beta}$, we have
    \begin{equation}
        \gamma_{\opt} \eqsim (C/\log C)^{-\frac{1+\beta(s-1)}{2+s\beta}}, \quad
        M_{\opt} \eqsim (C/\log C)^{\frac{1}{2+s\beta}}, \quad
        D_{\opt} \eqsim C^{\frac{1+s\beta}{2+s\beta}}(\log C)^{\frac{1}{2+s\beta}}, \nonumber
    \end{equation}
    and 
    \begin{equation}
        \cE_{\opt}(C) \eqsim (C/\log C)^{-\frac{s\beta}{2+s\beta}}. \nonumber
    \end{equation}
    When \( s\le 1-\frac{1}{\beta} \), since \( a \lesssim 1 \),
	we set \( \gamma_{\opt} \eqsim 1 \) accordingly, and proceed as follows:
    \begin{equation}
        M_{\opt} \eqsim (C/\log C)^{\frac{1}{1+\beta}}, \quad
        D_{\opt} \eqsim C^{\frac{\beta}{1+\beta}}(\log C)^{\frac{1}{1+\beta}}, \nonumber
    \end{equation}
    and 
    \begin{equation}
        \cE_{\opt}(C) \eqsim (C/\log C)^{-\frac{s\beta}{1+\beta}}. \nonumber
    \end{equation}
    
    \paragraph{Case 2.}$M^\beta\ge\gamma \tilde D$ and we have
    \[
    \cE_K \eqsim M^{-s\beta} +
    \frac{1}{(\gamma \tilde D)^s} + \sigma^2
    \frac{(\gamma \tilde D)^{1/\beta}}{\tilde D}
    \]
    As $M$ only appears in the second term, and $M^{-s\beta}$ is monotonically decreasing in $M$, we have that when $\cE_K$ is minimized, it must hold that
    \[
    M = (\gamma \tilde D)^{1/\beta}.
    \]
    And then the rest is identical to the first case.
\end{proof}

\subsection{Proof for the WSD-Like LRS}
\label{sub:proof_for_the_wsd_like_lrs}

To prove Theorem~\ref{thm: wsd}, we first present the following lemma, which gives an upper bound for the SGD noise induced by the stable phase.
\begin{lemma}
\label{lem: wsd-stable-noise}
	For $T_2 > 0$, we have
	\[
	\int_{0}^{\infty} \mathcal{K}_M(T_2 + t)\dd t \lesssim \frac{\min\{M,
	T_2^{\frac{1}{\beta}}\}}{T_2}.
	\]
\end{lemma}
\begin{proof}[Proof]
    We have
	\begin{align*}
	\int_{0}^{\infty} \mathcal{K}_M(T_2 + t)\dd t
	&= \int_{0}^{\infty} \sum_{j=1}^M \lambda_j^2 e^{-2\lambda_j (T_2+t)} \dd t\\
	&= \sum_{j=1}^M \frac{\lambda_j}{2}e^{-2\lambda_j T_2}\\
	&\lesssim \frac{1}{T_2^{1-\frac{1}{\beta}}},
	\end{align*}
	where the last line can be shown using the same approach as in the proof of Lemma~\ref{lemma:ke-power}.
    
	On the other hand, we have
	\begin{align*}
	    \int_{0}^{\infty} \mathcal{K}_M(T_2 + t)\dd t
	&= \int_{0}^{\infty} \sum_{j=1}^M \lambda_j^2 e^{-2\lambda_j (T_2+t)} \dd t\\
	&= \sum_{j=1}^M \frac{\lambda_j}{2}e^{-2\lambda_j T_2}\\
	&= \frac{1}{2T_2} \sum_{j=1}^{M} \lambda_j T_2 e^{-2\lambda_j T_2} \lesssim \frac{M}{T_2} \sup_{x>0}xe^{-2x} \lesssim \frac{M}{T_2},
	\end{align*}
	where the last line is because $\sup_{x>0}xe^{-2x} = \frac{1}{2e}$.
\end{proof}
\begin{theorem}[Theorem~\ref{thm: wsd} in the main paper]
\label{thm:wsd_decay_app}
    Suppose the FSL~\eqref{eqn: fsl} hold and $M,K$ are sufficiently large. Then, we have
    \begin{equation}
        \cE_K \eqsim M^{-s\beta} + T^{-s} + \sigma^2 \left(\frac{b}{B}+(a-b)\frac{\min\{M, T_2^{1/\beta}\}}{B T_2}\right), \nonumber
    \end{equation}
    where $T=a K_1 + (a-b)K_2/\log(a/b)$ is the total intrinsic training time, $T_2 = (a-b)K_2/\log(a/b)$ is the decay-phase intrinsic training time, and we require that $T\gtrsim 1$, $T_2\gtrsim 1$.
\end{theorem}
\begin{proof}
By the results of the exponential decay LRS, let $\lambda = \log(a / b) / K_2$,
we have
\[
    \int_{0}^{T(K)} \mathcal{K}_M(T(K) - t)\gamma(t) \dd t
	= \int_{0}^{T_1} \mathcal{K}_M(T(K) - t)a \dd t
	+ \int_{0}^{T_2} \mathcal{K}_M(T_2 - t)(a - \zeta t)
	\dd t,
\]
Hence by the estimation of the noise term of the exponential decay LRS (see the proof of Theorem~\ref{thm:geometric}), we have
\begin{equation}
    \int_{0}^{T_2} \mathcal{K}_M(T_2 - t)(a - \zeta t)
	\dd t \eqsim b + \frac{(a-b)\min\{M, T_2^{\frac{1}{\beta}}\}}{T_2}. \nonumber
\end{equation}
Thus, we know
\begin{align*}
    \int_{0}^{T(K)} \mathcal{K}_M(T(K) - t)\gamma(t) \dd t
	&\eqsim \int_{0}^{T_1} \mathcal{K}_M(T(K) - t)a \dd t
	+ b + \frac{(a-b)\min\{M, T_2^{\frac{1}{\beta}}\}}{T_2}\\
	&\eqsim a\int_{0}^{T_1} \mathcal{K}_M(T_2+t) \dd t
	+ b + \frac{(a-b)\min\{M, T_2^{\frac{1}{\beta}}\}}{T_2}\\
	&\eqsim b + \frac{(a-b)\min\{M, T_2^{\frac{1}{\beta}}\}}{T_2}. \qquad \text{(by using Lemma~\ref{lem: wsd-stable-noise})}
\end{align*}
Hence the loss is given by
\[
\mathcal{E}_{K} \eqsim \frac{1}{T^s}
+ M^{-s\beta} + \frac{\sigma^2}{B} \left(b + (a-b)\frac{\min\{M, T_2^{\frac{1}{\beta}}\}}{T_2}\right).
\]
\end{proof}

\begin{theorem}
    Assume $b = \frac{a}{K}$, then we have the following data-optimal strategy:
    \begin{itemize}
        \item If $s\geq 1-1/\beta$, we have $\gamma_\opt \eqsim
			D^{-\frac{1+s\beta-\beta}{1+s\beta}} (\log
			D)^{-\frac{\beta-1}{1+s\beta}}$, $(D_1)_\opt,(D_2)_\opt \eqsim D$
			and $\cE_{\opt}\eqsim D^{-\frac{s\beta}{s\beta+1}} (\log
			D)^{\frac{s\beta-s}{1+s\beta}}$.
        \item If $s<1-1/\beta$, we have $\gamma_\opt\eqsim 1$, $(D_1)_\opt\eqsim
			D,(D_2)_\opt \gtrsim D^{\frac{s\beta}{\beta-1}}\log D$
			and $\cE_\opt \eqsim D^{-s}$.
    \end{itemize}
\end{theorem}

\begin{proof}
	Since the total intrinsic time $T\lesssim \gamma
	D$, we can always take $D_1\eqsim D$ to ensure $T\eqsim \gamma D$.
	Denote $\tilde{D}_2 := \frac{D_2}{\log K}$, then by Theorem~\ref{thm:wsd_decay_app},
	\[
	\mathcal{E}_K \eqsim M^{-s\beta} + (\gamma D)^{-s} +
	\frac{\min(M, (\gamma \tilde{D}_2)^{\frac{1}{\beta}})}{\tilde{D}_2}.
	\]
	\paragraph{Case 1.} When $M^\beta \le \gamma \tilde{D}_2$,
	\[
	\mathcal{E}_K \eqsim M^{-s\beta} + (\gamma D)^{-s}
	+ \frac{M}{\tilde{D}_2}.
	\]
	We see that in this case $\gamma$ should be as large as possible,
	since $a \lesssim 1$, we set $\gamma\eqsim 1$ accordingly.

	In this case $M^{-s\beta} + \frac{M}{\tilde{D}_2} \gtrsim
	\tilde{D}_2^{-\frac{s\beta}{1+s\beta}}$, with equality at $M \eqsim
	\tilde{D}_2^{\frac{1}{1+s\beta}}$.

	When $s > 1-\frac{1}{\beta}$, the above equality condition can be achieved
	as $M^\beta = \tilde{D}_2^{\frac{\beta}{1+s\beta}} < \tilde{D}_2$. Hence we have
	that
	\[
	M_{\opt} \eqsim \tilde{D}_2^{\frac{1}{1+s\beta}},\quad
	\gamma_{\opt} \eqsim 1, \quad \mathcal{E}_{\opt} \eqsim
	\tilde{D}_2^{-\frac{s\beta}{1+s\beta}}.
	\]
	Therefore $(D_2)_{\opt}\eqsim D$.
	Note that $\gamma = \frac{a}{B} \eqsim 1$ and $a\lesssim 1$,
	which forces $B \eqsim 1$, hence $\tilde{D}_2\eqsim \frac{D}{\log D}$.

	When $s \le 1-\frac{1}{\beta}$, the quantity $M^{-s\beta} +
	\frac{M}{\tilde{D}_2}$ is decreasing with respect to $M$,
	hence the optimal $M$ in this case is $M = (\gamma
	\tilde{D}_2)^{\frac{1}{\beta}}$, which transfers to case 2.

	\paragraph{Case 2.} When $M^\beta > \gamma \tilde{D}_2$,
	\[
	\mathcal{E}_K \eqsim M^{-s\beta} + (\gamma D)^{-s}
	+ \gamma^{\frac{1}{\beta}} \frac{1}{\tilde{D}_2^{1-\frac{1}{\beta}}}.
	\]
	Clearly in this case $M_{\opt} = \infty$, and by AM-GM inequality,
	\[
		(\gamma D)^{-s} + \gamma^{\frac{1}{\beta}}
		\frac{1}{\tilde{D}_2^{1-\frac{1}{\beta}}} \gtrsim
		D^{-\frac{s}{1+s\beta}}\tilde{D}_2^{-\frac{s\beta-s}{1+s\beta}},
	\]
	with equality at $\gamma \eqsim D^{-\frac{s\beta}{1+s\beta}}\tilde{D}_2^{\frac{\beta - 1}{1+s\beta}}$.

	When $s > 1-\frac{1}{\beta}$, the equality can be achieved, hence we have
	that $(D_2)_{\opt}\eqsim D$, so $\tilde{D}_2\eqsim \frac{D}{\log K}$,
	\[
	M_{\opt} = \infty, \quad \gamma_{\opt} \eqsim
	D^{-\frac{1+s\beta-\beta}{1+s\beta}} (\log K)^{-\frac{\beta-1}{1+s\beta}},
	\quad \mathcal{E}_{\opt}
	\eqsim D^{-\frac{s\beta}{1+s\beta}} (\log K)^{\frac{s\beta-s}{1+s\beta}}.
	\]

	When $s\le 1-\frac{1}{\beta}$, since $\gamma \lesssim 1$, we must have
	either $\gamma\eqsim 1$ or $\gamma \eqsim D^{-\frac{s\beta}{1+s\beta}}
	\tilde{D}_2^{\frac{\beta-1}{1+s\beta}}\lesssim 1$.
	To reach the minimum risk,
	in both cases we require $(\tilde{D}_2)_{\opt} \gtrsim
	D^{\frac{s\beta}{\beta-1}}$ (this gives $(D_2)_{\opt} \gtrsim
	D^{\frac{s\beta}{\beta-1}}\log D$), and
	\[
	M_{\opt} = \infty, \quad \gamma_{\opt}\eqsim 1,
	\quad \mathcal{E}_{\opt} \eqsim D^{-s}.
	\]
	Similarly, as $a\lesssim 1$, we have $B\lesssim_{\log} D^{1 -
	\frac{\beta}{1+s\beta}}$, which means $K \gtrsim_{\log}
	D^{\frac{\beta}{1+s\beta}}$, hence $\log K \eqsim \log D$,
	which gives the desired rate.

	\paragraph{Summary.} Combining the two cases together, we see that $M_{\opt}
	= \infty$ (case 2) always achieves the optimal rate, hence the conclusion
	follows.
\end{proof}

\begin{theorem}
    Assume $b  =\frac{a}{K}$, under the compute constraint $C\gg 1$, the optimal strategy for
	minimizing the final population risk—expressed in terms of the effective
	maximum learning rate \( \gamma \), model size \( M \), and data size \( D \)—is
	given by:
    \begin{itemize}
        \item 
        When $s> 1-1/\beta$, the optimal scaling laws are:
        \begin{equation}
            \gamma_{\opt} \eqsim (C/\log C)^{-\frac{1+s\beta-\beta}{2+s\beta}}, 
            M_{\opt} \eqsim (C/\log C)^{\frac{1}{2+s\beta}},
            D_{\opt} \eqsim C^{\frac{1+s\beta}{2+s\beta}} (\log C)^{\frac{1}{2+s\beta}}, \nonumber
        \end{equation}
        \begin{equation}  
            (D_1)_{\opt} \eqsim D, (D_2)_{\opt} \eqsim D, \nonumber
        \end{equation}
        which leads to the following optimal final population risk:
        \[
            \cE_{\opt}\eqsim C^{-\frac{s\beta}{2+s\beta}}(\log C)^{\frac{s\beta-s}{2+s\beta}}. \nonumber
        \]
        \item When $s\leq 1-1/\beta$, the optimal scaling laws are:
        \[
            \gamma_{\opt} \eqsim 1,
            M_{\opt} \eqsim C^{\frac{1}{1+\beta}},
            D_{\opt} \eqsim C^{\frac{\beta}{1+\beta}},(D_1)_{\opt} \eqsim D, (D_2)_{\opt} \gtrsim D^{\frac{s\beta}{\beta-1}}\log D, \nonumber
        \]
        which leads to the following optimal final population risk:
        \[
            \cE_{\opt} \eqsim C^{-\frac{s\beta}{1+\beta}}. \nonumber
        \]
    \end{itemize}
\end{theorem}

\begin{proof}
    Since the total intrinsic time $T\lesssim \gamma D$, we can always take $D_1\eqsim D$ to ensure $T\eqsim\gamma D$. Denote $\tilde{D}_2 := \frac{D_2}{\log K}$, the loss can be reformulated as follows.
    \[
    \cE_K \eqsim M^{-s\beta} +
    \frac{1}{(\gamma D)^s} + \sigma^2
    \frac{\min\{M, (\gamma \tilde{D}_2)^{1/\beta}\}}{\tilde{D}_2}.
    \]
    \paragraph{Case 1.}
    $M^\beta < \gamma \tilde{D}_2$ and we have
    \[
    \cE_K \eqsim M^{-s\beta} +
    \frac{1}{(\gamma D)^s} + 
    \frac{M}{\tilde{D}_2}.
    \]
    As $\gamma$ only appears in the second term, and $\frac{1}{(\gamma D)^s}$ is
	monotone decreasing with $\gamma$, we have that when $\cE_K$ is minimized,
	it must hold that
    \[
    M = (\gamma \tilde{D}_2)^{1/\beta}.
    \]
    When $s>1-\frac{1}{\beta}$, we
    then consider a weighted AM-GM inequality, we have
    \[
    M^{-s\beta} =  \frac{M}{\tilde{D}_2}.
    \]
    Combining with $M=(\gamma \tilde{D}_2)^{1/\beta}$, we have
    \begin{equation}
        \gamma_{\opt} \eqsim \tilde{D}_2^{-\frac{1+\beta(s-1)}{1+s\beta}}, \quad
        M_{\opt} \eqsim \tilde{D}_2^{\frac{1}{1+s\beta}}\nonumber
    \end{equation}
    and 
    \begin{equation}
        \cE_{\opt}\eqsim \tilde{D}_2^{s-\frac{s\beta}{1+s\beta}}D^{-s}. \nonumber
    \end{equation}
    Notice that
    \[
    C \eqsim \tilde{D}_2^{\frac{1}{1+s\beta}} D \ge \tilde{D}^{\frac{2+s\beta}{1+s\beta}}
    \Longrightarrow \cE \gtrsim C^{-\frac{s\beta}{2+s\beta}}.
    \]
    Note that this implies $D^{\frac{2+s\beta}{1+s\beta}} \gtrsim C\gtrsim D \implies \log D\eqsim \log C$, and by similar reasons $\log K\eqsim \log D$ (the max learning rate $B\gamma \lesssim 1$).
    
    Hence when $\cE$ is optimized, we have $\tilde{D}_2 \eqsim D/\log C$ and
    \begin{equation}
        \gamma_{\opt} \eqsim (C/\log C)^{-\frac{1+\beta(s-1)}{2+s\beta}}, \quad
        M_{\opt} \eqsim (C/\log C)^{\frac{1}{2+s\beta}}, \quad
        D_{\opt} \eqsim C^{\frac{1+s\beta}{2+s\beta}} (\log C)^{\frac{1}{2+s\beta}}, \nonumber
    \end{equation}
    and 
    \begin{equation}
        \cE_{\opt}(C) \eqsim (C/\log C)^{-\frac{s\beta}{2+s\beta}}(\log C)^s. \nonumber
    \end{equation}
    
    When \( s\le 1-\frac{1}{\beta} \), since \( a \lesssim 1 \),
	we set \( \gamma_{\opt} \eqsim 1 \) accordingly, and proceed as follows:
    \begin{equation}
        M_{\opt} \eqsim \tilde{D}_2^{\frac{1}{\beta}}\nonumber
    \end{equation}
    and
    \[
    \cE_{\opt} \eqsim D^{-s}.
    \]
    Notice that
    \[
    C\eqsim \tilde{D}_2^{\frac{1}{\beta}} D \gtrsim \tilde{D}_2^{\frac{1+\beta}{\beta}} \Longrightarrow \cE \gtrsim C^{-\frac{s \beta}{1+\beta}}.
    \]
    Hence when $\cE$ is optimized, we have $\tilde{D}_2 \eqsim D/\log C$ and
    \[
        \gamma_{\opt} \eqsim 1,
        M_{\opt} \eqsim C^{\frac{1}{1+\beta}},
        D_{\opt} \eqsim C^{\frac{\beta}{1+\beta}}, \nonumber
    \]
    and
    \[
        \cE_{\opt} \eqsim C^{-\frac{s\beta}{1+\beta}}. \nonumber
    \]
    \paragraph{Case 2.}
    $M^{\beta} \ge \gamma \tilde{D}_2$ and we have
    \[
    \cE_K \eqsim M^{-s\beta} + \frac{1}{(\gamma D)^s} + \frac{(\gamma \tilde{D}_2)^{\frac{1}{\beta}}}{\tilde{D}_2}.
    \]
    By AM-GM inequality,
	\[
		(\gamma D)^{-s} + \gamma^{\frac{1}{\beta}}
		\frac{1}{\tilde{D}_2^{1-\frac{1}{\beta}}} \gtrsim
		D^{-\frac{s}{1+s\beta}}\tilde{D}_2^{-\frac{s\beta-s}{1+s\beta}},
	\]
	with equality at $\gamma \eqsim D^{-\frac{s\beta}{1+s\beta}}\tilde{D}_2^{\frac{\beta - 1}{1+s\beta}}$.

    When $s>1-\frac{1}{\beta}$, the equality can be achieved, hence $(D_2)_{\opt} \eqsim D$, and the loss can be written as follows.
    \[
    \cE_K \eqsim M^{-s\beta} + D^{-\frac{s}{1+s\beta}}\tilde{D}_2^{-\frac{s\beta-s}{1+s\beta}}
    \]
    Combining with $C = MD$, we have the optimal scaling laws as follows:
    \[
        \gamma_{\opt} \eqsim C^{-\frac{1+s\beta-\beta}{2+s\beta}}(\log C)^{-\frac{\beta-1}{1+s\beta}}, 
        M_{\opt} \eqsim C^{\frac{1}{2+s\beta}}(\log C)^{-\frac{1-1/\beta}{2+s\beta}}, 
        D_{\opt} \eqsim C^{\frac{1+s\beta}{2+s\beta}}(\log C)^{\frac{1-1/\beta}{2+s\beta}},
    \]
    which leads to the following optimal final population risk:
    \[
        \cE_{\opt}\eqsim C^{-\frac{s\beta}{2+s\beta}}(\log C)^{\frac{s\beta-s}{2+s\beta}}. \nonumber
    \]

    When $s\le 1-\frac{1}{\beta}$, since $\gamma \lesssim 1$, we must have
	either $\gamma\eqsim 1$ or $\gamma \eqsim D^{-\frac{s\beta}{1+s\beta}}
	\tilde{D}_2^{\frac{\beta-1}{1+s\beta}}\lesssim 1$.
	To reach the minimum risk,
	in both cases we require $(\tilde{D}_2)_{\opt} \gtrsim
	D^{\frac{s\beta}{\beta-1}}$ (this gives $(D_2)_{\opt} \gtrsim
	D^{\frac{s\beta}{\beta-1}}\log D$), and
	\[
	\gamma_{\opt}\eqsim 1,
	\quad \mathcal{E}_{K} \eqsim M^{-s\beta}+D^{-s}.
	\]
    Combining with $C = MD$, we have the optimal scaling laws as follows:
    \[
        \gamma_{\opt} \eqsim 1,
        M_{\opt} \eqsim C^{\frac{1}{1+\beta}},
        D_{\opt} \eqsim C^{\frac{\beta}{1+\beta}},\nonumber
    \]
    which leads to the following optimal final population risk:
    \[
        \cE_{\opt} \eqsim C^{-\frac{s\beta}{1+\beta}}. \nonumber
    \]
    \paragraph{Summary.} Combining the results of each case, we get the desired optimal scaling strategy stated in the theorem.
\end{proof}

\section{Auxiliary Lemmas}
\begin{lemma}
    \label{lem:4th_moment_estim}
    For any PSD matrix $\bA$ and
    a random gaussian vector $\bx \sim \mathcal{N}(0,
    \bH )$,
    \[
    \tr(\bH \bA)\bH  \preceq
    \EE \left[\bx\bx^\top \bA
    \bx\bx^\top -
    \bH \bA\bH \right]
    = \tr(\bH \bA)
    \bH  + \bH \bA\bH 
    \preceq 2\tr(\bH \bA)\bH 
    \]
\end{lemma}
\begin{proof}[Proof]
Assume $\bA = (A_{ij})_{i,j = 1,...,M}$. 
The $(i, j)$-th entry of $\bx\bx^\top
\bA\bx\bx^\top$ is
\[
    \sum_{k,l} \bx_i\bx_kA_{kl}\bx_l\bx_j.
\]

If $i \neq j$,
\[
    \EE \left[\sum_{k,l} \bx_i\bx_kA_{kl}\bx_l\bx_j\right] = 2\EE \left[A_{ij} \bx_i^2\bx_j^2 \right]= 2A_{ij}\lambda_i\lambda_j = 2\bH \bA\bH (i,j).
\]
If $i = j$
\[
    \EE \left[\sum_{k,l} \bx_i\bx_kA_{kl}\bx_l\bx_j\right] = \EE \left[\sum_{k=1}^M A_{kk} \bx_i^2\bx_k^2\right] = \sum_{k\neq i} A_{kk}\lambda_i\lambda_k + 3A_{ii}\lambda_i^2 = 2\bH \bA\bH (i,i) + \tr (\bH \bA) \bH .
\]

By the trace inequality we have
\[
    \bH \bA \preceq \tr(\bH \bA).
\]
Multiplying $\bH $ at both sides,
\[
    \bH \bA\bH  \preceq \tr(\bH \bA) \bH .
\]

Combining the results, we have
\[
    \EE [\bx\bx^\top
    \bA\bx\bx^\top ] = \tr(\bH \bA) \bH  + 2\bH \bA\bH  \preceq 2\tr(\bH \bA)\bH  + \bH \bA\bH .
\]
\end{proof}

\begin{lemma}
    \label{lem:spd_matrix_trace}
    Let $\mathbf{P}\preceq \mathbf{Q}$ be two PSD matrices. Then for any PSD matrix $\mathbf{U}$, we have
    \[
    \tr(\sqrt{\mathbf{P}}\mathbf{U}\sqrt{\mathbf{P}})
    \le \tr(\sqrt{\mathbf{Q}}\mathbf{U}\sqrt{\mathbf{Q}}).
    \]
\end{lemma}
\begin{proof}[Proof]
    It is clear that $\tr(\sqrt{\mathbf{P}}\mathbf{U}\sqrt{\mathbf{P}})
    = \tr(\mathbf{U}\mathbf{P})$ and
    \[
    \tr(\mathbf{U}\mathbf{Q}) - \tr(\mathbf{U}\mathbf{P})
    = \tr(\mathbf{U}(\mathbf{Q}-\mathbf{P})) \ge 0,
    \]
    since $\mathbf{U}$ and $\mathbf{Q} - \mathbf{P}$ are both PSD matrices.
\end{proof}

% \section{Concluding Remarks}

% In this paper, we present a systematic study of how learning-rate schedule (LRS)  influences scaling laws by considering  a teacher–student kernel regression setting. Our theoretical framework yields a novel functional scaling law, which explicitly characterizes the impact of both learning rate and batch size schedules via a convolution-type functional term. The utility of our FSL is demonstrated through detailed analyses of three widely used LRSs, providing theoretical justification for several prevailing practices in LLM pre-training -- most notably, offering an explanation for the effectiveness of the empirically popular but previously less-understood warmup–stable–decay  schedules.

% Looking ahead, several promising directions remain open. 
% First, it is essential to develop a theoretical characterization of the optimal LRS under various resource constraints, which remains largely unresolved.
% Second, leveraging FSL to study the effect of batch size scheduling -- a relatively under-explored topic despite its practical importance. 
% Third, to make FSL truly actionable, it is crucial to  refine our FSL through extensive large-scale LLM pre-training experiments. 

\end{document}